\documentclass[]{bytedance_seed}

\usepackage{amsmath,amsfonts,bm}

\def\1{\bm{1}}

\DeclareMathAlphabet{\mathsfit}{\encodingdefault}{\sfdefault}{m}{sl}
\SetMathAlphabet{\mathsfit}{bold}{\encodingdefault}{\sfdefault}{bx}{n}

\newcommand{\trinorm}[1]{{\left\vert\kern-0.25ex\left\vert\kern-0.25ex\left\vert #1 
   \right\vert\kern-0.25ex\right\vert\kern-0.25ex\right\vert}}

\newcommand{\cE}{\mathcal{E}}

\newcommand{\pll}{\kern 0.56em/\kern -0.8em /\kern 0.56em}

\newcommand{\R}{\mathbb R}
\newcommand{\E}{\mathbb E}
\newcommand{\dd}{\,\mathrm d}
\newcommand{\tr}{\operatorname{tr}}
\newcommand{\diag}{\operatorname{diag}}
\newcommand{\af}{\alpha_{\mathrm{fast}}}
\newcommand{\as}{\alpha_{\mathrm{slow}}}
\newcommand{\rf}{\rho_{\mathrm{fast}}}
\newcommand{\rs}{\rho_{\mathrm{slow}}}
\newcommand{\bd}{\gamma}
\newcommand{\atpg}{a_{\mathrm{top}}}
\newcommand{\aperp}{a_{\mathrm{full}}}
\newcommand{\SMS}{\mathsf S}
\newcommand{\KMS}{\mathsf K}
\newcommand{\FMS}{\mathsf F}

\usepackage{graphicx,tikz}
\usetikzlibrary{arrows.meta,positioning}

\usepackage[toc,page,header]{appendix}

\usepackage[utf8]{inputenc} 
\usepackage[T1]{fontenc}    
\usepackage{titletoc}       
\usepackage{booktabs}       
\usepackage{amsfonts}       
\usepackage{nicefrac}       
\usepackage{microtype}      
\usepackage{xcolor}         

\usepackage{framed}
\usepackage{footnote}
\usepackage{algorithm}
\usepackage{algorithmic}
\usepackage{enumerate}

\usepackage{graphicx}
\usepackage{etoc}

\usepackage{minitoc}

\usepackage{amsmath}
\usepackage{amssymb}
\usepackage{mathtools}
\usepackage{amsthm}
\usepackage{multirow}
\usepackage{subcaption}
\theoremstyle{plain}
\newtheorem{theorem}{Theorem}[section]
\newtheorem{proposition}[theorem]{Proposition}
\newtheorem{lemma}[theorem]{Lemma}
\newtheorem{corollary}[theorem]{Corollary}
\theoremstyle{definition}

\newtheorem{assumption}[theorem]{Assumption}

\newtheorem{remark}[theorem]{Remark}

\usepackage{wrapfig}
\usepackage{xspace}
\usepackage{cancel}

\newcommand{\muonours}{{MuonM}\xspace}

\colorlet{shadecolor}{orange!15}

\title{Curvature-Conditioned Multiscale Momentum with Sphere Constraints for LLM Pretraining}
\author[1,2]{Shuchen Zhu}
\author[1]{Yuxin Fang}
\author[2, \dagger]{Mingze Wang}
\author[2, \dagger]{Kun Yuan}

\affiliation[1]{ByteDance Seed}
\affiliation[2]{Peking University}

\contribution[\dagger]{Corresponding authors}
\abstract{
Pretraining accounts for a large fraction of the total computational cost in LLM training. 
However, noise-dominant gradients and the highly ill-conditioned loss landscape bring severe challenges. 
Although modern adaptive optimizers such as AdamW and Muon have achieved great success in large-scale pretraining, their reliance on gradient normalization  offers limited mitigation of the ill-conditioned curvature. The progress along flat directions (eigen-directions of small eigenvalues), which dominates the final loss reduction, remains relatively slow. 
To enhance training dynamics along flat directions, we propose a curvature-conditioned multiscale momentum method with sphere constraints,  delivering steady acceleration in LLM pretraining. 
This multiscale momentum, applied only along flat directions, pairs a slow-decay component for noise reduction with a fast-decay component for rapid curvature adaptation, harnessing their complementary strengths.  
Crucially, we employ a sphere constraint technique to prevent  parameter inflation and excessively rapid effective learning rate decay that would otherwise arise from a naive combination. 
Extensive experiments show that the proposed method significantly accelerates Muon across diverse architectures (dense, MoE) and model sizes (0.12B--2.3B parameters). 
Theoretically, we verify the acceleration effect and provide insight into the design principles underlying the flat-direction multiscale momentum.
}

\date{\today}
\correspondence{\text{Mingze Wang} at \email{mingzewang.math@gmail.com}, Kun Yuan at \email{kunyuan@pku.edu.cn}.}

\begin{document}
\maketitle


\section{Introduction}
\label{sec:intro}

The optimizer is one of the most critical factors for training efficiency and continued scaling in the pretraining stage of large language models (LLMs). Several inherent challenges complicate LLM pretraining. Because the total number of training tokens vastly exceeds the per-step batch size, \textit{noise dominates the stochastic gradient}: its magnitude far exceeds the true gradient signal. Moreover, the loss landscape is \textit{highly nonconvex and ill-conditioned}, characterized by numerous flat directions alongside a few sharp ones. Notably, \textit{the training dynamics along the flat directions primarily drive the final loss reduction} \cite{wen2025understanding,song2025does}. Since the learning rate is constrained by the largest eigenvalue along the sharp directions and the noise level \cite{cohen2021gradient,andreyev2025edgestochasticstabilityrevisiting}, progress along these crucial flat directions remains slow, hampering overall training and potentially compromising stability.

Adaptive optimizers such as AdamW \cite{loshchilov2017decoupled,kingma2014adam} and Muon \cite{jordan2024muon,liu2025muon} partially alleviate ill-conditioning through coordinate- or block-wise normalization of gradient statistics. They have demonstrated stability and favorable scaling behavior in industrial-scale LLM training. However, this normalization (whitening) only partially compensates for the ill-conditioned curvature, and progress along the flat directions leaves room for improvement. Recent works~\cite{wang2024improving,wang2025the,zhou2025bsfaleveragingsubspacedichotomy,zhu2026acceleratingllmpretrainingflatdirection} attempt to address this by estimating flat directions and manually amplifying the learning rate along them. Yet this approach also \textbf{amplifies the stochastic noise along flat directions}, yielding limited improvement and potentially destabilizing training. This motivates the question:

\begin{center}
   \textbf{How to accelerate training dynamics along flat directions by reducing gradient noise?} 
\end{center}

We address this with a momentum mechanism tailored to flat directions. Our contributions are summarized as follows:

\begin{itemize}
    \item \textbf{Algorithm design.} We propose a sphere-constrained multiscale momentum method that combines fast- and slow-decay momentum components to enhance training dynamics along the flat directions. The multiscale momentum is \textit{curvature-conditioned}: it activates in flat directions for robust noise reduction and deactivates in sharp directions to preserve stability. On its own, however, it induces norm inflation and effective learning rate collapse. To address this, we introduce   \textit{sphere constraints to weights with parallel transport  on each momentum component}, which prevents norm inflation and ensures proper inheritance of momentum across varying tangent spaces. Together, they \textit{unlock the potential of multiscale momentum} in LLM pretraining.

    \item \textbf{Empirical evaluation.} We evaluate our proposed method (denoted by appending the suffix ``M'' to the base optimizer, e.g., \muonours)  across a wide range of LLM pretraining settings, covering both dense and MoE architectures, model sizes from 0.12B to 2.3B, and diverse learning rate schedules (cosine decay and WSD). The results show that \muonours consistently outperforms Muon with lower terminal loss. Notably, these loss gains persist under extended token budgets, suggesting promising scalability to  longer training horizons.

    \item \textbf{Theoretical analysis.} Recent studies~\cite{lin2024scaling,li2025functional,meterez2026defensequadraticmodel} show that appropriately designed linear regression models can serve as a useful proxy for LLM pretraining dynamics, capturing key trends in loss evolution. Within this framework, we analyze how slow-decay momentum injection along flat directions accelerates optimization.  Our analysis reveals that it \textit{enables faster forgetting of accumulated noise while improving the signal learning rate, thereby  leading to  lower terminal loss}. These findings support the design intuition behind the multiscale momentum for flat directions.
\end{itemize}

\subsection{Notations}
For a matrix $M\in \mathbb{R}^{m\times n}$, we define the matrix sign function as $\text{msign}(M) = M(M^\top M)^{-1/2} = (MM^\top)^{-1/2}M$, where the matrix inverse square root is interpreted via the Moore-Penrose pseudo-inverse if $M$ is not full rank. Let $\|\cdot\|_F$ denote the Frobenius norm. The operator $\mathsf{RowScale}_r(M)$ scales each row of $M$ to have a Frobenius norm of $r$, and $\mathsf{Norm}_r(M)=rM/\|M\|_F$. For a positive semi-definite matrix $H$, the induced semi-norm of a vector $v$ is defined as $\|v\|_H = \sqrt{v^\top Hv}$. The symbol $\odot$ represents
the element-wise product.  We use the term \textit{curvature} in an intuitive sense to refer to Hessian-related information (e.g., eigenvalue distributions).  $\lesssim$, $\gtrsim$, and $\asymp$ denote $\mathcal{O}$, $\Omega$, and $\Theta$, respectively, to absorb absolute constants.

\section{Preliminaries and Related Works}
\paragraph{Landscape geometry  and training dynamics in deep learning.}
Empirical studies ~\cite{pmlr-v97-ghorbani19b,yao2020pyhessian,zhang2024why,pmlr-v267-tang25d} show that the loss landscape in deep learning is highly ill-conditioned. 
Specifically, the Hessian possesses a small proportion of large positive eigenvalues (the corresponding eigenspaces are referred to as the \textit{sharp directions})  and a massive number of near-zero and negative  eigenvalues (corresponding to the \textit{flat directions}).
Recent studies~\cite{song2025does, cohen2025understanding, wen2025understanding} show that this landscape structure drives the \textbf{fast–slow training dynamics}: oscillatory behavior along sharp directions without divergence, and slow but persistent progress along flat directions that  ultimately dominates the final loss reduction.
Thus, accelerating optimization along flat directions is critical, and this is precisely the focus of our work.

\paragraph{Efficient adaptive optimizer design for LLM pretraining.}
Adaptive optimizers have become indispensable in large-scale LLM pretraining.  AdamW~\cite{loshchilov2017decoupled} has become the de facto choice, owing to its simple and effective coordinate-wise preconditioning. However, its diagonal preconditioner inherently limits its capacity to capture cross-parameter correlations within the full curvature.  Muon~\cite{jordan2024muon,liu2025muon} instead applies matrix-level preconditioning to  parameter blocks, exhibiting favorable convergence speed and scaling behavior. Recent methods, such as~\cite{vyas2025soap,pethick2025training,li2026normuon}, further develop structured matrix preconditioners to  better approximate curvature-related information. Our work is complementary to this line of research. Instead of improving the preconditioner structure, we focus on the design of momentum.

\paragraph{Improved momentum mechanism for adaptive optimizers.} Heavy Ball momentum~\cite{POLYAK19641} has been widely  adopted in modern adaptive optimizers to facilitate training. Building upon it, Nesterov-type momentum~\cite{xie2023adan,yuan2025mars} offers additional steady performance improvements. More recently, multiscale momentum has been introduced to smooth and accelerate the dynamics of momentum methods. AggMo~\cite{lucas2019aggregatedmomentumstabilitypassive} linearly combines multiple momentum buffers with different decay rates to dampen oscillations. AdEMAMix~\cite{pagliardini2025the} extends this  by integrating multiscale momentum with Adam's preconditioner in training neural networks. GPA~\cite{defazio2026smoothingdilocoprimalaveraging} employs extra primal averaging to smooth the iterates of  Nesterov momentum. SODA~\cite{pethick2026optimisticdualaveragingunifies} and EMA-Nesterov~\cite{yau2026emanesterovstabilizingnesterovslookahead} introduce an extra momentum via primal extrapolation and  exponential moving average of look-ahead directions  respectively to achieve acceleration in pretraining, while accommodating various base optimizers. Different from these approaches which are motivated primarily by deterministic  optimization, our work focuses on its role in variance reduction under noise-dominated stochastic environments, as well as on distinct behaviors at different time scales across regions of varying curvature (sharp and flat). In addition, we provide a unified high-order ODE formulation that offers a continuous-time perspective for understanding these different algorithmic forms in Appendix \ref{sec:multiscale-ode-unification}.

\paragraph{Controlling the effective learning rate: from weight decay to sphere constraints.}
Normalization layers in modern LLMs introduce redundant degrees of freedom, making the learning rate alone insufficient to fully characterize the parameter update speed. 
For instance, an RMSNorm layer of the form $W_1\bigl(\gamma \odot \mathrm{Norm}(W_2 x)\bigr)$ exhibits a radial redundancy structure $W_1 \operatorname{diag}(\gamma)$ and norm‑scale invariance with respect to $W_2$ (when $W_2$ is learnable). 
This naturally motivates treating radial and angular updates separately. 
Recent studies~\cite{roburin2022sphericalperspectivelearningnormalization,kosson2024rotational,wen2026fantasticpretrainingoptimizersii,zhou-zhou-gu-2026-elr} regard the angular update rate $\|\Delta W\|_F / \|W\|_F$ as the \textit{effective learning rate}, which has been empirically shown to play a dominant role in shaping the loss trajectory~\cite{li2025efficienthyperparametertuningtrajectory,xiao2026hyperballfreelunch,liu2026effectivelearningrategoverns}. 
The radial evolution determines the weight norm and, in turn, affects the effective learning rate. 
Weight decay can be interpreted as controlling the noise‑induced inflation of the weight norm, thereby preventing the effective learning rate from decaying too rapidly. Reference \cite{defazio2025gradientsrapidlyincreasenear} uses corrected weight decay coefficients to stabilize the weight norm. 
To achieve finer control, some recent approaches~\cite{xie2026controlledllmtrainingspectral,wen2026fantasticpretrainingoptimizersii,an2026demystifyingmanifoldconstraintsllm,hagele2026improvingneuralnetworktraining} replace weight decay with  direct norm constraints to weights, and they have demonstrated promising advantages over weight decay in LLM pretraining. 

\section{Illustrative Examples Motivating Flat-Direction Multiscale Momentum}
In this section, we illustrate how varying momentum decay rates affect stochastic optimization dynamics. 

\subsection{Momentum Accelerates Flat-Direction Optimization Through Variance Reduction}\label{sec:momentum_var_red}
We begin with a simple case: at iteration $k$, the stochastic gradient follows $g_k\sim \mathcal{N}(\mu,\Sigma^2)$ with a constant true gradient $\mu$, and the heavy-ball momentum updates as  $m_k=(1-\alpha)m_{k-1}+\alpha g_k$.   The coefficient $1-\alpha$ measures the momentum decay rate: \textbf{smaller  $\alpha$ corresponds to slower decay, and larger $\alpha$ gives faster decay.}   The  estimator $ m_k$ has the same mean $\mu$ and a reduced equilibrium covariance $\frac{\alpha}{2-\alpha}\Sigma^2$. 
Thus, using a smaller $\alpha$ (slower momentum decay) allows $m_k$ to estimate the true gradient signal $\mu$ with lower variance.

However, the above reasoning requires the true gradient to be fixed or nearly unchanged. 
Since gradient variation is directly governed by its derivative (i.e., the Hessian), slow momentum decay is only safe in flat directions, where small Hessian eigenvalues lead to slowly changing gradients. 
Applying overly slow momentum decay to sharp directions can introduce bias and cause instability. 
A simple example of minimizing $f(x,y)=x^2-0.1y$ with Gaussian  gradient noise illustrates this clearly (Figure~\ref{fig:ex-quadratic-river-valley}). 
Figure~\ref{fig:ex-quadratic-river-valley} (a,b) show that compared to a baseline with $\alpha=0.1$, \textit{using a smaller $\alpha$ (slower decay) for the flat direction $y$ yields a faster and less noisy trajectory, whereas the same slow decay for the sharp direction $x$ amplifies oscillations.}  
While increasing the learning rate for $y$ also promotes progress, it simultaneously amplifies noise and leads to an oscillatory and relatively more unstable acceleration (Figure~\ref{fig:ex-quadratic-river-valley} (c)). 
These observations motivate the use of anisotropic momentum decay to better suppress noise:
\begin{equation}\label{alg:anisotropic_momentum_decay}
    \begin{aligned}
        m_k&=(I-\alpha_{\text{sharp}}\mathcal{P}_k-\alpha_{\text{flat}}\mathcal{Q}_k)m_{k-1}+g_k, \\
        w_{k+1}&=w_k-\eta_k m_k,
    \end{aligned}
\end{equation}
where $1>\alpha_{\text{sharp}}\gg\alpha_{\text{flat}}>0$ denote the decay coefficients for sharp and flat directions, and $\mathcal{P}_k,\mathcal{Q}_k$ denote projections onto the sharp and flat subspaces at the $k$-th step, respectively.

\begin{figure}
    \centering
    \includegraphics[width=1\linewidth]{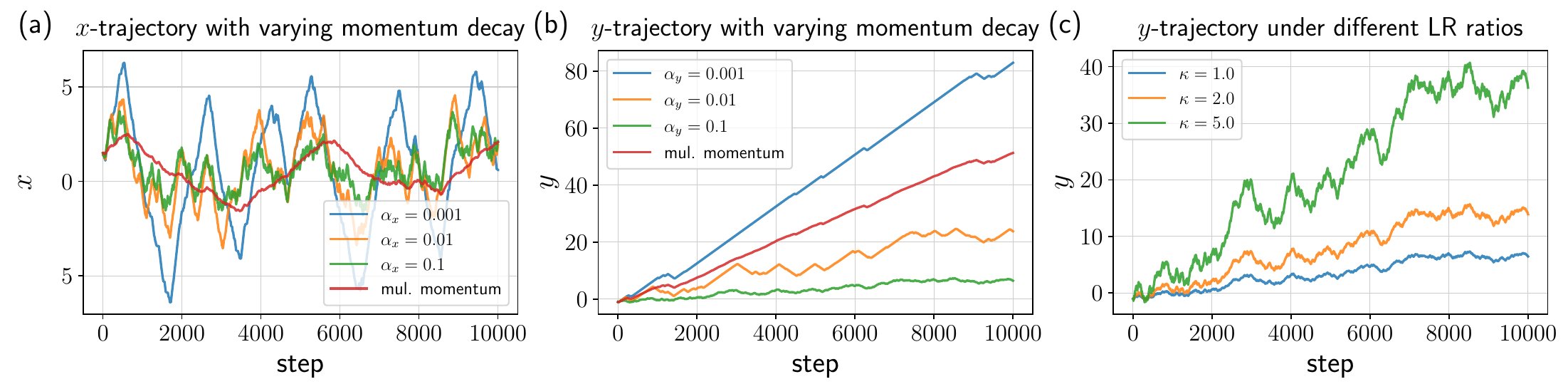}
    \caption{Coordinate-wise trajectories of different momentum methods for minimizing $f(x,y)=x^2-0.1y$ with the step size $\eta=0.01$ and sign-momentum preconditioning. The $x$- and $y$-axes correspond to the sharp and flat directions, respectively. (a, b) employ various momentum decay coefficients $\alpha_x,\alpha_y$ for $x,y$ respectively. ``mul. momentum'' refers to \eqref{alg:flat-mts-base} with $(\alpha^{\text{fast}},\alpha^{\text{slow}},c)=(0.1,0.001,0.75)$ and a biased flat-direction projection $\mathcal{Q}_k (x,y)=\frac{1}{\sqrt{0.2^2+0.8^2}}(0.2x+0.8y)$. (c) amplifies the step size for $y$ by a factor of $\kappa$.}
    \label{fig:ex-quadratic-river-valley}
\end{figure}


\subsection{Flat-Direction Multiscale Momentum: Robust Variance Reduction against Projection Errors}
In practice, the acceleration provided by \eqref{alg:anisotropic_momentum_decay} is compromised by various imperfections, including errors in estimating the projection onto flat subspaces and unmodeled variations in the Hessian eigenspace. Projection estimation errors introduce a portion of the slow decay for momentum in sharp directions, thereby destabilizing training. At the same time, 
persistent rotation and alternating transitions of the Hessian eigenspaces inject large bursts of sharp-direction momentum into the flat subspaces. 
Due to the slow forgetting rate in flat regions (requiring $\Theta(\alpha_{\text{flat}}^{-1})$ steps), such biases can persist and cause oscillations.

To address this, we use multiscale momentum in flat directions, which linearly interpolates a fast-decay component $m^{\text{fast}}$ for quick adaptation to gradient and curvature changes and a slow-decay component $m^{\text{slow}}$ for variance reduction. At the same time, we retain the fast-decay momentum $m^{\text{fast}}$ to help maintain stability.  
This leads to the following algorithm:
\begin{equation}\label{alg:flat-mts-base}
    \begin{aligned}
        m_k^{\#}&=(1-\alpha_{\#})m_{k-1}^{\#}+\alpha_{\#}g_k, \quad {\#}\in\{\text{fast},\text{slow}\}, \\
        w_{k+1}&=w_k-\eta_k \bigl(c\mathcal{Q}_k  m^{\text{slow}}_k  +(1-c)  m^{\text{fast}}_k\bigr)
        \\&=w_k-\eta_k \bigl(\mathcal{Q}_k(c m^{\text{slow}}_k + (1-c)m^{\text{fast}}_k)+(1-c)\mathcal{P}_k m^{\text{fast}}_k\bigr),
    \end{aligned}
\end{equation}
where $\alpha_{\text{fast}}\gg\alpha_{\text{slow}}>0$ denote the momentum decay rates, $0<c<1$ is the interpolation coefficient, and $\mathcal{P}_k,\mathcal{Q}_k$ are   the projections onto sharp and flat subspaces. The  interpolation $c m^{\text{slow}}_k + (1-c)m^{\text{fast}}_k$ also yields lower noise in estimating the true gradient compared to $ m^{\text{fast}}_k$ 
  under the setting of Section~\ref{sec:momentum_var_red} (Appendix \ref{app:var-mts}). As illustrated in Figure \ref{fig:ex-quadratic-river-valley} (a, b) and Figure~\ref{fig:rosen}, with projection errors and varying curvature, the flat-direction multiscale momentum still yields \textit{less noisy acceleration along the flat direction while not amplifying oscillations in the sharp direction.}

  \begin{figure}[H]
    \centering
    \includegraphics[width=1\linewidth]{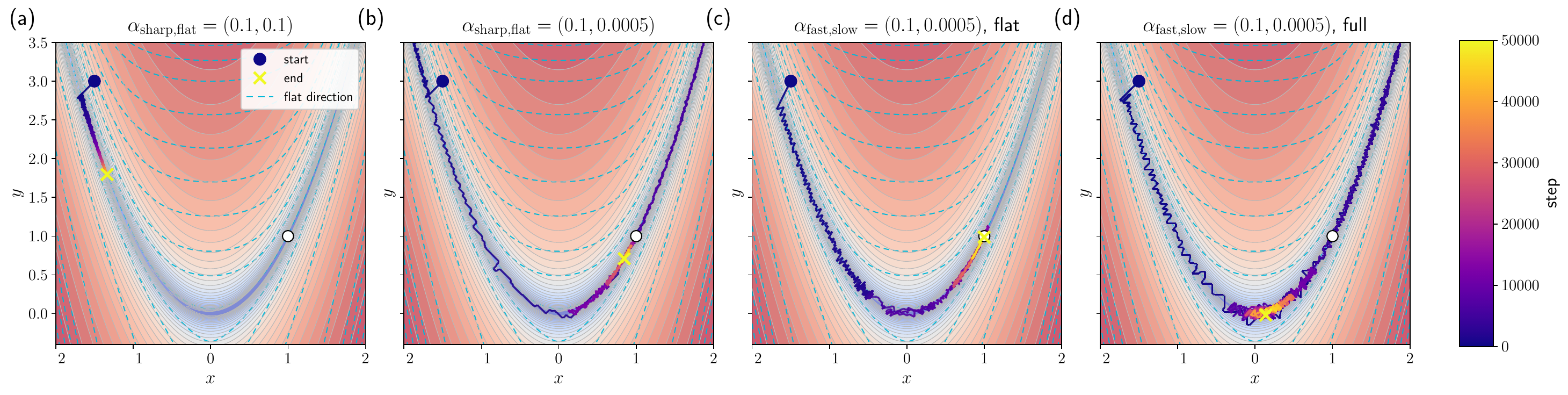}
    \caption{Optimization trajectories on $f(x,y)=(1-x)^2+1000(y-x^2)^2$ (global minima at $(1,1)$) under Gaussian noise. (a) Standard momentum progresses slowly.  (b) Anisotropic momentum \eqref{alg:anisotropic_momentum_decay} advances faster along flat directions but induces back-and-forth oscillations. (c) Flat-direction multiscale momentum \eqref{alg:flat-mts-base} with $c=0.5$ achieves stable acceleration. (d) Total-space multiscale momentum with $c=0.5$ leads to instability, confirming that this technique should be restricted to flat directions.}
    \label{fig:rosen}
\end{figure}


The intuition behind the robust variance reduction of flat-direction multiscale momentum   is as follows. Assume $m^{\text{fast}}_k=\nabla f(w_k)+b^{\text{fast}}_k+\xi^{\text{fast}}_k$, $m^{\text{slow}}_k=\nabla f(w_k)+b^{\text{slow}}_k+\xi^{\text{slow}}_k$ with bias $\|b^{\text{slow}}_k\|\gg\|b^{\text{fast}}_k\|$ (due to slower forgetting over longer gradient histories), zero mean noise $\|\xi^{\text{slow}}_k\|\ll\|\xi^{\text{fast}}_k\|$ and noise dominating the fast momentum such that $\|\xi^{\text{fast}}_k\|\gg\|b^{\text{fast}}_k\|+\|\nabla f(w_k)\|$. Then, for an appropriate choice of $c$, the expected   loss decrease conditioned on $w_k$ resulting from the $\mathcal{Q}_k$-subspace update can be approximated via a second-order Taylor expansion:
\begin{equation}
    \begin{aligned}
       &\mathbb{E}_k [f(w_{k
       +1})- f(w_{k})]|_{\mathcal{Q}_k}
       \\\approx& -\eta\mathbb{E}_k\left\langle \mathcal{Q}_k \nabla  f(w_k),\textcolor{orange}{c} \mathcal{Q}_km^{\text{slow}}_k + (1-c)\mathcal{Q}_km^{\text{fast}}_k\right\rangle\\&+ 
\frac{\eta^2}{2}\mathbb{E}_k(c m^{\text{slow}}_k + \textcolor{orange}{(1-c)}m^{\text{fast}}_k)^\top\mathcal{Q}_k\nabla^2 f(w_k)\mathcal{Q}_k(cm^{\text{slow}}_k +\textcolor{orange}{(1-c)} m^{\text{fast}}_k) 
\\\approx &\underbrace{-\eta \|\mathcal{Q}_k \nabla f(w_k)\|^2}_{\text{gradient descent}}+\underbrace{\textcolor{orange}{c}\eta \left\langle \mathcal{Q}_k\nabla f(w_k),- \mathcal{Q}_k b^{\text{slow}}_k \right\rangle}_{\text{bias injection}} + 
\underbrace{\frac{\textcolor{orange}{(1-c)^2}\eta^2}{2}\mathbb{E}_k(\xi^{\text{fast}}_k)^\top\mathcal{Q}_k\nabla^2 f(w_k)\mathcal{Q}_k\xi^{\text{fast}}_k}_{\text{noise variance}}.\nonumber
    \end{aligned}
\end{equation}
The noise-variance term is typically positive, since noise tends to accumulate along directions with larger Hessian eigenvalues \cite{wu2022the,song2025does,wang2026gradpower}. Assuming the bias 
 $\mathcal{Q}_kb^{\text{slow}}_k$ in $\mathcal{Q}_k$ subspace  is moderate relative to $\nabla f(w_k)$, we can choose suitable  $\textcolor{orange}{c<1}$ so that the bias injection is   negligible compared to the gradient-descent term, while the noise variance term is simultaneously suppressed, yielding an improved local loss reduction.  For the $\mathcal{P}_k$ subspace, we note that  $\mathbb{E}_k [f(w_{k
       +1})- f(w_{k})]|_{\mathcal{P}_k}\approx0$ as training typically lies in the Edge of Stability regime \cite{cohen2021gradient,andreyev2025edgestochasticstabilityrevisiting,andreyev2026momentum}. On the other hand, take $\textcolor{orange}{c>0}$ also reduces the  actual learning rate in the $\mathcal{P}_k$ 
  subspace, which helps decrease oscillations. We exclude slow momentum from $\mathcal{P}_k$, as slow-momentum bias in sharp directions can be much larger than the gradient signal, which would force $c$ to be vanishingly small and thereby restrict the acceleration in flat directions. Consequently, the overall effect is a net reduction in the local loss.

\section{Obstacles with Naive Multiscale Momentum: Parameter Inflation and Effective Learning Rate Collapse}\label{sec:obt-mts}

We now extend the idea of \eqref{alg:flat-mts-base} to the Muon optimizer for LLM pretraining. To estimate flat directions, we consider distinguishing them within each parameter block. Following \cite{zhu2026acceleratingllmpretrainingflatdirection}, we empirically identify the bottom 90\% singular spaces of the momentum matrix as flat directions, based on empirical evidence that the top singular spaces align well with the top eigenspaces of the Hessian. See implementation details in Appendix \ref{app:flat-proj}. For a transformer matrix $W\in \mathbb{R}^{m\times n} $, we extend the preconditioner-free algorithm \eqref{alg:flat-mts-base} by equipping  $m^{\text{fast}}$ with Muon preconditioning, while additionally incorporating the slow momentum $m^{\text{slow}}$ in flat directions.  The update rule \footnote{Here we set the coefficient of $G_k$ to 1. This choice differs from the EMA format in \eqref{alg:flat-mts-base} only by a constant factor and is thus essentially equivalent.} at the $k$-th iteration is
\begin{equation}\label{alg:ours-with-wd}
\begin{aligned}
M_k^{\#}&=(1-\alpha_{\#})M_{k-1}^{\#}+G_k, \quad{\#}\in\{\text{fast},\text{slow}\},
\\W_{k+1}&=(1-\lambda \eta_k)W_k-\eta_k\left(\mathsf{msign}((1-\alpha_{\text{fast}})M_{k}^{\text{fast}}+G_k)+\chi\mathcal{Q}_k ( \mathsf{RowScale 
}_{\min\left(1,\sqrt{\frac{n}{m}}\right)}(M_k^{\textrm{slow}}))\right), 
\end{aligned}
\end{equation}
where $\lambda$ is the weight decay coefficient, $\chi>0$ is the slow-momentum coefficient,  $\eta_k$ is the scaled learning rate \footnote{We use the RMS norm alignment rescaling  in \cite{liu2025muon}.}, and $\mathcal{Q}_k: \mathbb{R}^{m\times n}\to \mathbb{R}^{m\times n}$ is the estimated linear projection to flat directions. 
Rather than designing a specialized preconditioner for the slow momentum, we simply apply row-wise normalization to match the norm of the msign term. This approach is inspired by the near-zero inter-row Hessian observed in \cite{zhang2025adammini}, and it proves efficient in practice.


\begin{figure}[htbp]
    \centering
    \begin{subfigure}[b]{0.36\textwidth}
        \centering
        \includegraphics[width=\textwidth]{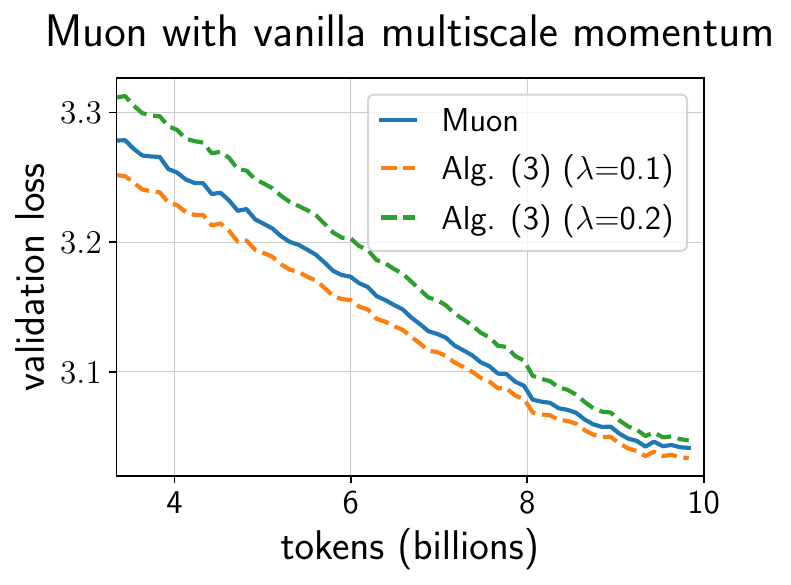}
        \label{fig:mts-wd-val}
    \end{subfigure}
     \hspace{0.3cm} 
    \begin{subfigure}[b]{0.36\textwidth}
        \centering
        \includegraphics[width=\textwidth]{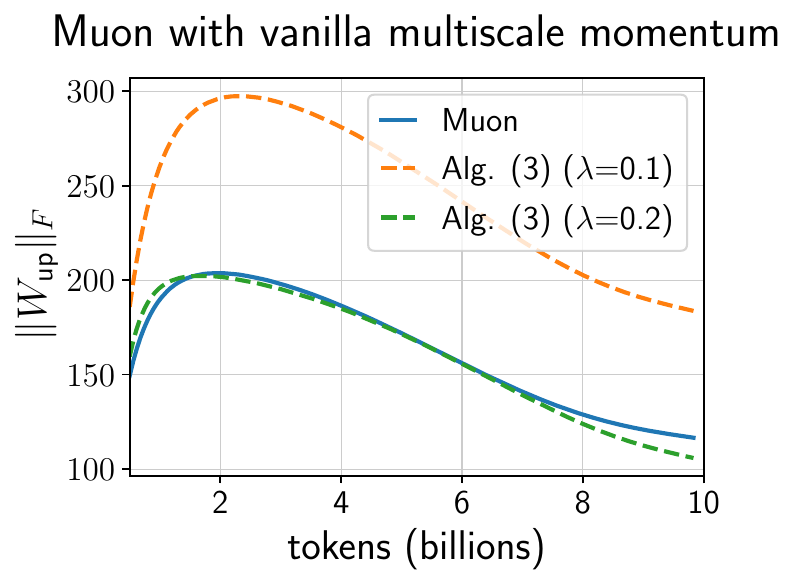}
        \label{fig:mts-wd-norm}
    \end{subfigure}
    \caption{Straightforward application of flat-direction multiscale momentum to Muon on a 0.12B dense model. (Left) Validation loss; (Right) Frobenius norm evolution of an up projection block during training.}
    \label{fig:mts-wd-120m}
\end{figure}

The results on 0.12B dense model are in Figure~\ref{fig:mts-wd-120m} (left), where both the Muon baseline and   the value of $\chi$ and $\alpha_{\text{slow}}$ have been thoroughly tuned (see details in Appendix \ref{app:exp-obt-mts}). The inclusion of slow momentum along flat directions yields a certain reduction in loss. However, the gain diminishes progressively, resulting in a terminal loss improvement of less than 0.01. This behavior is attributed to the amplified weight norm $\|W_k\|_F$ by slow momentum  (e.g., the norm of an up projection block in Figure~\ref{fig:mts-wd-120m}, right)  and in turn accelerates the decay of the effective learning rate  $\frac{\|\Delta W_k\|_F}{\|W_k\|_F}$. We provide an illustrative theoretical explanation to this phenomenon in Appendix \ref{app:norm-inflation},  which reveals its intrinsic link to the slow momentum.  This mechanism also explains the evolution of the loss gain: a sharp initial loss drop driven by the rapid effective learning rate decay in the early phase, followed by the baseline gradually catching up in the later phase, though the flat-direction slow-momentum acceleration still retains a marginal edge.

A natural remedy would be to use a larger weight decay coefficient $\lambda$ (default 
0.1) 
to restrain the norm inflation. However, even with $\lambda=0.2$
 which brings the norm back to the baseline level, the loss becomes even higher. This suggests that simply tuning existing hyper-parameters is insufficient to unleash the  acceleration potential of multiscale momentum. We therefore turn to weight-decay-free techniques.

\section{The Flat-Direction  Multiscale Momentum Method with Sphere Constraints}
We address the issue of weight norm inflation discussed in the previous section by replacing weight decay with  \textbf{sphere constraints}. Before presenting the final algorithm, we first give a short review of manifold optimization.

\subsection{Preliminaries on Manifold Optimization}\label{pre:manifold-opt} Manifold optimization deals with optimization problems where the variable is constrained to lie on a manifold. In this work, we restrict ourselves to the Frobenius sphere constraint $\mathbb{S}_R = \{ w \in \mathbb{R}^p : \|w\|_{F} = R \}$ with 
$R > 0$. The \textit{tangent space} at a point $w \in \mathbb{S}_R$ is
$T_w \mathbb{S}_R = \{ v \in \mathbb{R}^p : v^\top w = 0 \}
$.  To define the optimization geometry, the tangent space is equipped with a metric (a positive definite quadratic form). Here we simply use the Euclidean metric.

\begin{wrapfigure}{r}{0.35\linewidth}
    \centering
    \includegraphics[width=\linewidth]{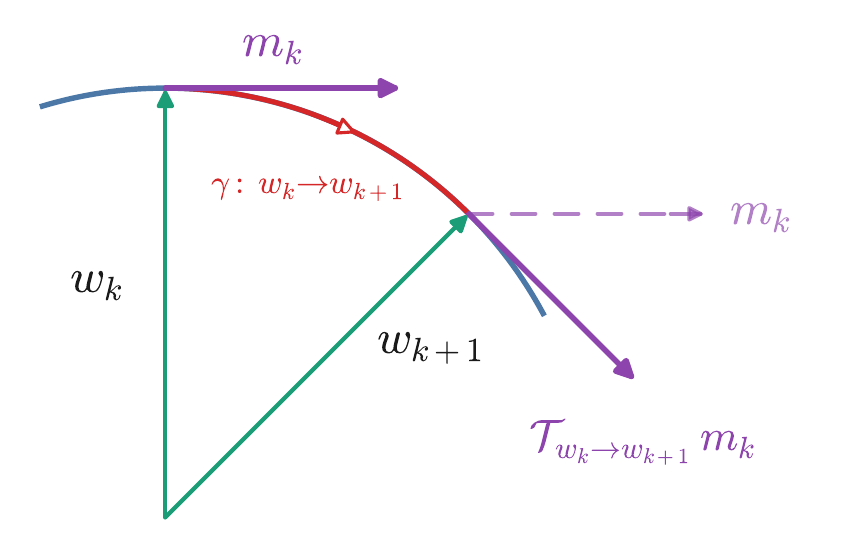}
    \caption{Illustration of parallel transport on the sphere.}
    \label{fig:placeholder}
    \vspace{-0.2cm}  
\end{wrapfigure}
At the $k$-th step, the update direction $\Delta w_k$ (e.g.,  gradient or momentum) is \textit{projected onto the tangent space} via $
\mathcal{P}_{T_{w_k} \mathbb{S}_R}(\Delta w_k) = \Delta w_k - \frac{\Delta w_k^\top w_k}{R^2}w_k $.
A \textit{retraction} $
w_{k+1} = \mathsf{Norm}_R(w_k - \eta_k \Delta w_k)$ is then applied to keep the iterate on the sphere, where $\eta_k$ is the step size. In this framework, momentum thus accumulates the Riemannian gradients  $\mathcal{P}_{T_w \mathbb{S}_R}(g)$, with $g$ denoting the Euclidean gradient. After moving from $w_k$ to $w_{k+1}$, the momentum vector  in $T_{w_k}\mathbb{S}_R$ is transferred to the new tangent space $T_{w_{k+1}}\mathbb{S}_R$. This is accomplished by \textit{parallel transport} $
\mathcal{T}_{w_k \to w_{k+1}} : T_{w_k}\mathbb{S}_R \to T_{w_{k+1}}\mathbb{S}_R,
$
which moves tangent vectors along the manifold while preserving their norms. Intuitively, parallel transport ensures the  momentum  is adjusted to the new geometry, preventing a gradual distortion of the accumulated information. For the Frobenius sphere, parallel transport admits a simple closed form. Omitting it would amplify direction errors in the accumulated gradient history, especially when momentum decays slowly and the tangent spaces vary substantially along the optimization trajectory. We will elaborate on this point in Section~\ref{sec:ablation}. Given $v \in T_{w}\mathbb{S}_R$ and a new point $w' \in \mathbb{S}_R$ ($w'\neq- w$), the transported vector is
\begin{equation}
  \mathcal{T}_{w \to w'}(v) = v - \frac{v^\top w'}{R^2 + w^\top w'} \, (w + w').  
\end{equation}
The above discussion can be directly extended to other tensor types such as matrices. 

\subsection{The Proposed  Approach}\label{sec:ours-approach}

\paragraph{Directional Sphere Parameterization with Learnable Radius} 

Motivated by \cite{wen2026fantasticpretrainingoptimizersii}, we adopt a fixed Frobenius norm constraint (denoted as $R$) and rescale the optimizer update $\Delta\widehat{W}$ to the same norm $R$, which translates the learning rate into an angular velocity on the sphere. 
To decouple radial and angular updates, we re-parameterize each transformer weight matrix $W$  as $W = r \widehat{W}$, where $\|\widehat{W}\|_F\equiv R$  and $r\in \mathbb{R}$ is a learnable scalar.
Compared to a purely fixed‑norm parameterization, the learnable $r$ automatically adapts to the natural heterogeneity across blocks and layers \cite{wang2026negligiblesizesignificanteffect}. 
Vector parameters (e.g., biases) are re-parameterized analogously. 
Given the token-level heterogeneity in the embedding and output matrices $W \in \mathbb{R}^{V \times d}$ with vocabulary size $V$ and model dimension $d$, we adopt a row-wise re-parameterization as $W=\diag{(\gamma)} \widehat{W}$, where $\gamma\in \mathbb{R}^V$   and each row of $\widehat{W}$ maintains a fixed norm. We denote a base optimizer used on the re-parameterized model by appending the suffix "S" (sphere) to its name (e.g., \textbf{MuonS}).

\paragraph{Algorithm Design} 

We now apply the manifold optimization paradigm (parallel transport $\textcolor{orange}{\mathcal{T}}$ and tangent space projection ${\color{violet}\mathcal{P}}$) from Section~\ref{pre:manifold-opt} to the base multiscale momentum method~(\ref{alg:ours-with-wd}) under the sphere constraint described above. We rescale both the fast-momentum and slow-momentum update norms to the sphere radius  $R$, so that the learning rate  acts as the angular velocity and simplifies tuning. The resulting algorithm (denoted as \textbf{\muonours}) follows the Nesterov-type momentum update of~\cite{jordan2024muon}, using the Nesterov-type momentum $\widetilde{M}_k^{\mathrm{fast}}:=(1-\alpha_{\text{fast}}){M}_k^{\mathrm{fast}}+\mathcal{P}_{T_{W_k}\mathbb{S}_R}G_k$ in \eqref{m3-update}. 
The complete procedure for transformer blocks is summarized in Algorithm~\ref{alg:ours}, which degenerates to MuonS when $\chi=0$. We apply Adam to scalar parameters and sphere-constrained Adam to vector parameters. Please refer to Section~\ref{app:ours-alg-adam} for the complete algorithm details.

\begin{algorithm}[H]
  \caption{\muonours for a transformer block matrix $W\in \mathbb{R}^{m\times n}$}
  \label{alg:ours}
  \begin{algorithmic}[1]
    \STATE {\bfseries Input:}  $\{\eta_k\},\alpha_{\text{fast}}>\alpha_{\text{slow}}>0,\chi\ge0$.
    \FOR{$k=0$ {\bfseries to} $K$}
    \STATE Update multiscale momenta $M_k^{\#}=(1-\alpha_{\#})\textcolor{orange}{\mathcal{T}_{W_{k-1}\to W_{k}}} (M_{k-1}^{\#})+{\color{violet}\mathcal{P}_{T_{W_k}\mathbb{S}_R}}G_k, \quad {\#}\in\{\text{fast},\text{slow}\}$.
    \STATE Estimate projection $\mathcal{Q}_k$ onto the flat directions in the tangent space. Then compute the update direction
    \begin{snugshade}
    \vspace{-.1cm}
    \begin{equation}\label{m3-update}
    \begin{aligned}
U_k&=\mathsf{Norm}_R({\color{violet}\mathcal{P}_{T_{W_k}\mathbb{S}_R}}\mathsf{msign}(\widetilde{M}_k^{\mathrm{fast}}))+\chi\mathcal{Q}_k(\mathsf{RowScale}_{R/\sqrt{m}}(M_k^{\mathrm{slow}})).
        \end{aligned}
\end{equation}
    \end{snugshade}
    \vspace{-.1em}
    \STATE Update parameters $W_{k+1}= \mathrm{\mathsf{Norm}}_R(W_k-\eta_k U_k)$.
    \ENDFOR
  \end{algorithmic}
\end{algorithm}

The row‑wise rescaling aligns the norm of the slow momentum with that of the msign term.
The flat‑direction projection $\mathcal{Q}_k$ already includes the projection to the tangent space, and is computed by online power iterations to estimate top singular subspaces of $\widetilde{M}_k^{\mathrm{fast}}$.  Implementation details are provided in Section~\ref{app:flat-proj}.

\paragraph{Computation, Communication and Memory Overhead.} Our method incurs only marginal additional overhead in computation, memory, and communication compared to the base Muon optimizer.
The extra computation stems mainly from manifold-related operations (tangent space projection, retraction, and parallel transport) and the flat-direction projection $\mathcal{Q}_k$.
The former involves no dense matrix–matrix products  while the latter's main cost consists of two lightweight msign operations on matrices much smaller than $W$, which is  nearly negligible as analyzed in Section~\ref{app:flat-proj}.
The additional memory comes from the slow momentum buffer and the small projection matrices used by $\mathcal{Q}_k$, totaling roughly the size of one extra set of model parameters.
In pretraining, memory is typically  dominated by activations, and the  extra optimizer state can be further reduced to a marginal amount per GPU via Fully Sharded Data Parallelism (FSDP)~\cite{zhao2023pytorchfsdpexperiencesscaling}. 
Regarding communication, in large-scale distributed training that employs parallelism strategies such as expert parallelism (EP) or pipeline parallelism (PP), the communication volume is similarly dominated by activations, and the extra communication from the sharded optimizer states is insignificant.

\section{Experiments}

\begin{wrapfigure}{r}{0.4\textwidth}
    \centering
    \includegraphics[width=\linewidth]{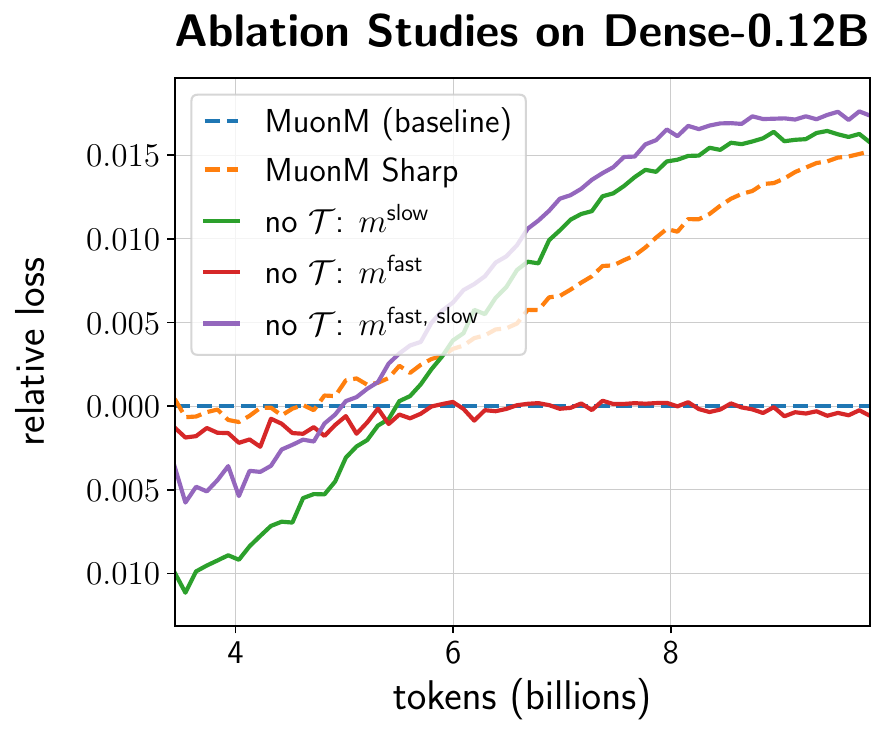}
    \caption{Ablation studies on projection direction and parallel transport $\mathcal{T}$. The label "Sharp" corresponds to replacing $\mathcal{Q}$ with the  sharp-direction projection.  }
    \label{fig:ablation}
\end{wrapfigure}

In this section, we evaluate our proposed optimizer on dense and Mixture-of-Expert  (MoE)  architectures with scales from 0.12B to 2.3B. All models are trained using data from a 350B-token high-quality pretraining corpus. The training budget is approximately 100 tokens per parameter (TPP) for dense models and approximately 500 tokens per activated parameter for MoE models, substantially exceeding the 20 TPP of the Chinchilla‑optimal data~\cite{hoffmann2022training} and aligning more closely with industrial pretraining practice. Unless otherwise specified, the token batch size is about 0.5M for dense models and 4M for MoE models. More experimental details are in Appendices \ref{app:alg-details}, \ref{app:exp-details}.

\subsection{Validation of Critical Algorithm Components}\label{sec:ablation}

We evaluate the roles of the flat-direction projection $\mathcal{Q}$ and parallel transport for momentum on a 0.12B dense model, using \muonours as the baseline (Figure \ref{fig:ablation}).
We first replace the flat-direction projection $\mathcal{Q}$ with the sharp-direction projection (i.e., the orthogonal complement projection of $\mathcal{Q}$), which leads to a large increase in loss, indicating that \textbf{applying multiscale momentum along sharp directions is detrimental}.

Next, we examine the effect of parallel transport on the fast and slow momentum separately.
Removing parallel transport for the fast momentum alone results in a negligible change in the loss trajectory.
In contrast, removing it for the slow momentum causes a significant increase in loss, and additionally disabling parallel transport for the fast momentum has little further impact.
This suggests that \textbf{parallel transport is unnecessary for the fast momentum but indispensable for the slow momentum}. 
A plausible explanation is that the fast momentum decays rapidly. Over its short decay horizon, relatively few parameter updates occur and the tangent space changes little. Thus  the accumulated error from skipping parallel transport remains small.
The slow momentum, on the other hand, decays slowly and spans many updates over a large parameter range, during which the tangent space can change substantially. Parallel transport is therefore required to keep the error accumulation under control.
For instance, with a decay factor of $0.95$ (half-life $\approx 14$ steps) for the fast momentum and $0.999$ (half-life $\approx 700$ steps) for the slow momentum, this difference in update ranges becomes evident.

\subsection{Results on Dense Models}\label{sec:results-dense}


\paragraph{Comparison with Related  Methods.} We first compare our sphere constraint strategy (Section~\ref{sec:ours-approach}) with two recent methods of the same type: MuonH~\cite{wen2026fantasticpretrainingoptimizersii} (Frobenius norm sphere) and SSO~\cite{xie2026controlledllmtrainingspectral} (spectral norm sphere).
All sphere-constrained methods outperform the weight-decayed Muon baseline, with MuonS (i.e., \muonours without slow momentum) achieving the lowest loss (Figure~\ref{fig:compare-120m}, left).
We attribute this advantage to the adaptivity of the learnable radius across heterogeneous blocks.
Next, we compare  \muonours with other enhanced momentum approaches, including AdEMAMix~\cite{pagliardini2025the} (using AdamW as the base optimizer) and EMA-Nesterov~\cite{yau2026emanesterovstabilizingnesterovslookahead} (using Muon as the base optimizer).
\muonours achieves a lower terminal loss than these methods by a clear margin (Figure~\ref{fig:compare-120m}, right). 
Notably, the loss improvement of \muonours over MuonS (-0.019) and over Muon (-0.028) is substantially larger than the gain from directly applying flat-direction multiscale momentum to Muon (-0.008, see Figure~\ref{fig:mts-wd-120m}).
This indicates that \textbf{the sphere constraint more effectively unlocks the acceleration potential of flat-direction multiscale momentum.}

\begin{figure}[htbp]
    \centering
    \begin{subfigure}[b]{0.35\textwidth}
        \centering
        \includegraphics[width=\textwidth]{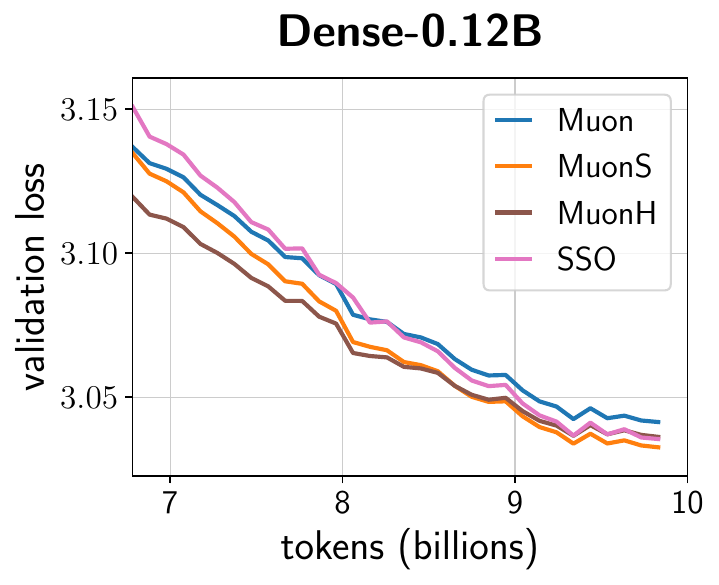}
    \end{subfigure}
    \hspace{0.3cm}
    \begin{subfigure}[b]{0.35\textwidth}
        \centering
        \includegraphics[width=\textwidth]{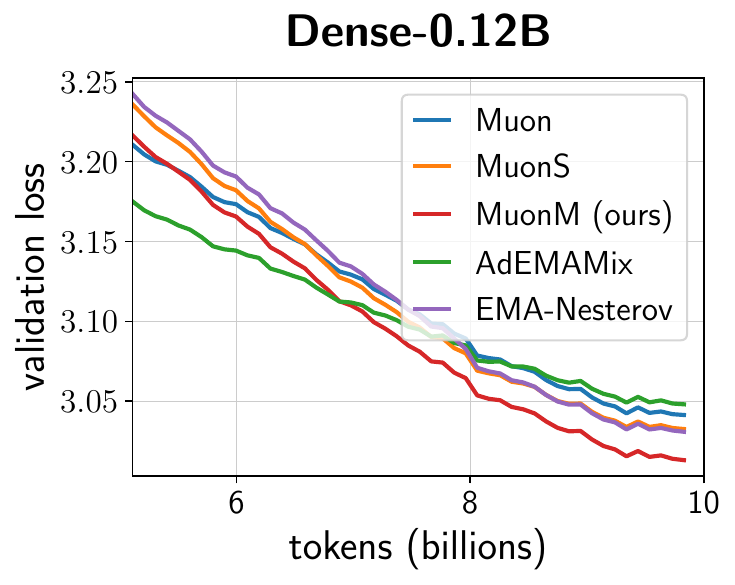}
    \end{subfigure}
    \caption{Comparison of MuonS and \muonours with other sphere-constrained methods (left) and  momentum-enhancement  approaches (right).}
    \label{fig:compare-120m}
\end{figure}

\paragraph{Main Results.} We systematically evaluate the acceleration benefits of \muonours against well-tuned MuonS and Muon baselines across dense models ranging from 0.12B to 1.4B parameters (Figures  \ref{fig:dense_val_loss} and \ref{fig:compare-120m}). Across all settings, \muonours \textbf{consistently achieves the lowest terminal loss}, outperforming MuonS and Muon by approximately 0.02 and 0.03, respectively.
\begin{figure}[!ht]
    \centering
    \includegraphics[width=0.32\linewidth]{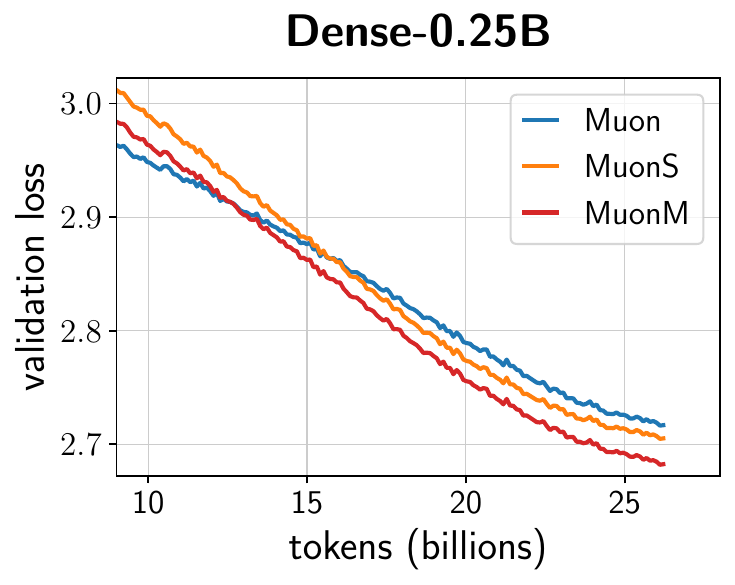}
    \hfill
    \includegraphics[width=0.32\linewidth]{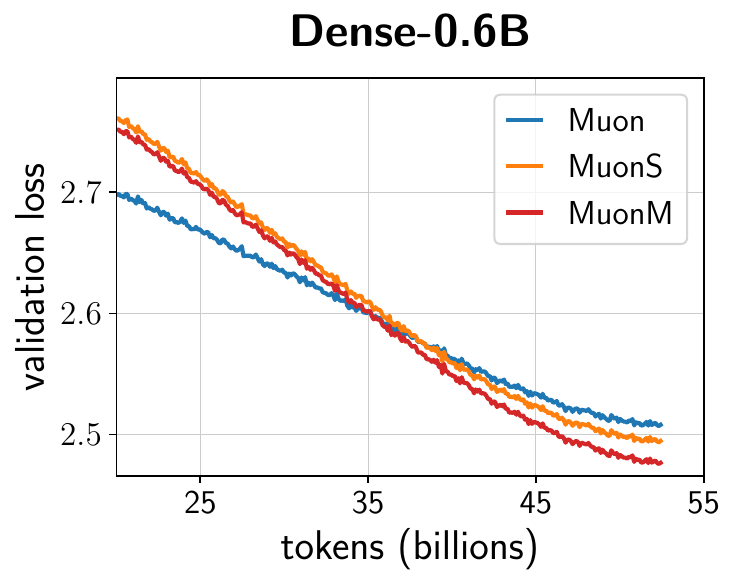}
    \hfill
    \includegraphics[width=0.32\linewidth]{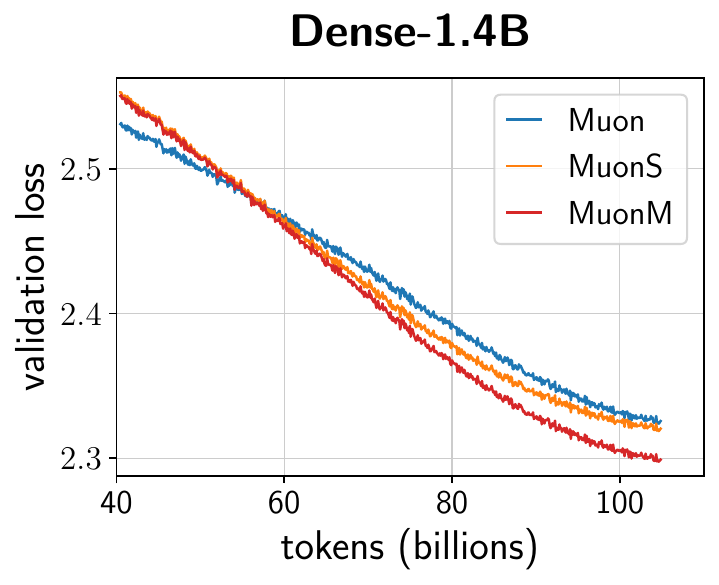}
\caption{Validation losses on dense models across different optimizers. \muonours consistently achieves lower loss than all baselines with a notable margin.}
    \label{fig:dense_val_loss}
\end{figure}

\paragraph{Scaling Behavior under Extended Training.} We evaluate Muon, MuonS, and \muonours on the 0.12B dense model across token budgets ranging from approximately 100 to 1000 TPP, with a peak learning rate sweep performed at each budget. The terminal losses are fitted using $L(N)=L_\infty+AN^{-\gamma}$ with the terminal loss $L(N)$, where $N$ is the number of training tokens. As shown in Figure~\ref{fig:long_horizon_cos_scaling_fit}, \muonours consistently achieves the lowest terminal loss across all budgets, and the \textbf{performance gap widens as the training budget increases}, demonstrating its potential  in long-horizon training. In addition, the searched optimal learning rate of MuonS (and thus \muonours) decays much more slowly (from at 1e-2 at 100 TPP to 7e-3 at 1000 TPP) compared to Muon, which decays from 5e-3
to 1e-3  over the same range (Table~\ref{tab:optimal_lr_sweep_overtrain}). \textbf{This mild shift in the optimal learning rate makes it more readily transferable across varying training horizons.}

\begin{figure}[h]
    \centering
    \includegraphics[width=0.66\linewidth]{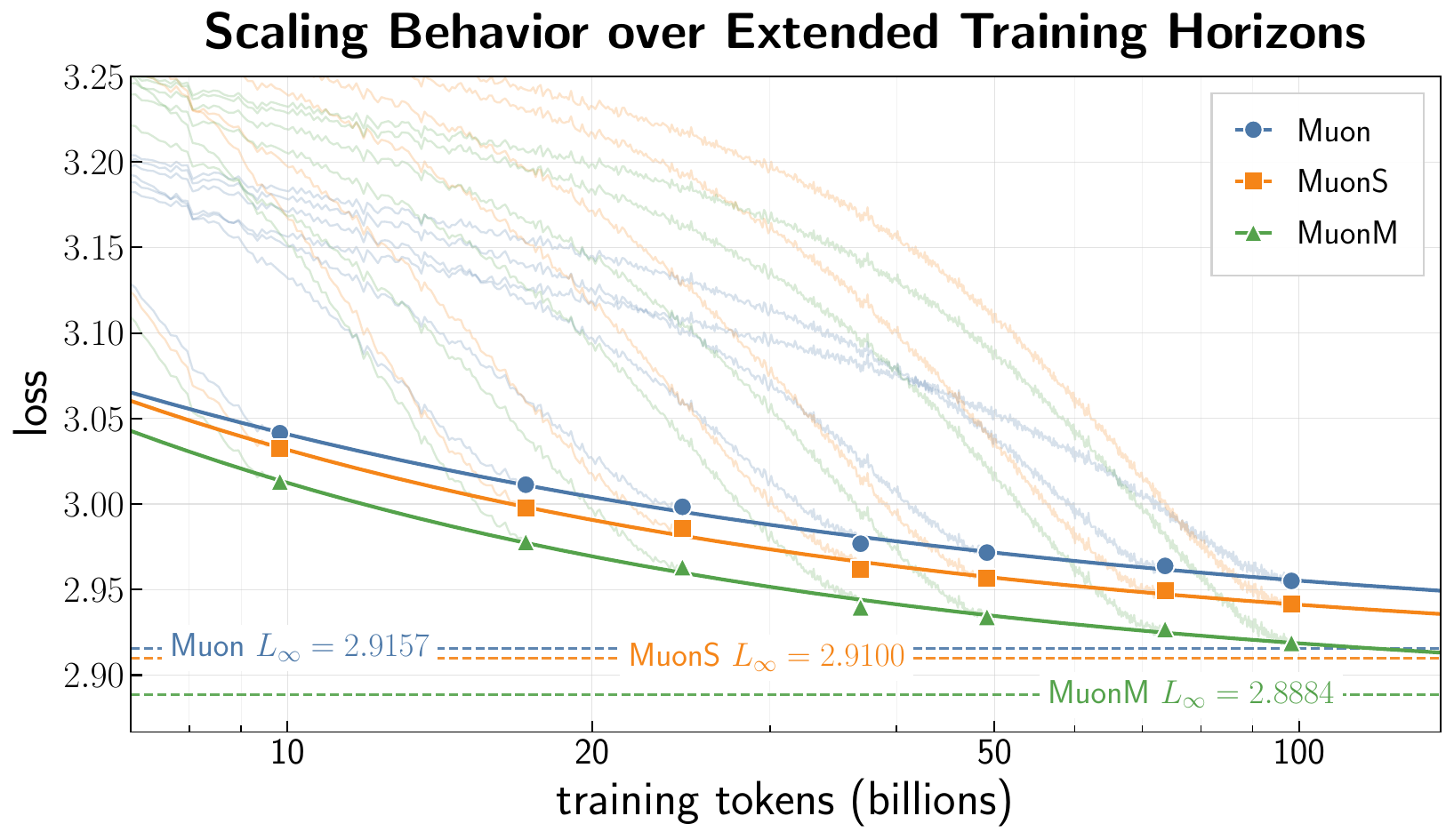}
    \caption{Terminal validation loss on the 0.12B model under extended training. Dots indicate the terminal loss at each token budget with the corresponding tuned learning rate, and the solid curves are the fitted power-law trends.}
    \label{fig:long_horizon_cos_scaling_fit}
\end{figure}

\subsection{Results on MoE Models}\label{sec:results-moe}

Figure~\ref{fig:moe_val_loss} shows the comparison of Muon, MuonS, and \muonours on MoE models with 0.64B (A0.13B) and 2.3B (A0.36B) parameters. \muonours consistently \textbf{achieves notable terminal loss reductions} over the well-tuned Muon and MuonS baselines. We further compare these methods under the WSD schedule on the 0.64B (A0.13B) model, where \muonours also attains the lowest terminal loss (Figure \ref{fig:wsd-moe-0.6b}). These results indicate the advantages of \muonours extend robustly from dense architectures to MoE models.
\begin{figure}[!ht]
    \centering
    \begin{subfigure}[b]{0.35\linewidth}
        \centering
        \includegraphics[width=\linewidth]{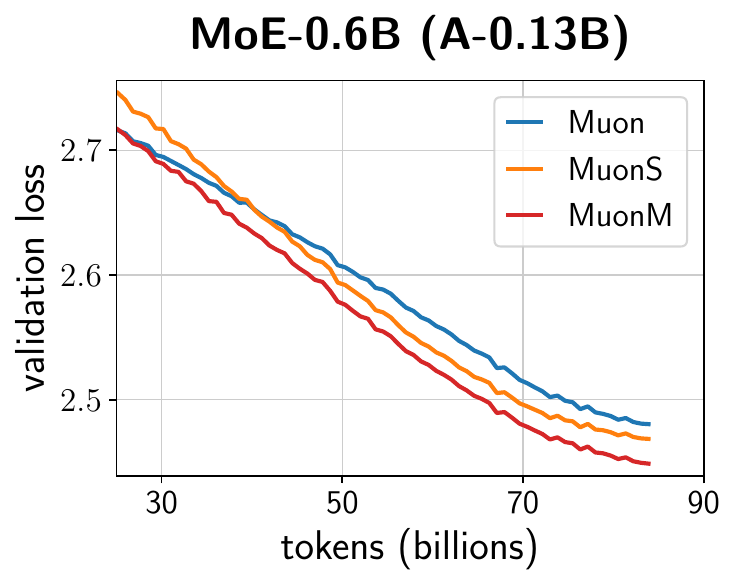}
        \label{fig:moe-0.64b}
    \end{subfigure}
    \hspace{0.3cm}   
    \begin{subfigure}[b]{0.35\linewidth}
        \centering
        \includegraphics[width=\linewidth]{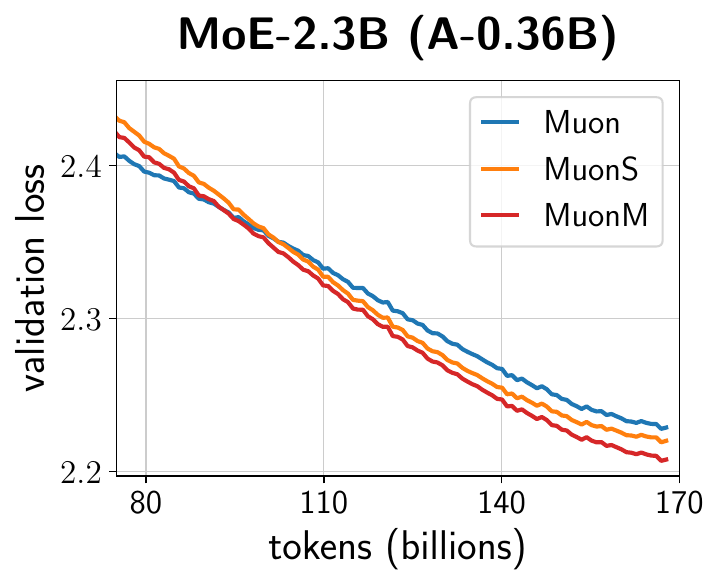}
        \label{fig:moe-2.3b}
    \end{subfigure}
    \caption{Validation losses on MoE models across different optimizers.}
    \label{fig:moe_val_loss}
\end{figure}
\section{Theoretical Analysis}

In this section, we analyze the training dynamics of flat-direction multiscale momentum and characterize its acceleration over vanilla momentum. Due to the inherent complexity of neural network loss landscapes, we resort to a  proxy that captures the essential features of the true optimization trajectory. Following recent practices that demonstrate the surprising effectiveness of quadratic models (e.g., linear regression with a power-decayed spectrum) approximating actual LLM training dynamics~\cite{paquette2024,lin2024scaling,li2025functional,meterez2026anytimepretraininghorizonfreelearningrate,meterez2026defensequadraticmodel}, we adopt this setting to approximate the pretraining landscape. Specifically, we instantiate our analysis under the feature space linear regression framework established in~\cite{li2025functional,wang2026fast}. 

\subsection{Problem Formulation}
Let $\mathcal{X}$ and $\mathcal{D}$ denote the data domain and data distribution, respectively. Let $\phi:\mathcal{X}\to \mathbb{R}^d$ be a feature map with $\phi(x)\sim \mathcal{N}(0,H)$ for $x\sim \mathcal{D}$, and let the target function be $f^\star(x)=\langle \phi(x), w^\star\rangle$ with target parameter $w^\star\in\mathbb{R}^d$. Without loss of generality, we assume $H=\diag(\lambda_1,\ldots,\lambda_d)$ is diagonal with entries $\lambda_1\geq \cdots \geq \lambda_d>0$.

Labels are generated according to $
y(x)=\langle\phi(x),w^\star\rangle+\epsilon$ with $ \epsilon\sim\mathcal N(0,\sigma^2)$.  We learn $f^\star$ using a student model $f(x,w)=\langle \phi(x), w\rangle$ with $w\in\mathbb{R}^d$ by minimizing the population risk
\begin{equation}
\mathcal{R}(w)=\frac{1}{2}\mathbb{E}_{x\sim\mathcal{D}}(f(x,w)-y(x))^2=\frac12 u^\top H u+\frac{1}{2}\sigma^2,
\end{equation}
where $u=w-w^\star$. For brevity, we define the excess risk $\mathcal{E}(w)=\frac{1}{2}u^\top H u$.

\begin{assumption}[Power-law model]\label{ass:power-law-short}
There exist constants $\nu>1$, $s>0$, such that the Hessian spectrum and target parameter show a power-decayed structure with 
\begin{equation}
    \lambda_j\asymp j^{-\nu},\quad |w_j^\star|^2\asymp j^{-1}\lambda_j^{s-1} \asymp j^{-1-(s-1)\nu},\quad 1\leq j\leq d.
\end{equation}
\end{assumption}
Here, the spectral exponent $\nu$ characterizes the model capacity of $\phi$, while $s$ measures the relative difficulty of the task with respect to the model capacity.

Let $\Sigma(w)$ denote the single-batch gradient noise covariance at $w$. Without loss of generality, we set the batch size to $1$. The stochastic gradient at iteration $k$, computed on the sampled data $x_k$, takes the form $
g_k = \nabla_w \frac{1}{2}(f(x_k,w_k)-y(x_k))^2 = H u_k + \xi_k$, 
where the gradient noise satisfies $\mathbb{E}_k\xi_k=0$ and has covariance $\mathbb{E}_k[\xi_k\xi_k^\top]=\Sigma(u_k)$.

We consider the multiscale momentum with  Nesterov-type momentum  
\begin{equation}\label{eq:algorithm}
\begin{aligned}
 m_k^{\#}&=(1-\alpha_{\#})m_{k-1}^{\#}+g_k,\quad \#\in\{\mathrm{fast},\mathrm{slow}\},\\
 w_{k+1}&=w_k-\eta_k
 \bigl(\chi Qm_k^{\mathrm{slow}}+\bd g_k+m_k^{\mathrm{fast}}\bigr),
\end{aligned}
\end{equation}
with initializations $w_0=m_{-1}^{\mathrm{fast}}=m_{-1}^{\mathrm{slow}}=0$ and hyper-parameters $0<\as<\af\leq1, \chi\geq0, \bd\geq0$. For an integer \(0\leq r<d\), we denote the flat-direction projection as $Q=\sum_{j=1}^d q_j e_je_j^\top$,  where $e_j$ is the eigenvector of $\lambda_j$ and $q_j=\mathbb I_{\{j>r\}}$ is the indicator for flat directions. For a given learning rate schedule, we define the intrinsic time as the accumulated learning rate $T_k:=\sum_{i=0}^{k-1}\eta_i$.

\subsection{Theoretical Results}

\begin{theorem}[Asymptotic Estimate of the Excess Risk]
\label{thm:discrete-fsl}
Define the constants $\atpg=\bd+\af^{-1}, 
 \aperp=\bd+\af^{-1}+\chi\as^{-1},p=2-\nu^{-1}$. We further let \(a_j=\atpg\) for \(j\leq r\) and \(a_j=\aperp\) for \(j>r\). 
Suppose Assumption~\ref{ass:power-law-short} holds and \(d\geq 2(r+1)\). 
Then for any $k$ and learning rate schedule $\{\eta_i\}$ such that \(\aperp T_k\lesssim d^\nu\) and additional conditions in \eqref{eq:discrete-regime} hold, the iterates \(w_k\) in 
\eqref{eq:algorithm} satisfy
\begin{equation}\label{eq:end-to-end-fsl}
 \E\cE(w_k)\asymp
 \underbrace{\SMS(T_k)}_{\mathrm{signal\,\, learning}}+\underbrace{\sigma^2\sum_{i=0}^{k-1}
 \eta_i^2\KMS(T_k-T_i)}_{\mathrm{noise\,\, accumulation}},
\end{equation}
where
\begin{equation}\label{eq:closed-kernels}
\begin{aligned}
 \SMS(t)
 &=\sum_{j=1}^{r}\lambda_j|w_j^\star|^2
       e^{-2\atpg\lambda_jt}
   +\bigl((r+1)^\nu+\aperp t\bigr)^{-s},\\
 \KMS(t)
 &=\sum_{j=1}^{r}\atpg^2\lambda_j^2
       e^{-2\atpg\lambda_jt}
   +\aperp^2\bigl((r+1)^\nu+\aperp t\bigr)^{-p}.
\end{aligned}
\end{equation}

\end{theorem}
\begin{remark}[Scope of $T_k$] 
The regime $T_k\lesssim d^\nu/\aperp$ 
 already encompasses the practically typical range $k\lesssim d$.  For $T_k\gg d^\nu/\aperp$, 
  matching upper and lower bounds for $\E \mathcal{E}(w_k)$   are not available within our current analysis.  The same difficulty also arises in the SDE analysis of analogous scenarios of \cite{li2025functional}. 
\end{remark}

Under a constant learning rate, the asymptotic estimate admits a simplified form as follows.

\begin{corollary}[Constant Learning Rate]
\label{cor:constant-schedule}
Let \(\eta_i\equiv\eta\) and \(r\geq4\). 
Suppose the conditions of Theorem~\ref{thm:discrete-fsl} hold.  If we further have
\begin{equation}\label{eq:constant-schedule-window}
\atpg^{-1}r^\nu(1+\log(\aperp/\atpg))
 \lesssim \eta k \lesssim \aperp^{-1}d^\nu,
\end{equation}
then
\begin{equation}\label{eq:constant-schedule-final}
 \E\cE(w_k)\asymp(\aperp\eta k)^{-s}
 +\eta\sigma^2\left(
 \atpg\sum_{j=1}^r j^{-\nu}
 +\aperp\sum_{j=r+1}^d j^{-\nu}\right)
 \overset{\Delta}{=}\mathcal{E}_\chi(\eta;k).
\end{equation}
\end{corollary}

Building on this formulation, we demonstrate that for the same \(\af,\bd\) (while allowing $\eta$
 to be tuned) in \eqref{eq:algorithm}, using $\chi>0$ (i.e., incorporating slow momentum) leads to accelerated convergence. A direct insight  is that with a larger $\chi$, taking a proportionally smaller  $\eta$ can  \textit{reduce both the signal learning error and noise accumulation}. We detail the mechanism underlying this improvement below.

\begin{corollary}[Acceleration by Slow Momentum]
\label{cor:noise-aware-acceleration}
Suppose the conditions of Corollary~\ref{cor:constant-schedule} hold. Define the following quantities as functions of $\chi$:
\begin{equation}
a_\chi=\bd+\frac{1}{\af}+\frac{\chi}{\as},\,\,  \kappa_\chi=\frac{a_\chi}{a_0},\,\,N_\chi=a_0(\Lambda_P+\kappa_\chi\Lambda_Q),\,\, \Gamma_\chi
= \frac{\Lambda_P+\Lambda_Q}
        {\Lambda_Q+\Lambda_P/\kappa_\chi},
\end{equation}
where $\Lambda_P=\sum_{j=1}^r j^{-\nu},\,\Lambda_Q=\sum_{j=r+1}^d j^{-\nu}$.   Then the following statements hold:

\begin{enumerate}
\item \textbf{(Proxy‐optimal loss reduction).}  
The optimal learning rate $\eta_{\chi,k}^{\star}:=\operatorname{argmin}_\eta \mathcal{E}_\chi(\eta;k)$ satisfies
\begin{align}
\eta_{\chi,k}^{\star}
 \asymp \left[\frac{s}{\sigma^2 a_0N_\chi (a_\chi k)^s}\right]^{\frac{1}{s+1}},  \, 
\mathcal{E}_\chi(\eta_{\chi,k}^{\star}; k)
\asymp 
\left[\frac{\sigma^2}{k}
\left(\Lambda_Q + \frac{\Lambda_P}{\kappa_\chi}\right)
\right]^{\frac{s}{s+1}},
\end{align}
where the signal learning term, noise accumulation term, and risk $\mathcal{E}_\chi(\eta_{\chi,k}^{\star}; k)$ are balanced at the same asymptotic order (up to constants). 
Consequently, the loss reduction ratio is $\frac{\mathcal{E}_\chi(\eta_{\chi,k}^{\star}; k)}
     {\mathcal{E}_0(\eta_{0,k}^{\star}; k)}
\asymp \Gamma_\chi^{-s/(s+1)}$. 

\item \textbf{(Iteration complexity speedup).}  
For any $\epsilon>0$, let $k_\chi(\epsilon)$  solve $\mathcal{E}_\chi(\eta_{\chi,k}^{\star}; k_\chi(\epsilon)) = \epsilon$. Then we have
\begin{equation}
\frac{k_\chi(\epsilon)}
     {k_0(\epsilon)}
\asymp \Gamma_\chi^{-1}. \label{eq:noise-aware-complexity}
\end{equation}
\end{enumerate}
\end{corollary}
The proof follows by direct computation with Weighted AM-GM Inequality. According to the constraints in  Corollary~\ref{cor:constant-schedule} and Theorem \ref{thm:discrete-fsl}, the optimal learning rate   $\eta_{\chi,k}^{\star}$ is admissible if $k$ satisfies $\frac{r^\nu}{a_0}(1+\log \kappa_\chi)
 \lesssim \left(\frac{k}{\sigma^2 N_\chi (a_\chi )^s}\right)^{\frac{1}{s+1}} \lesssim  \frac{d^\nu}{a_\chi}$ and $d$ is sufficiently large. The corollary explicitly quantifies the speedup induced by slow momentum, which becomes more pronounced for larger values of $\chi/\alpha_{\text{slow}}$.

\section{Conclusion and Discussion}
In this work, we propose  a flat-direction multiscale momentum method with sphere constraints to accelerate LLM pretraining. The sphere constraint serves as a more effective foundation than standard weight decay for unleashing the acceleration potential of flat-direction multiscale momentum, by resolving the effective learning rate collapse inherent to the multiscale design. Extensive experiments show that our proposed method (\muonours) consistently accelerates Muon with substantial loss reductions across diverse model architectures and training setups. Our results also indicate that the sphere-constraint technique holds promise as a general building block for developing more efficient optimizers.

Several limitations and open questions exist in this study. First, the optimal learning rate schedule for sphere-constrained methods  remains an open question. Second, we have not investigated preconditioner design for the slow momentum, which we posit could further accelerate training.

\section{Acknowledgment}
We thank Prof. Lexing Ying and Prof. Weijie Su for helpful discussions.

\clearpage

\bibliographystyle{plainnat}
\bibliography{ref}

\clearpage

\beginappendix

\startcontents[sections]
\printcontents[sections]{l}{1}{\setcounter{tocdepth}{2}}

\newpage
\section{Algorithm Details}\label{app:alg-details}

\subsection{Estimating Flat Directions}\label{app:flat-proj}

\begin{wrapfigure}{r}{0.5\textwidth}   
    \centering
    \includegraphics[width=\linewidth]{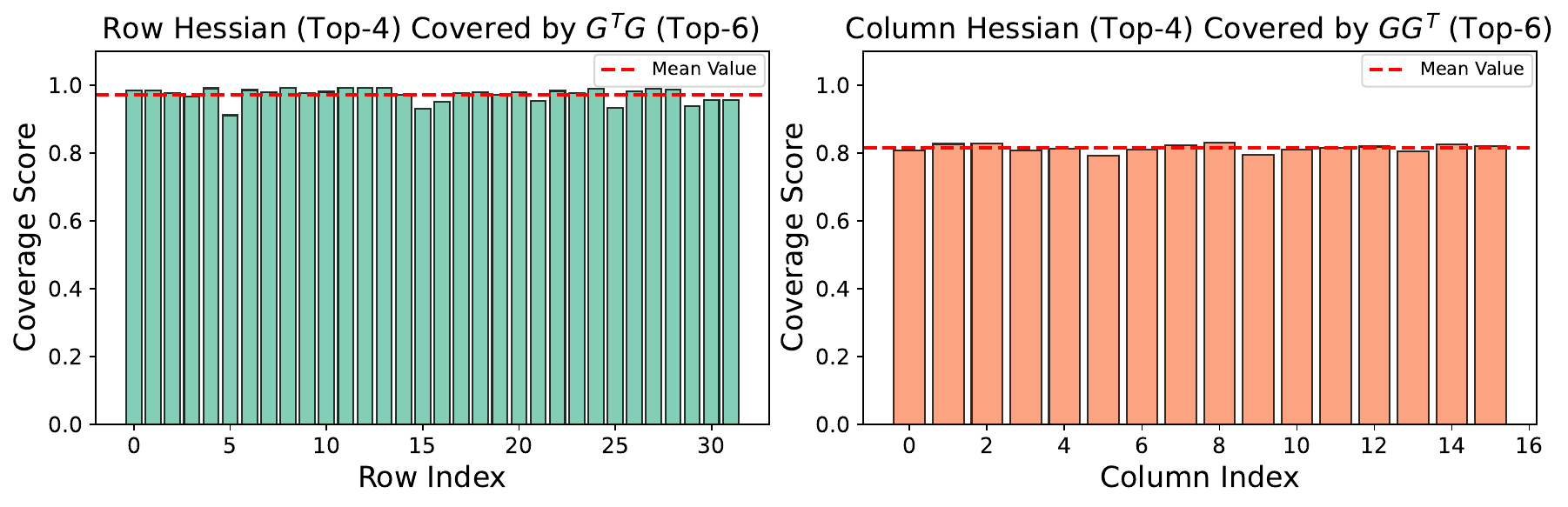}
    \caption{(Figure 3 in~\cite{zhu2026acceleratingllmpretrainingflatdirection}) Coverage of the top eigenspaces of row (column) Hessians by those of $G^\top G$ ($G G^\top$) for the \texttt{up\_proj} block of an FFN layer in a toy dense language model, where $G\in \mathbb{R}^{32\times 16}$. A higher score indicates a greater degree of containment.}
    \label{fig:top-align}
\end{wrapfigure}

In this study, we focus on separating sharp and flat directions within each block.
The construction of the flat-direction projection is motivated by the empirical observation in~\cite{zhu2026acceleratingllmpretrainingflatdirection} that, for a transformation block $W \in \mathbb{R}^{m \times n}$, the top eigenspaces of $G^\top G$ and $G G^\top$ are highly aligned with the per‑row and per‑column Hessians, respectively (Figure~\ref{fig:top-align}).
Therefore, at the $k$-th iteration, we estimate the top singular space of the Nesterov‑type momentum $\widetilde{M}_k^{\text{fast}} := (1-\alpha_{\text{fast}}) M_k^{\text{fast}} + G_k$, which serves as a stabilized proxy for the gradient $G_k$, as the sharp directions.
Specifically, we maintain a projection matrix $P^{\text{right}} \in \mathbb{R}^{n \times d}$ ($d \ll n$, e.g., $d =0.1n$ in our experiments) that captures the sharp directions of the row‑wise Hessian.
At iteration $k$, $P^{\text{right}}$ is updated via online power iteration to track the top $d$ right singular spaces:
\[
P^{\text{right}}_k = \operatorname{msign}\left(\left(\widetilde{M}_k^{\text{fast}}\right)^\top\widetilde{M}_k^{\text{fast}}P^{\text{right}}_{k-1}\right),
\]
and the corresponding projection for the column‑wise Hessian is then computed as
\[
P^{\text{left}}_k = \operatorname{msign}\bigl(\widetilde{M}_k^{\text{fast}} P^{\text{right}}_k\bigr) \in \mathbb{R}^{m \times d}.
\]
Finally, for any $U\in \mathbb{R}^{m \times n}$, the flat‑direction projection is given by
\[
\mathcal{Q}_k(U) = \mathcal{P}_{T_{W_k}\mathbb{S}_R}\left(\bigl(I_m - P^{\text{left}}_k (P^{\text{left}}_k)^\top\bigr) \, U \, \bigl(I_n - P^{\text{right}}_k (P^{\text{right}}_k)^\top\bigr)\right),
\]
where $I$ denotes the identity matrix, $\mathcal{P}_{T_{W_k}\mathbb{S}_R}$ is the projection to the tangent space at $W_k$. Similarly, \cite{robert2025ldadam} applies the power method to  estimate the top eigenspaces. Here we employ the more efficient msign operator to estimate the subspace projection, replacing the computationally expensive QR decomposition.

For the embedding and output matrices $W \in \mathbb{R}^{V \times d_{\text{model}}}$, where $V$ and $d_{\text{model}}$ denote the vocabulary size and model dimension, respectively, we simply select the $d$ rows with the largest norm of the second moment in Adam. This choice is motivated by~\cite{zhang2025adammini}, which shows that the token‑wise Hessian exhibits an approximately block‑diagonal structure across rows, each row (token) corresponding to a dense sub‑block. In practice we set $d \ll V$ (e.g., $d = 0.1V$ in our experiments). For vector and scalar parameters, we do not apply flat‑direction acceleration.

\paragraph{Computation and Memory Overhead of Flat-Direction Projection.} The computational cost (in FLOPs) for computing $M\in \mathbb{R}^{m\times n}$  in computing $\mathsf{msign}(M)$ via Newton-Schulz (NS) iterations is approximately  $6sr^2c$, where $c=\max\{m,n\}, r=\min\{m,n\}$, and $s$ denotes the number of iterations. Consequently, compared to performing NS iterations on the full-dimensional matrix of size  $m\times n$, employing the right and left projection matrices of sizes  $m\times d$ and $n \times d$ reduces the computation to approximately $(d/n)^2$ of the original cost  (e.g., $0.1\times 0.1=0.01$) , rendering it nearly negligible. The additional memory overhead primarily stems from storing the right projection matrix of size $n\times d$, which is significantly smaller than that required for the momentum.

\subsection{Update Rules of \muonours for Different Parameter Types}\label{app:ours-alg-adam}

\paragraph{Notations.}
For a matrix $A$, we use $A[i,:]$ to denote its $i$-th row. For a vector $v$, $v[i]$ denotes its $i$-th element.

We categorize the parameters into the following three types based on their structure:
\begin{enumerate}

\item \textbf{Matrix parameters}: Transformer weight matrices, such as Q, K, V, O, and FFN blocks. We apply \muonours (Algorithm \ref{alg:ours}) to these.
\item \textbf{Vector parameters}: biases, learnable vectors in RMSNorm layers, each row of embedding/output layers (corresponding to each token),   and the learnable grouped radii introduced by the sphere constraint. For these, we use Adam with sphere constraints, with slow momentum applied to partial vector blocks.
\item \textbf{Scalar parameters}: native architectural scalars and the learnable radii from the sphere constraint. We directly apply Adam to these.
\end{enumerate}
The specific update rules are as follows.

\begin{itemize}
\item \textbf{Scalar.}
It is initialized as $1$ and trained directly by Adam without weight decay. The learning rate is the same as other vector and matrix parameters. In practice, if several scalars are interchangeable (e.g., the scalar multipliers of a weight matrix and its preceding RMSNorm, of the query and key matrices in attention, or of the up and down matrices in the FFN), we group them into a single learnable parameter.

\item \textbf{Vector.}
For a vector $\gamma \in \mathbb{R}^d$, we parameterize it as $\gamma = r\,\widehat{\gamma}$, where $r \in \mathbb{R}$ and $\widehat{\gamma} \in \mathbb{R}^d$ has a fixed norm $R$.
At the $k$-th iteration, it is updated by
\begin{equation}
\begin{aligned}
   m_{k}  &= (1-\alpha_{\text{fast}}^{\text{adam}})\,\mathcal{T}_{\widehat{\gamma}_{k-1}\to \widehat{\gamma}_{k}} (m_{k-1}) + \mathcal{P}_{T_{\widehat{\gamma}_k}\mathbb{S}_{R}}\,{g}_k, \\
   v_{k}  &= (1-\beta_{\text{adam}})\,v_{k-1} + \beta_{\text{adam}}\,{g}_k\odot g_k, \\
   \widehat{\gamma}_{k+1} &= \mathsf{Norm}_{R}\Bigl(\widehat{\gamma}_{k}-\eta_k\mathcal{P}_{T_{\widehat{\gamma}_{k}}\mathbb{S}_{R}}\Bigl(\frac{m_k}{\sqrt{v_k}+\epsilon}\Bigr)\Bigr),
\end{aligned}
\end{equation}
where $g_k$ is the stochastic gradient.

\item \textbf{Embedding and Output.}
They are re-parameterized as $W = \operatorname{diag}(\gamma)\,\widehat{W}$, and we impose row-wise sphere constraints on $\widehat{W}\in \mathbb{R}^{V\times d}$, where the $i$-th row has a fixed Frobenius norm $R_i$. The vector $\gamma$ is treated as a vector parameter and follows the vector update rules above.
At the $k$-th iteration, for each row $i$ we update the momenta as
\begin{equation}
    M_k^{\#}[i,:] = (1-\alpha_{\#}^{\text{adam}})\,\mathcal{T}_{W_{k-1}[i,:]\to W_{k}[i,:]}\,(M_{k-1}^{\#}[i,:]) + \mathcal{P}_{T_{W_k[i,:]}\mathbb{S}_{R_i}}\,G_k[i,:], \quad \#\in\{\text{fast},\text{slow}\},
\end{equation}
where $G_k$ is the stochastic gradient, and parallel transport is performed row‑wise.
Then the flat‑direction projection $\mathcal{Q}_k$ is computed as described in Section~\ref{app:flat-proj}, yielding a vector $q_k \in [0,1]^V$, where a value of $1$ indicates that the corresponding row is selected as a flat direction.
The update direction for each row is then
\begin{equation}
U_k[i,:] = \mathsf{Norm}_{R_i}\Bigl(\mathcal{P}_{T_{W_k[i,:]}\mathbb{S}_{R_i}}\Bigl(\frac{\widetilde{M}_k^{\mathrm{fast}}[i,:]}{\sqrt{V_k[i,:]}+\epsilon}\Bigr)\Bigr) + \chi\,q_k[i]\,\mathsf{Norm}_{R_i}\bigl(M_k^{\mathrm{slow}}[i,:]\bigr),
\end{equation}
where the second moment $V_k$ is maintained as $V_{k} = (1-\beta_{\text{adam}})\,V_{k-1} + \beta_{\text{adam}}\,G_k\odot G_k$.
Finally, the parameters are updated row‑wise via
\begin{equation}
    W_{k+1}[i,:] = \mathsf{Norm}_{R_i}\bigl(W_k[i,:] - \eta_k\,U_k[i,:]\bigr).
\end{equation}
\end{itemize}

\section{Experimental Details}\label{app:exp-details}
We use 6 Newton-Schulz (NS) iterations with coefficients in  Polar Express~\cite{amsel2025polarexpressoptimalmatrix} to implement the msign operation, which achieves higher precision than the method in \cite{jordan2024muon}. The peak learning rate is searching in $\{$5e-4, 7e-4, 1e-3, 2e-3, 3e-3, 5e-3, 7e-3, 1e-2, 1.5e-2, 2e-2$\}$ for all optimizers. 

\subsection{Experimental Details of Section~\ref{sec:obt-mts}}\label{app:exp-obt-mts}
In our hyperparameter search, the optimal learning rate for Muon was found to be \(0.005\). The sharp subspace ratio was set to \(10\%\), following the procedure described in Appendix~\ref{app:flat-proj}. Beyond the learning rate, we set $\alpha_{\text{fast}}=0.05$, searched over \(\alpha_{\text{slow}} \in \{0.005, 0.002, 0.001, 0.0005\}\) and \(\chi \in \{0.1, 0.2, 0.5\}\). The final best configuration is \(\text{lr} = 0.005\), \(\alpha_{\text{slow}} = 0.002\), and \(\chi = 0.1\).

\subsection{Experimental Details of Section~\ref{sec:results-dense}}

\paragraph{Comparing Different Optimizers on the 0.12B Dense Model.}
We use the cosine decay learning rate schedules for all optimizers. For MuonH and SSO, the optimal learning rates found by grid search are $1\times10^{-2}$ and $5\times10^{-3}$, respectively.
We also tuned the radius for SSO over $\{1.0, 2.0, 5.0\}$ and found that $2.0$ yields the lowest loss.
For AdEMAMix, we searched $\beta_3 \in \{0.01, 0.001, 0.0001\}$ and $\alpha \in \{1, 3, 5, 8\}$, obtaining the best configuration $(\mathrm{lr}, \beta_3, \alpha) = (3\times10^{-3}, 1\times10^{-3}, 5.0)$.
Following~\cite{pagliardini2025the}, the warmup iterations $T_\alpha$ and $T_{\beta_3}$ are set to the total number of iterations.
For EMA-Nesterov, we searched $\gamma \in \{0.99, 0.995, 0.999\}$ and $\beta_{\max} \in \{0.3, 0.5, 0.7, 0.9\}$, and selected $(\mathrm{lr}, \beta_{\max}, \gamma) = (3\times10^{-3}, 0.9, 0.99)$.
The coefficient schedule of $\beta_t$ involves two transition points $T_w$ and $T_r$, and we adopted the recommended values $T_w = 0.3\,T$ and $T_r = 0.8\,T$ from~\cite{yau2026emanesterovstabilizingnesterovslookahead}, where $T$ denotes the total number of iterations.

\textbf{Muon Configuration.} We follow the standard Muon setup described in \cite{liu2025muon}.  Muon is applied exclusively to all 2D transformer weight matrices, while AdamW handles all other parameters (RMSNorm, embeddings, and the output layer). For a Muon‑updated matrix of shape $m\times n$, the learning rate is scaled by $0.2\sqrt{\max\{m,n\}}$
to align the RMS of the updates with that of AdamW. We employ Nesterov momentum with decay rate $\theta_{\text{muon}}=0.95$, weight decay $\lambda=0.1$, and a   gradient clipping threshold of $1.0$.

\textbf{\muonours Configuration.} Unless otherwise specified, we uniformly adopt $\chi=0.2$, $\alpha_{\text{fast}}=0.05$, $\alpha_{\text{slow}}=0.001$, and set the top $10\%$ subspace as sharp directions for each block using \muonours.  Given the rapid gradient variations during the early training phase, we apply a warmup mechanism to the slow momentum coefficient $\alpha_{\text{slow}}$ as
\[
\alpha_{\text{slow},t}=
\begin{cases}
\alpha_{\text{fast}}, & t\le T_{\text{warmup}},\\[6pt]
\displaystyle
\frac{\alpha_{\mathrm{fast}}}
{1+\frac{\alpha_{\mathrm{fast}}/\alpha_{\mathrm{slow}}-1}{T_{\text{warmup}}}(t-T_{\text{warmup}})},
& T_{\text{warmup}}<t\le 2T_{\text{warmup}},\\[12pt]
\displaystyle
\alpha_{\mathrm{slow}},
& t>2T_{\text{warmup}},
\end{cases}
\]
where $T_{\text{warmup}}$ is the warmup iterations for the learning rate.

\paragraph{Models.}  We utilize two popular classes of LLM architectures in the experiments:

\begin{itemize}
    \item \textbf{Dense.} We adopt a dense decoder-only Transformer architecture that uses Rotary Positional Encoding (RoPE) \cite{su2024roformer}, SwiGLU feed-forward layers, RMSNorm, and a Pre-Norm design.  Multi-head attention (MHA) is employed in self attention. Detailed configurations are provided in Table \ref{table:dense_model_config_and_max_lrs}.

\item \textbf{MoE.} The Mixture-of-Experts (MoE) architecture we use is built on a decoder-only Transformer with Multi-head Latent Attention (MLA) \cite{deepseekai2024deepseekv2strongeconomicalefficient}. Each MoE layer employs one shared expert together with partially activated experts per token. The first layer is configured as dense. We use an auxiliary-loss-free load balancing strategy. RMSNorm with a Pre-Norm design and RoPE (applied within the decoupled rotary components of MLA) are used throughout. Detailed configurations are provided in Table \ref{table:moe_model_config_and_max_lrs}.
\end{itemize}

\begin{table}[!ht]
    \centering
    \renewcommand{\arraystretch}{1.25}
\caption{\small Dense LLaMA model configurations. All models are trained with a default sequence length of 1024.}
    \label{table:dense_model_config_and_max_lrs}
    \begin{small}
        \begin{tabular}{l|c|c|c|c|c|c|c|c|c}
            \hline
            \multirow{2}{*}{Acronym} & \multirow{2}{*}{Size} &
            \multirow{2}{*}{$d_{\mathrm{model}}$} & \multirow{2}{*}{$d_{\mathrm{FFN}}$} &
            \multirow{2}{*}{$n_{\text{head}}$} & \multirow{2}{*}{$n_{\text{layer}}$} &
            \multirow{2}{*}{Batch size} &
            \multicolumn{3}{c}{\texttt{lr\_max}} \\
            \cline{8-10}
            & & & & & & & Muon & MuonS & \muonours \\
            \hline\hline
            Dense (0.12B) & 0.12B & 768  & 2048 & 12 & 6  & 480  & 5e-3 & 1e-2 & 7e-3 \\
            Dense (0.25B) & 0.25B & 768  & 2048 & 16 & 24 & 512 & 3e-3 & 7e-3 & 4.5e-3 \\
            Dense (0.6B)  & 0.6B  & 1280 & 3412 & 20 & 24 & 512 & 1e-3 & 5e-3 & 3.5e-3 \\
            Dense (1.4B)  & 1.4B  & 2048 & 5461 & 32 & 24 & 512 & 7e-4 & 3e-3 & 2e-3 \\
            \hline
        \end{tabular}
    \end{small}
\end{table}


\begin{table}[!ht]
    \centering
    \renewcommand{\arraystretch}{1.25}
    \caption{\small MoE model configurations. All models are trained with a
    default sequence length of 4096 and batch size of 1024.}
    \label{table:moe_model_config_and_max_lrs}
    \begin{small}
    \resizebox{\textwidth}{!}{%
        \begin{tabular}{l|*{16}{c|}c}
            \hline
            \multirow{2}{*}{Acronym}
            & \multirow{2}{*}{Size}
            & \multirow{2}{*}{Activated Size}
            & \multirow{2}{*}{$d_{\mathrm{model}}$}
            & \multirow{2}{*}{$n_{\mathrm{head}}$}
            & \multirow{2}{*}{$n_{\mathrm{layer}}$}
            & \multicolumn{3}{c|}{$d_{\mathrm{FFN}}$}
            & \multirow{2}{*}{$n_{\mathrm{expert}}$}
            & \multicolumn{5}{c|}{MLA Dimensions}
            & \multicolumn{3}{c}{\texttt{lr\_max}} \\
            \cline{7-9}
            \cline{11-15}
            \cline{16-18}
            & & & & & 
            & Dense
            & Routed Expert
            & Shared Expert
            &
            & $r_q$
            & $r_{kv}$
            & $d_{qk}$
            & $d_{qk}^{\text{rope}}$
            & $d_v$
            & Muon
            & MuonS
            & \muonours \\
            \hline\hline

            MoE-0.6B
            & 0.64B & 0.13B
            & 704 & 11 & 12
            & 3584 & 672 & 672
            & 4/32
            & 704 & 240 & 128 & 64 & 128
            & 7e-3 & 1e-2 & 7e-3 \\

            MoE-2.3B
            & 2.27B & 0.36B
            & 1280 & 16 & 12
            & 7168 & 704 & 1536
            & 4/64
            & 1536 & 512 & 128 & 64 & 128
            & 3e-3 & 7e-3 & 4.5e-3 \\
            \hline
        \end{tabular}%
    }
    \end{small}
\end{table}

\paragraph{Learning Rate Schedulers.}
We evaluate two common learning rate  schedules with $\min\{1000,T/50\}$  warm-up steps in all  experiments ($T$ is the total steps):
\begin{itemize}
\item \texttt{cos} (cosine decay):
a linear warm-up to peak \texttt{lr\_max}, followed by cosine decay to a terminal learning rate $\texttt{lr\_min}=0.05\times \texttt{lr\_max}$.
\item \texttt{wsd} (warmup-stable-decay):
a linear warm-up to \texttt{lr\_max}, followed by a stable phase with constant  learning rate of \texttt{lr\_max} (up to 70\% of total steps), and then a minus-sqrt ($1-\sqrt{t}$) decay to $\texttt{lr\_min}=0.05\times \texttt{lr\_max}$ over the final 30\% steps. The searched peak learning rates for Muon and MuonS on the 0.64B MoE model are 5e-3 and 7e-3  respectively. Then we directly set the peak learning rate of \muonours to  4.5e-3 $\approx$ 7e-3$\times2/3$.
\end{itemize}

\begin{table}[h]
\centering
\caption{Optimal peak learning rates for Muon and MuonS  at various training steps on the 0.12B dense model. The peak learning rate of \muonours is set to approximately 2/3 of that of MuonS without search.}
\label{tab:optimal_lr_sweep_overtrain}
\begin{tabular}{lccccccc}
\toprule
Iterations & 20k & 35k & 50k & 75k & 100k & 150k & 200k \\
\midrule
Muon      & 5e-3 & 3e-3 & 2e-3 &2e-3  &2e-3  &  2e-3& 1e-3 \\
MuonS     & 1e-2 & 1e-2 & 7e-3 & 7e-3 &  7e-3 & 7e-3 & 7e-3 \\
\muonours & 7e-3 &7e-3  & 4.5e-3 & 4.5e-3 & 4.5e-3 & 4.5e-3 & 4.5e-3 \\
\bottomrule
\end{tabular}
\end{table}

\section{Additional Experimental Results}

\subsection{Results with WSD Learning Rate Schedules}
Figure~\ref{fig:wsd-moe-0.6b} shows the loss trajectories of Muon, MuonS, and \muonours under the WSD learning rate schedule. \muonours achieves the lowest terminal loss. The peak learning rates for Muon, MuonS, and \muonours are 0.005, 0.007 and 0.0045, respectively.
\begin{figure}[H]
    \centering
    \includegraphics[width=0.3\linewidth]{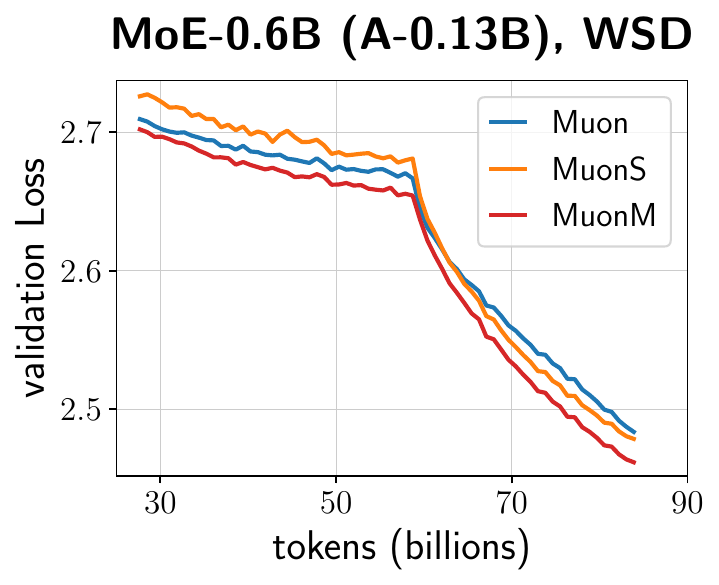}
    \caption{Validation losses of different optimizers on the 0.6B MoE model under the WSD schedule.}
    \label{fig:wsd-moe-0.6b}
\end{figure}

\section{A Continuous-Time Perspective for Unifying Enhanced Momentum Methods}
\label{sec:multiscale-ode-unification}

In this section, we establish the connections among AdEMAMix, GPA, EMA-Nesterov, SODA, and our method: all share the same continuous-time ODE form when preconditioning is ignored. For GPA and EMA-Nesterov, we assume the underlying base optimizer is Heavy Ball momentum. For brevity, our ODE derivation is conducted under gradient-noise-free conditions and with constant hyper-parameters.

Let $g_k= \nabla f(w_k)$ denote the gradient and $P_k$ the preconditioner at the $k$-th iteration. Let $h>0$ denote the discretization interval and $t_k=kh$. For the  EMA coefficient $\rho$, we use the scaling $\rho=e^{-h/\tau}\approx 1-\frac{h}{\tau}$ 
where $\tau>0$ is the corresponding time  constant. The algorithmic details and continuous limits are as follows.

\paragraph{Our Approach}  We illustrate our approach using \muonours as an example. Omitting the sphere constraint and Muon preconditioning yields the following discrete-time recursion (constants have been absorbed into step sizes for notational simplicity):
\begin{equation}\label{eq:alg-discrete}
\begin{aligned}
 m_k^{\#}&=(1-\alpha_{\#})m_{k-1}^{\#}+\alpha_{\#}g_k,\quad \#\in\{\mathrm{fast},\mathrm{slow}\},\\
 w_{k+1}&=w_k-\eta_k
 \bigl(\chi m_k^{\mathrm{slow}}+\bd g_k+m_k^{\mathrm{fast}}\bigr).
\end{aligned}
\end{equation}
Under a minor change of notation, the continuous-time counterpart takes the form
\begin{equation}
\label{eq:ours-state-ode}
\begin{aligned}
    \dot m_{\mathrm f}
    &=-\tau_{\mathrm f}^{-1} m_{\mathrm f}+\nabla f(w),\\
    \dot m_{\mathrm s}
    &=-\tau_{\mathrm s}^{-1} m_{\mathrm s}+\nabla f(w),\\
    \dot w
    &=-a_{\mathrm f}m_{\mathrm f}-a_{\mathrm s}m_{\mathrm s}-a_{\mathrm g} \nabla f(w).
\end{aligned}
\end{equation}
Let \(g_t=\nabla f(w_t)\). The   momentum equations can be written as$
\left(1+\tau_{\mathrm f}\frac{d}{dt}\right)m_{\mathrm f}
=\tau_{\mathrm f}g_t,
\,
\left(1+\tau_{\mathrm s}\frac{d}{dt}\right)m_{\mathrm s}
=\tau_{\mathrm s}g_t$.

Applying the operator
$
\left(1+\tau_{\mathrm f}\frac{d}{dt}\right)
\left(1+\tau_{\mathrm s}\frac{d}{dt}\right)
$
to  
$
\frac{d}{dt}w_t
=-a_{\mathrm f}m_{\mathrm f}
-a_{\mathrm s}m_{\mathrm s}
-a_{\mathrm g}g_t,
$
we obtain
\[
\begin{aligned}
0={}&
\left(1+\tau_{\mathrm f}\frac{d}{dt}\right)
\left(1+\tau_{\mathrm s}\frac{d}{dt}\right)
\frac{d}{dt}w_t
+a_{\mathrm f}\tau_{\mathrm f}
\left(1+\tau_{\mathrm s}\frac{d}{dt}\right)g_t\\
&+a_{\mathrm s}\tau_{\mathrm s}
\left(1+\tau_{\mathrm f}\frac{d}{dt}\right)g_t
+a_{\mathrm g}
\left(1+\tau_{\mathrm f}\frac{d}{dt}\right)
\left(1+\tau_{\mathrm s}\frac{d}{dt}\right)g_t.
\end{aligned}
\]
Expanding the operator products gives the third-order ODE
\[
\begin{aligned}
0={}&
\tau_{\mathrm f}\tau_{\mathrm s}\frac{d^3}{dt^3}w_t
+\left(\tau_{\mathrm f}+\tau_{\mathrm s}\right)
\frac{d^2}{dt^2}w_t
+\frac{d}{dt}w_t
+\left(a_{\mathrm f}\tau_{\mathrm f}
+a_{\mathrm s}\tau_{\mathrm s}+a_{\mathrm g}\right)
\nabla f(w_t)\\
&+\left[\tau_{\mathrm f}\tau_{\mathrm s}
\left(a_{\mathrm f}+a_{\mathrm s}\right)
+a_{\mathrm g}\left(\tau_{\mathrm f}+\tau_{\mathrm s}\right)\right]
\nabla^2 f(w_t)\frac{d}{dt}w_t
+a_{\mathrm g}\tau_{\mathrm f}\tau_{\mathrm s}
\left(\nabla^3 f(w_t)\left[\frac{d}{dt}w_t,\frac{d}{dt}w_t\right]
+\nabla^2 f(w_t)\frac{d^2}{dt^2}w_t\right),
\end{aligned}
\]
where we also use
\[
\frac{d^2}{dt^2}\nabla f(w_t)
=
\nabla^3 f(w_t)\left[\frac{d}{dt}w_t,\frac{d}{dt}w_t\right]
+\nabla^2 f(w_t)\frac{d^2}{dt^2}w_t.
\]

\paragraph{GPA} GPA maintains a base sequence $z_k$, an exponentially averaged sequence $x_k$, and an interpolation point $y_k$ at which the stochastic gradient is evaluated \cite{defazio2026smoothingdilocoprimalaveraging}. The generic GPA framework permits an arbitrary base optimizer. We instantiate it with preconditioned Heavy Ball:
\begin{equation}
\label{eq:gpa-hb-full}
\begin{aligned}
    y_k
    &= \mu_y x_k+(1-\mu_y)z_k,\\
    b_k
    &= \rho_{\mathrm H}b_{k-1}+(1-\rho_{\mathrm H})g_k,\\
    d_k
    &= P_k b_k,\\
    z_{k+1}
    &= (1-\eta_k\lambda)z_k-\eta_k d_k,\\
    x_{k+1}
    &= \mu_xx_k+(1-\mu_x)z_{k+1}.
\end{aligned}
\end{equation}
Here $b_k$ is the Heavy-Ball state, $\mu_x$ controls primal averaging, and $\mu_y$ controls the interpolation between the averaged and base iterates.

The continuous time limit of \eqref{eq:gpa-hb-full} without preconditioner and weight decay can be written as
\begin{equation}
\label{eq:gpa-hb-state-ode}
\begin{aligned}
    \dot b &=-\tau_{\mathrm f}^{-1}b+\tau_{\mathrm f}^{-1}\nabla f(y),\\
    \dot z &=-b,\\
    \dot x &=-\tau_{\mathrm s}^{-1}x+\tau_{\mathrm s}^{-1}z,\\
    y &=\mu_yx+(1-\mu_y)z.
\end{aligned}
\end{equation}
The last two equations imply
\begin{equation}
\label{eq:gpa-filter-relation}
    (1+\tau_{\mathrm s}\frac{d}{dt})y
    =\bigl(1+(1-\mu_y)\tau_{\mathrm s}\frac{d}{dt}\bigr)z.
\end{equation}
The first two equations give
\begin{equation}
    (1+\tau_{\mathrm f}\frac{d}{dt})\frac{d}{dt}z+\nabla f(y)=0.
\end{equation}
Eliminating $z$ yields the third-order GPA (Heavy-Ball) ODE
\begin{equation}
\label{eq:gpa-hb-third-order-operator}
\begin{aligned}
 (1+\tau_{\mathrm f}\frac{d}{dt})(1+\tau_{\mathrm s}\frac{d}{dt})\frac{d}{dt}y
 +\bigl(1+(1-\mu_y)\tau_{\mathrm s}\frac{d}{dt}\bigr)\nabla f(y)=0.
\end{aligned}
\end{equation}
Namely
\begin{equation}
\label{eq:gpa-hb-third-order-expanded}
\begin{aligned}
    \tau_{\mathrm f}\tau_{\mathrm s}\frac{d^3}{dt^3}y_t
    +\bigl(\tau_{\mathrm f}+\tau_{\mathrm s}
      \bigr)\frac{d^2}{dt^2}y_t
    +\frac{d}{dt}y_t
      +\nabla f(y_t)+(1-\mu_y)\tau_{\mathrm s}\nabla^2 f(y_t) \frac{d}{dt}y_t=0.
\end{aligned}
\end{equation}

\paragraph{AdEMAMix} 
AdEMAMix augments AdamW with a second, substantially slower first-moment EMA \cite{pagliardini2025the}. Let $m_k^{\mathrm f}$ and $m_k^{\mathrm s}$ denote the fast and slow first moments, and let $v_k$ denote the second moment. Its complete update is
\begin{equation}
\label{eq:ademamix-full}
\begin{aligned}
    m_k^{\mathrm f}
    &= \beta_1 m_{k-1}^{\mathrm f}+(1-\beta_1)g_k,
    &\widehat m_k^{\mathrm f}
    &= \frac{m_k^{\mathrm f}}{1-\beta_1^k},\\
    m_k^{\mathrm s}
    &= \beta_{3,k}m_{k-1}^{\mathrm s}+(1-\beta_{3,k})g_k,\\
    v_k
    &= \beta_2v_{k-1}+(1-\beta_2)g_k^{\odot 2},
    &\widehat v_k
    &= \frac{v_k}{1-\beta_2^k},\\
    w_{k+1}
    &= (1-\eta_k\lambda)w_k
       -\eta_k\frac{\widehat m_k^{\mathrm f}+a_km_k^{\mathrm s}}
       {\sqrt{\widehat v_k}+\epsilon}.
\end{aligned}
\end{equation}

After removing the second-moment preconditioner, bias-correction and weight decay, the continuous state equations are
\begin{equation}
\label{eq:ademamix-state-ode}
\begin{aligned}
    \dot m_{\mathrm f}
    &=-\tau_{\mathrm f}^{-1} m_{\mathrm f}+\tau_{\mathrm f}^{-1}\nabla f(w),\\
    \dot m_{\mathrm s}
    &=-\tau_{\mathrm s}^{-1} m_{\mathrm s}+\tau_{\mathrm s}^{-1}\nabla f(w),\\
    \dot w
    &=-a_{\mathrm f}m_{\mathrm f}-a_{\mathrm s}m_{\mathrm s}.
\end{aligned}
\end{equation}
Applying $(1+\tau_{\mathrm f}\frac{d}{dt})(1+\tau_{\mathrm s}\frac{d}{dt})$ to the parameter equation and eliminating the two momentum states gives
\begin{equation}
\label{eq:ademamix-third-order-operator}
\begin{aligned}
 (1+\tau_{\mathrm f}\frac{d}{dt})(1+\tau_{\mathrm s}\frac{d}{dt})\frac{d}{dt}w
 +a_{\mathrm f}(1+\tau_{\mathrm s}\frac{d}{dt})\nabla f(w)
 +a_{\mathrm s}(1+\tau_{\mathrm f}\frac{d}{dt})\nabla f(w)=0.
\end{aligned}
\end{equation}
Equivalently,
\begin{equation}
\label{eq:ademamix-third-order-expanded}
\begin{aligned}
    \tau_{\mathrm f}\tau_{\mathrm s}\frac{d^3}{dt^3}w_t
    +\bigl(\tau_{\mathrm f}+\tau_{\mathrm s}
      \bigr)\frac{d^2}{dt^2} w_t
    +\frac{d}{dt} w_t
    +(a_{\mathrm f}+a_{\mathrm s})\nabla f(w_t)
      +(a_{\mathrm f}\tau_{\mathrm s}
        +a_{\mathrm s}\tau_{\mathrm f})\nabla ^2f(w_t)\frac{d}{dt} (w_t)=0.
\end{aligned}
\end{equation}

\paragraph{EMA-Nesterov} EMA-Nesterov is a wrapper that replaces the one-step Nesterov look-ahead direction by an EMA of parameter increments \cite{yau2026emanesterovstabilizingnesterovslookahead}. Let $z_k$ denote the EMA look-ahead state. When the wrapped base optimizer is preconditioned Heavy Ball momentum, the update rule is
\begin{equation}
\label{eq:ema-nesterov-hb-full}
\begin{aligned}
    y_k
    &= x_k+\beta_kz_k,\\
    b_k
    &= \rho_{\mathrm H}b_{k-1}+(1-\rho_{\mathrm H})g_k,\\
    x_{k+1}
    &= (1-\eta_k\lambda)y_k-\eta_kP_kb_k,\\
    z_{k+1}
    &= \rho_{z}z_k+(1-\rho_{z})(x_{k+1}-x_k).
\end{aligned}
\end{equation}
If $\mathcal A_k$ denotes the base-optimizer map, the original wrapper is
$x_{k+1}=\mathcal A_k(x_k+\beta_kz_k)$ followed by the final line of \eqref{eq:ema-nesterov-hb-full}. 

The continuous time limit of \eqref{eq:ema-nesterov-hb-full} without preconditioning and weight decay is
\begin{equation}
\label{eq:ema-nesterov-hb-state-ode}
\begin{aligned}
    \dot b  &=-\tau_{\mathrm f}^{-1}b+\tau_{\mathrm f}^{-1}\nabla f(x),\\
    \dot z  &=-\tau_{\mathrm s}^{-1}z+\tau_{\mathrm s}^{-1}\dot x,\\
    \dot x &=\beta z-b.
\end{aligned}
\end{equation}
Eliminating $b$ and $z$ gives
\begin{equation}
\label{eq:ema-nesterov-hb-third-order-operator}
\begin{aligned}
 (1+\tau_{\mathrm f}\frac{d}{dt})
 \left[(1+\tau_{\mathrm s}\frac{d}{dt})\frac{d}{dt}-\beta \frac{d}{dt}\right]x  +(1+\tau_{\mathrm s}\frac{d}{dt})\nabla f(x)=0.
\end{aligned}
\end{equation}
Equivalently,
\begin{equation}
\label{eq:ema-nesterov-hb-third-order-expanded}
\begin{aligned}
    \tau_{\mathrm f}\tau_{\mathrm s}\frac{d^3}{dt^3} x_t
    +\bigl(\tau_{\mathrm f}+\tau_{\mathrm s}
      -\beta\tau_{\mathrm f}\bigr)\frac{d^2}{dt^2} x_t
    +\bigl(1-\beta\bigr)\frac{d}{dt}x_t
      +\nabla f(x_t)+\tau_{\mathrm s}\nabla^2 f(x_t) \frac{d}{dt}x_t=0.
\end{aligned}
\end{equation}

\paragraph{SODA} 
SODA combines gradient averaging, dual optimism, primal averaging, and primal extrapolation \cite{pethick2026optimisticdualaveragingunifies}. Its full update is
\begin{equation}
\label{eq:soda-full}
\begin{aligned}
    m_{k+1}
    &= (1-\alpha_k)m_k+\alpha_k\nabla f(y_k),\\
    \bar m_{k+1}
    &= (1-\bar\alpha_k)m_{k+1}
       +\bar\alpha_k\nabla f(y_k),\\
    z_{k+1}
    &\in \partial h_k^*(-\gamma_k\bar m_{k+1})
      =\arg\min_z\left\{\gamma_k\langle\bar m_{k+1},z\rangle+h_k(z)\right\},\\
    x_{k+1}
    &= (1-\lambda_k)x_k+\lambda_k z_{k+1},\\
    y_{k+1}
    &= (1-\bar\lambda_k)x_{k+1}+\bar\lambda_k z_{k+1}.
\end{aligned}
\end{equation}

We consider the Euclidean mirror map $h(z)=\frac{1}{2}\|z\|^2$, and 
 freeze these coefficients with new notations. The continuous limit  is
\begin{equation}
\label{eq:soda-state-ode}
\begin{aligned}
    \dot m &=-\tau^{-1} m+\tau^{-1}\nabla f(y),\\
    \dot x&=-\lambda_1 x-\eta(m+\omega \nabla f(y)),\\
    y &=(1-\lambda_2)x-\nu (m+\omega \nabla f(y)).
\end{aligned}
\end{equation}

Define the quantities $    s= m+\omega g(y),
    \delta= \eta a+\nu\lambda_1$. The state equations then become
\begin{equation}
\label{eq:soda-s-state}
    \dot x=-\lambda_1x-\eta s,
    \qquad
    y=(1-\lambda_2)x-\nu s.
\end{equation}
Moreover,
\begin{align}
    \dot s
    =\dot m+\omega \nabla^2f(y)\dot y
    =-\tau^{-1}s
      +\frac{1+\omega}{\tau}\nabla f(y)
      +\omega \nabla^2f(y)\dot y.
\label{eq:soda-s-dynamics}
\end{align}

Differentiating $y=ax-\nu s$ and using
\eqref{eq:soda-s-state}--\eqref{eq:soda-s-dynamics} yield
\begin{align}
    \bigl(I+\nu\omega \nabla^2f(y)\bigr)\dot y
    &=
    -\lambda_1y
    +
    q s
    -
    \frac{\nu(1+\omega)}{\tau}g(y),
\label{eq:soda-y-first-order}
\end{align}
where $  q=
    \nu\bigl(\tau^{-1}-\lambda_1\bigr)
    -\eta(1-\lambda_2)$.

Differentiating \eqref{eq:soda-y-first-order} and eliminating
$s$ and $\dot s$ give a closed equation for $y$:
\begin{equation}
\label{eq:soda-y-second-order}
\begin{aligned}
    \bigl(I+\nu\omega \nabla^2f(y)\bigr)\ddot y
    +
    \nu\omega
    \nabla^3f(y)[\dot y,\dot y]
    +
    \Biggl[
        \bigl(\lambda_1+\frac{1}{\tau}\bigr)I
        +
        \ell \nabla^2f(y)
    \Biggr]\dot y
    +
    \frac{\lambda_1}{\tau}y
    +
    \frac{(1+\omega)\delta}{\tau}
    \nabla f(y)
    =0,
\end{aligned}
\end{equation}
where $\ell=
            \frac{\nu(1+\omega)}{\tau}
            +
            \omega\delta
        $, and the $\frac{\lambda_1}{\tau}y$ term denotes the weight decay. Optimistic momentum and primal extrapolation together yield an implicit preconditioner  $\bigl(I+\nu\omega \nabla^2f(y)\bigr)^{-1}$.

\subsection{Unified ODE Representation}
\label{subsec:common-operator-family}

The preceding derivation indicates that these  continuous-time limits can be cast in a unified form
\begin{equation}
\label{eq:general-cubic-family}
\mathcal A\left(\frac{d}{dt}\right)w_t+\mathcal B\left(\frac{d}{dt}\right)\nabla f(w_t)=0,
\end{equation}
where $\mathcal{A}, \mathcal{B}$ are matrix-valued polynomials in the differential operator $\frac{d}{dt}$:
\begin{equation}
\begin{aligned}
\mathcal A\left(\frac{d}{dt}\right)
=\sum_{i=0}^n A_i\left(\frac{d}{dt}\right)^i,\qquad
\mathcal B\left(\frac{d}{dt}\right)
=\sum_{i=0}^m B_i\left(\frac{d}{dt}\right)^i.
\end{aligned}
\end{equation}

For two-timescale momentum, $\mathcal{A}$ is of degree three. This polynomial framework naturally generalizes to the $n$-timescale case, where the degree of $\mathcal{A}$ becomes $n+1$.

\section{Illustrative Theoretical Analysis}

\subsection{Variance Reduction of Multiscale Momentum}\label{app:var-mts}
In this section, we analyze the variance of multiscale momentum through an
illustrative Gaussian example, demonstrating the variance-reduction effect of
introducing a slow momentum component. Consider
\begin{equation}
    \begin{aligned}
        m^{\mathrm{slow}}_k
        &=(1-\alpha_{\mathrm{slow}})\,m^{\mathrm{slow}}_{k-1}+\alpha_{\mathrm{slow}}\,g_k,\\
        m^{\mathrm{fast}}_k
        &=(1-\alpha_{\mathrm{fast}})\,m^{\mathrm{fast}}_{k-1}+\alpha_{\mathrm{fast}}\,g_k,\\
        m_k&=c\, m^{\mathrm{slow}}_k+(1-c)\,m^{\mathrm{fast}}_k,
    \end{aligned}
\end{equation}
where $0<\alpha_{\mathrm{slow}}<\alpha_{\mathrm{fast}}\le 1$ and
$g_k\overset{\mathrm{i.i.d.}}{\sim}\mathcal{N}(\mu,\Sigma^2)$.
We are interested in the stationary covariance of the update $m_k$
which preserves the stationary mean, i.e.,
$\mathbb{E}[m_\infty]=\mu$.

Let $\epsilon_k=g_k-\mu$ denote the zero-mean noise. Unrolling the recursions,
the noise components of the two momenta are
\begin{equation}
    m^{\mathrm{slow}}_k-\mu
    =\sum_{j=0}^{k}(1-\alpha_{\mathrm{slow}})^{j}\,\alpha_{\mathrm{slow}}\,\epsilon_{k-j},
    \qquad
    m^{\mathrm{fast}}_k-\mu
    =\sum_{j=0}^{k}(1-\alpha_{\mathrm{fast}})^{j}\,\alpha_{\mathrm{fast}}\,\epsilon_{k-j}.
\end{equation}
Therefore, in stationarity,
\begin{equation}
\begin{aligned}
    &\operatorname{Cov}(m^{\mathrm{slow}}_\infty)
    =\sum_{j=0}^{\infty}(1-\alpha_{\mathrm{slow}})^{2j}\,\alpha_{\mathrm{slow}}^2\,\Sigma^2,
    \qquad
    \operatorname{Cov}(m^{\mathrm{fast}}_\infty)
    =\sum_{j=0}^{\infty}(1-\alpha_{\mathrm{fast}})^{2j}\,\alpha_{\mathrm{fast}}^2\,\Sigma^2,
    \\&
    \operatorname{Cov}(m^{\mathrm{slow}}_\infty,m^{\mathrm{fast}}_\infty)
    =\sum_{j=0}^{\infty}(1-\alpha_{\mathrm{slow}})^{j}(1-\alpha_{\mathrm{fast}})^{j}\,\alpha_{\mathrm{slow}}\alpha_{\mathrm{fast}}\,\Sigma^2.
\end{aligned}
\end{equation}
Summing the geometric series, it follows that
\begin{equation}
    \operatorname{Cov}(m_\infty)
    =\left(
        \frac{c^2\,\alpha_{\mathrm{slow}}^2}{1-(1-\alpha_{\mathrm{slow}})^2}
        +\frac{(1-c)^2\,\alpha_{\mathrm{fast}}^2}{1-(1-\alpha_{\mathrm{fast}})^2}
        +\frac{2c(1-c)\,\alpha_{\mathrm{slow}}\alpha_{\mathrm{fast}}}{1-(1-\alpha_{\mathrm{slow}})(1-\alpha_{\mathrm{fast}})}
    \right)\Sigma^2.
\end{equation}
Simplifying yields
\begin{equation}
    \operatorname{Cov}(m_\infty)
    =\underbrace{\left(
        \frac{c^2\,\alpha_{\mathrm{slow}}}{2-\alpha_{\mathrm{slow}}}
        +\frac{(1-c)^2\,\alpha_{\mathrm{fast}}}{2-\alpha_{\mathrm{fast}}}
        +\frac{2c(1-c)\,\alpha_{\mathrm{slow}}\alpha_{\mathrm{fast}}}{\alpha_{\mathrm{slow}}+\alpha_{\mathrm{fast}}-\alpha_{\mathrm{slow}}\alpha_{\mathrm{fast}}}
    \right)}_{:=\,f(c)}\,\Sigma^2.
\end{equation}
Differentiating with respect to $c$ gives
\begin{equation}
    f'(c)
    =-\frac{2(\alpha_{\mathrm{fast}}-\alpha_{\mathrm{slow}})
    \Bigl(\alpha_{\mathrm{fast}}(2-\alpha_{\mathrm{slow}})
    -2(\alpha_{\mathrm{fast}}-\alpha_{\mathrm{slow}})\,c\Bigr)}
    {(\alpha_{\mathrm{slow}}+\alpha_{\mathrm{fast}}
    -\alpha_{\mathrm{slow}}\alpha_{\mathrm{fast}})
    (2-\alpha_{\mathrm{slow}})(2-\alpha_{\mathrm{fast}})}<0,
\end{equation}
since every factor is positive for
$0<\alpha_{\mathrm{slow}}<\alpha_{\mathrm{fast}}\le 1$ and $c\in[0,1]$.
Hence $f$ is strictly decreasing on $[0,1]$. \textbf{Increasing the weight $c$ of the slow momentum monotonically reduces the stationary variance.} In particular, the two limiting regimes are
\begin{equation}
    f(0)=\frac{\alpha_{\mathrm{fast}}}{2-\alpha_{\mathrm{fast}}}
    \quad\text{(fast momentum only)},
    \qquad
    f(1)=\frac{\alpha_{\mathrm{slow}}}{2-\alpha_{\mathrm{slow}}}
    \quad\text{(slow momentum only)}.
\end{equation}
Thus incorporating slow momentum reduces the stationary variance by a factor of
up to
$\dfrac{\alpha_{\mathrm{slow}}(2-\alpha_{\mathrm{fast}})}
{\alpha_{\mathrm{fast}}(2-\alpha_{\mathrm{slow}})}
\approx\dfrac{\alpha_{\mathrm{slow}}}{\alpha_{\mathrm{fast}}}$ relative to fast
momentum alone.

\subsection{Weight Norm Inflation by Slow Momentum}\label{app:norm-inflation}
In this section, we provide a theoretical understanding of the role of slow momentum in norm inflation. Motivated by the fact that noise dominates in LLM pretraining, the intuition behind norm inflation for a scale‑invariant loss $L$ (i.e., $L(cW)=L(W)$ for any 
$c>0$) is that a Brownian‑motion‑like movement drives the norm to blow up, while weight decay can suppress it. We therefore approximately assume that the stochastic gradients $g_k\overset{\mathrm{i.i.d.}}{\sim}
\mathcal N(0,\sigma^2I_d)$. For brevity, we assume that the projection onto the flat directions $\mathcal Q$ is constant across iterations and has rank $q\approx d$. We then consider the update with weight decay coefficient$\lambda>0$
\begin{equation}\label{alg:flat-mts-base-norm-inflation}
    \begin{aligned}
        m_k^{\#}&=(1-\alpha_{\#})m_{k-1}^{\#}+g_k, \quad {\#}\in\{\text{fast},\text{slow}\}, \\
        w_{k+1}^{\chi}&=(1-\eta\lambda)w_k^{\chi}-\eta (\chi\mathcal{Q} m^{\text{slow}}_k + m^{\text{fast}}_k),
    \end{aligned}
\end{equation}
where $m_{-1}^{\text{fast}}=m_{-1}^{\text{slow}}=0$ and the superscript $\chi$ denotes the slow momentum coefficient. We have the following conclusion.

\begin{proposition}[Slow-momentum-induced norm inflation]
\label{prop:slow-momentum-norm-inflation}
Suppose $\rho:=1-\eta\lambda>0$ and both processes start from the
same deterministic $w_0$.  Then for every $k\geq 0$, the iterates in \eqref{alg:flat-mts-base-norm-inflation} satisfy
\begin{equation}\label{norm-diff}
    \mathbb E\|w_{k+1}^{\chi}\|^2
-
\mathbb E\|w_{k+1}^{0}\|^2
=
q\eta^2\sigma^2
\sum_{n=0}^{k}
\left(
2\chi A_{\mathrm{fast},n}A_{\mathrm{slow},n}
+
\chi^2A_{\mathrm{slow},n}^2
\right),
\end{equation}
where $
\beta_{\mathrm{fast}}=1-\alpha_{\mathrm{fast}}, 
\beta_{\mathrm{slow}}=1-\alpha_{\mathrm{slow}}$ and 
\[
A_{\#,n}
:=
\sum_{\ell=0}^{n}
\rho^{\,n-\ell}\beta_{\#}^\ell,
\qquad
{\#}\in\{\mathrm{fast},\mathrm{slow}\}.
\]

Moreover, the weight norm converges to
\begin{equation}\label{equ-norm}
    \lim_{k\to\infty}\mathbb E\|w_k^{\chi}\|^2
=
V_{\chi}
:=
d\eta^2\sigma^2
\sum_{n=0}^{\infty}A_{\mathrm{fast},n}^2+\underbrace{q\eta^2\sigma^2
\sum_{n=0}^{\infty}
\left(
2\chi
A_{\mathrm{fast},n}A_{\mathrm{slow},n}
+
\chi^2A_{\mathrm{slow},n}^2
\right)}_{\text{inflated norm by slow momentum}},
\end{equation}
where $V_{\chi}<\infty$ if and only if $\rho<1$ (i.e., $\lambda>0$).
\end{proposition}

\begin{remark}
The norm inflation due to slow momentum has two contributing factors. The first is the explicit positive coefficient $\chi>0$, and the second is the asymptotic dominance $A_{\mathrm{slow},n}\gg A_{\mathrm{fast},n}$ for large $n$.
\end{remark}

\begin{proof}
Unrolling the momentum recursions gives
\[
m_k^{\#}
=
\sum_{j=0}^{k}\beta_{\#}^{\,k-j}g_j,
\qquad
{\#}\in\{\mathrm{fast},\mathrm{slow}\}.
\]
Substituting it into the update rule yields
\[
w_{k+1}^{\chi}
=
\rho^{k+1}w_0
-\eta\sum_{j=0}^{k}
\left(
A_{\mathrm{fast},k-j}I_d
+
\chi A_{\mathrm{slow},k-j}\mathcal Q
\right)g_j.
\]
Since $g_k\overset{\mathrm{i.i.d.}}{\sim}
\mathcal N(0,\sigma^2I_d)$, 
\[
\mathbb E\|w_{k+1}^{\chi}\|^2
=
\rho^{2(k+1)}\|w_0\|^2
+
\eta^2\sigma^2
\sum_{j=0}^{k}
\left\|
A_{\mathrm{fast},k-j}I_d
+
\chi A_{\mathrm{slow},k-j}\mathcal Q
\right\|_F^2.
\]
We obtain \eqref{norm-diff}  by 
\[
\left\|A_{\mathrm{fast},n}I_d+\chi A_{\mathrm{slow},n}\mathcal Q\right\|_F^2
-
\left\|A_{\mathrm{fast},n}I_d\right\|_F^2
=
q\left(
2\chi A_{\mathrm{fast},n}A_{\mathrm{slow},n}
+
\chi^2A_{s,n}^2
\right).
\]

Moreover, 
\[
A_{\#,n}
=
\begin{cases}
\dfrac{\rho^{n+1}-\beta_{\#}^{n+1}}
{\rho-\beta_{\#}},
& \rho\neq\beta_{\#},\\[2mm]
(n+1)\rho^n,
& \rho=\beta_{\#},
\end{cases}
\]
gives 
$\sum_{n=0}^{\infty}A_{\#,n}^2<\infty$ iff $\lambda>0$. Therefore, the stationary
mean-square parameter norms are finite iff $\lambda>0$, and taking the trace of
the stationary covariance matrix gives $V_\chi$ in \eqref{equ-norm}.
\end{proof}

\section{Proof}

\paragraph{Notations.} Define $\rf=1-\af, \rs=1-\as$, $\eta_{\max,k}=\max_{i<k}\eta_i$. For sharp/flat subspace indicator \(q\in\{0,1\}\), write $
 d_q=\frac{\rf}{\af^2}
       +\chi q\frac{\rs}{\as^2}$, \(a_0=\atpg\) and \(a_1=\aperp\). We take 
$e=2.71828...$ as the base of the natural logarithm, unless explicitly defined otherwise.

We first  introduce explicit constant coefficients in Assumption~\ref{ass:power-law-short} to facilitate cleaner mathematical derivations.

\begin{assumption}[Detailed Version of Assumption \ref{ass:power-law-short}]\label{ass:power-law}
There exist constants \(\ell_-,\ell_+,
\vartheta_-,\vartheta_+>0\), a capacity exponent \(\nu>1\)  and  a source exponent \(s>0\) such that, for every
\(1\leq j\leq d\),
\[
 \ell_-j^{-\nu}\leq\lambda_j
 \leq \ell_+ j^{-\nu},
 \qquad
 \vartheta_- j^{-1}\lambda_j^{s-1}
 \leq |w_j^\star|^2
 \leq \vartheta_+ j^{-1}\lambda_j^{s-1}.
\]

\end{assumption}

\begin{lemma}[Hessian-aligned anisotropic  noise structure]\label{lem:gaussian-noise}
For every \(u\in\R^d\), the gradient noise covariance 
is
\begin{equation}\label{eq:exact-noise-covariance}
 \Sigma(u)=Huu^\top H+(u^\top Hu+\sigma^2)H,
\end{equation}
and it satisfies 
\begin{equation}\label{eq:noise-sandwich}
 (2\cE(u)+\sigma^2)H
 \preceq\Sigma(u)
 \preceq(4\cE(u)+\sigma^2)H.
\end{equation}

\end{lemma}

\begin{proof}
This lemma is Lemma A.1 in \cite{wang2026fast}, and we restate its proof here for completeness. For one sample, the stochastic gradient satisfies
\[
 g=\phi\phi^\top u-\phi\epsilon,
 \qquad \E g=Hu.
\]
 Wick's
formula for the centered Gaussian vector \(\phi\) gives
\[
 \E[\phi\phi^\top uu^\top\phi\phi^\top]
 =2Huu^\top H+(u^\top Hu)H.
\]
Further computation based on it proves \eqref{eq:exact-noise-covariance}.  Finally, we have 
 $Huu^\top H\preceq(u^\top Hu)H$ 
by Cauchy--Schwarz to the \(H\)-inner product.  Since
$u^\top Hu=2\cE(u)$, this proves \eqref{eq:noise-sandwich}.
\end{proof}

\begin{remark}
Lemma~\ref{lem:gaussian-noise} shows the alignment between the noise and the Hessian. In the early phase of training, $\cE$ is large and the noise is dominated by the loss scale. In the later phase, where $\cE \lesssim \sigma^2$, the noise aligns well with the Hessian, and the stochastic gradient is noise-dominated.
\end{remark}

The following lemma gives the aggregate parameter‑state update rule.
\begin{lemma}[Discrete parameter-momentum recursion]
\label{lem:exact-three-state-recursion}
For the $j$-th eigen-direction \(1\leq j\leq d\), define
\[
 X_{j,k}:=(u_{j,k},m_{j,k-1}^{\mathrm{fast}},m_{j,k-1}^{\mathrm{slow}})^\top.
\]
Then the concrete algorithm \eqref{eq:algorithm} is equivalent with 
\begin{equation}\label{eq:exact-modal-state}
 X_{j,k+1}=M_{j,k}X_{j,k}+d_{j,k}\xi_{j,k},
\end{equation}
where
\begin{equation}\label{eq:modal-matrix}
 M_{j,k}=\begin{pmatrix}
 1-\eta_kc_j\lambda_j&-\eta_k\rf&-\eta_k\chi q_j\rs\\
 \lambda_j&\rf&0\\
 \lambda_j&0&\rs
 \end{pmatrix},
 \qquad
 d_{j,k}=\begin{pmatrix}-\eta_kc_j\\1\\1\end{pmatrix},
\end{equation}
$c_j:=\bd+1+\chi q_j$, and \(\xi_{j,k}\in\mathbb{R}\) denotes the shared noise   in all three entries.
\end{lemma}

\begin{proof}
Substituting 
\[
 m_k^{\mathrm{fast}}=\rf m_{k-1}^{\mathrm{fast}}+\lambda_ju_{j,k}+\xi_{j,k},
 \qquad
 m_k^{\mathrm{slow}}=\rs m_{k-1}^{\mathrm{slow}}+\lambda_ju_{j,k}+\xi_{j,k}
\]
into \eqref{eq:algorithm} gives the results.
\end{proof}

We next analyze the roots of the characteristic polynomial $\text{det}(zI-M_{j,k})$ of the transition matrix $M_{j,k}$, which determine the final convergence rate. Consider any fixed coordinate $j$, and hide the subscript $j$ for brevity.  Define the constants 
\[
 a:=\bd+\af^{-1}+\widetilde\chi\as^{-1},
 \qquad
 d:=\frac{1-\af}{\af^2}
       +\widetilde\chi\frac{1-\as}{\as^2},
\]
where $\widetilde\chi\ge0$ denotes the corresponding coefficient for the slow momentum. The 
matrix with zero learning rate has eigenvalues
\[
 1,\qquad \rf=1-\af,\qquad \rs=1-\as.
\]
The next lemma shows
that for a sufficiently small positive learning rate, the root issuing
from \(1\) remains separated from the two roots issuing from \(\rf\) and
\(\rs\).

\begin{lemma}[Spectrum distribution of the transition matrix]
\label{lem:modal-polynomial}
Let $M_{\widetilde\chi,\lambda,\eta}$  be the matrix obtained from
\eqref{eq:modal-matrix}   with $\lambda_j=\lambda$, $\chi q_j=\widetilde\chi$ and $\eta_k=\eta$. Define $x=\eta\lambda$.  If
\begin{equation}\label{eq:root-smallness}
 a\eta\lambda\leq\frac{\as}{8},
\end{equation}
then we have:
\begin{itemize}
    \item 
$M_{\widetilde\chi,\lambda,\eta}$ has a real simple eigenvalue
\(z_0\) in
 $\left[1-\frac{\as}{4},1\right]$ with 
\begin{equation}\label{eq:slow-root-log-bound}
 \left|\log z_0+ax
 +a\left(d+\frac{a}{2}\right)x^2\right|
 \leq c_{\log}x^3,
\end{equation}
where 
\(c_{\log}=4a(d+a)^2+6a^2e\).
\item The other two eigenvalues, counted with algebraic multiplicity, satisfy
\begin{equation}\label{eq:fast-block-localization}
\max_{z\in\sigma(M_{\widetilde\chi,\lambda,\eta})\setminus\{z_0\}}|z|
 \leq1-\frac{3\as}{4}.
\end{equation}

\end{itemize}

\end{lemma}

\begin{proof}
 Expanding \(\det(zI-M_{\widetilde\chi,\lambda,\eta})\) along its first row
gives
\begin{equation}\label{eq:factored-characteristic}
\begin{aligned}
 P(z)
=(z-\rf)(z-\rs)(z-1+\eta c\lambda)
 +\eta\lambda\rf(z-\rs)
  +\eta\widetilde\chi\lambda\rs(z-\rf),
\end{aligned}
\end{equation}
where $c=\bd+1+\widetilde\chi$. 
Define the constants $e=2\frac{1-\af}{\af^3}
      +2\widetilde\chi\frac{1-\as}{\as^3}$,
and $c_{\log}=4a(d+a)^2+6a^2e$. 

\begin{itemize}
    \item \underline{Bound the largest real root.} Note that $F(z)>0$ for $z\ge1$. Thus all real roots are less than $1$.  Let $\delta:=1-z_1$. The characteristic equation
gives 
\begin{equation}\label{eq:slow-root-fixed-point}
 \delta=\underbrace{x\left[
 c+\frac{\rf}{\af-\delta}
       +\widetilde\chi\frac{\rs}{\as-\delta}\right]}_{:=F_x(\delta)}.
\end{equation}
On the interval $I=[0,\as/2]$, we have $0\le F_x(\delta)\le 2ax\le\as/4$. Thus $F_x(I)\subset I$. On the other hand, for $\delta\in I$, 
\[F_x'(\delta)
=x\left[
\frac{\rf}{(\af-\delta)^2}
+\widetilde\chi
 \frac{\rs}{(\as-\delta)^2}
\right]
\le4xd
\le\frac12.\]
Since \(F_x\) is a contraction mapping on \(I\), it admits a unique real fixed point \(\delta_x \in I\), with
\begin{equation}\label{delta_x-1st_est}
  0 \le \delta_x=F_x(\delta_x) \le2ax\le \frac{\as}{4},\qquad \lim_{x\to0}  \delta_x=0.  
\end{equation}

Combining it with $P(-\infty)<0,P(\rs)>0$, $P$ has at least 2 real roots. Thus $P$ has 3 real roots.

Next we consider finer estimate to the largest root. Expanding the two $(1-\delta)^{-1}$ terms in \eqref{eq:slow-root-fixed-point} via their geometric series (with remainder) in powers of $\delta$ gives 
\[
 a+d\delta+\frac{e}{2}\delta^2
 \leq
 c+\frac{\rf}{\af-\delta}
       +\widetilde\chi\frac{\rs}{\as-\delta}
 \leq
 a+d\delta+\frac{2e}{3}\delta^2
\]
by using \(\delta\leq2ax\leq\as/4\).  Hence
\begin{equation}\label{delta_x-1st_est-1}
    0\leq\delta-ax
 \leq2adx^2+\frac83a^2ex^3. 
\end{equation}

Further applying
\(dx\leq ax/\as\leq1/8\)  gives
\[
 0\leq\delta-ax-adx^2
 \leq\left(2ad^2+3a^2e\right)x^3.
\]
Combining them with \(\sum_{n\geq3}\delta^n/n\leq\delta^3/[3(1-\delta)]\) gives 
\begin{equation}
    \begin{aligned}
       |\log(1-\delta)+ax+a(d+a/2)x^2|\le &|\delta-ax-adx^2|+\frac{|\delta^2-a^2x^2|}{2}+\sum_{n\geq3}\frac{\delta^n}{n}
       \\\le&\left(2ad^2+3a^2e\right)x^3+ \left(3a^2d+\frac12a^2e\right)x^3+ 4a^3x^3.
    \end{aligned}
\end{equation}
Their sum is at most \(c_{\log}x^3\), proving
\eqref{eq:slow-root-log-bound}.

\item \underline{Bound the remaining two roots.} Define 
\[
T=
\begin{pmatrix}
1&0&0\\
-1&\alpha_{\text{fast}}/\lambda&0\\
-1&0&\alpha_{\text{slow}}/\lambda
\end{pmatrix}.
\]
Then $M$ is similar to
\begin{equation}\label{eq:balanced-modal-matrix}
 N_{q}(x):=TMT^{-1}=
 \begin{pmatrix}
 1-ax&-xb_{\text{fast}}&-xb_{\text{slow}}\\
 ax&\rf+xb_{\text{fast}}&xb_{\text{slow}}\\
 ax&xb_{\text{fast}}&\rs+xb_{\text{slow}}
 \end{pmatrix},
 \quad
 b_{\text{fast}}:=\frac{\rf}{\af},
 \quad b_{\text{slow}}:=\widetilde\chi\frac{\rs}{\as}.
\end{equation}
Thus
\(\|N_{q}(x)-\diag(1,\rf,\rs)\|_\infty
\leq2ax\leq\as/4\).
The Gershgorin disk associated with the first row lies in
\(\operatorname{Re}z\geq1-2ax\geq1-\as/4\).  Each of the other two
disks lies in \(|z|\leq\rs+2ax\leq1-3\as/4\).  The first disk is
therefore disjoint from the union of the other two.  By the
Gershgorin disk theorem (Theorem \ref{prop:gershgorin}), the first component
contains exactly one eigenvalue, while the complementary component contains
exactly two eigenvalues counted with algebraic multiplicity, proving
\eqref{eq:fast-block-localization}.
\end{itemize}
\end{proof}

\begin{lemma}[Properties of spectral projections]
\label{lem:modal-projectors}
We use the notation \(N_q(x)\) in \eqref{eq:balanced-modal-matrix}, where \(q\in\{0,1\}\) is the sharp subspace indicator.   For \(x\geq0\) satisfying
\(a_qx\leq\as/8\), let
\begin{equation}\label{eq:riesz-projectors}
\begin{aligned}
 \Gamma_0&:=\{z\in\mathbb C:|z-1|=\as/2\},
 &\Gamma_1&:=\{z\in\mathbb C:|z|=1-\as/2\},\\
 P_0(x)&:=\frac{1}{2\pi\mathrm i}
 \int_{\Gamma_0}(zI-N_q(x))^{-1}\dd z,
 &P_1(x)&:=\frac{1}{2\pi\mathrm i}
 \int_{\Gamma_1}(zI-N_q(x))^{-1}\dd z.
\end{aligned}
\end{equation}
Both contours are positively oriented.  The following properties hold.
\begin{enumerate}[(i)]
\item The operators in \eqref{eq:riesz-projectors} are complementary
 oblique subspace projections with $P_0(x)+P_1(x)=I$ and $P_0(x)P_1(x)=P_1(x)P_0(x)=0$. \(P_0(x)\) is associated with the largest eigenvalue and \(P_1(x)\) is associated with the complementary
two-dimensional subspace\footnote{In general, 
 $P_0, P_1$ are not orthogonal to each other.}.  If \(x,y\geq0\) and
\(a_q\max\{x,y\}\leq\as/8\), then
\begin{equation}\label{eq:projector-primitives}
 \|P_0(x)\|_\infty\leq2,
 \qquad
 \|P_0(x)-P_0(y)\|_\infty
 \leq\frac{16a_q}{\as}|x-y|,
\end{equation}
and
\begin{equation}\label{eq:projector-interface}
 \max\{\|P_1(y)P_0(x)\|_\infty,
          \|P_0(y)P_1(x)\|_\infty\}
 \leq\frac{32a_q}{\as}|x-y|.
\end{equation}

\item At \(x=0\), one has \(P_0(0)=\diag(1,0,0)\). Besides, the  projector  $P_0$ is
differentiable at $x=0$ with
\begin{equation}\label{eq:projector-first-derivative}
 P_0'(0)=
 \begin{pmatrix}
  0&-\rf/\af^2&-\chi q\rs/\as^2\\
  a_q/\af&0&0\\
  a_q/\as&0&0
 \end{pmatrix},
\end{equation}
and its quadratic remainder satisfies 
\begin{equation}\label{eq:projector-second-order-remainder}
 \|P_0(x)-P_0(0)-xP_0'(0)\|_\infty
 \leq\frac{32a_q^2x^2}{\as^2}.
\end{equation}

\item Let \(\delta=1-z_{0,q}(x)\).  The range of \(P_0(x)\) is spanned
by the eigenvector
\begin{equation}\label{eq:normalized-slow-right-eigenvector}
 r_q(x)=\left(
  1,\frac{\delta}{\af-\delta},\frac{\delta}{\as-\delta}\right)^\top
 \,\text{with } \|r_q(x)\|_\infty=1.
\end{equation}
The complementary  projection $P_1$ satisfies for any $ v\in\operatorname{Range}P_1(x)$, 
\begin{equation}\label{eq:one-step-fast-contraction}
 \|N_q(x)v\|_\infty
 \leq\left(1-\frac{17\as}{20}\right)\|v\|_\infty.
\end{equation}

\item Let $a\leq b$,  $x_m\geq0$, and suppose
$a_qx_m\leq\as/8$ for
$a\leq m\leq b$.  Define
\[
 L_m:=\|P_0(x_m)-P_0(x_{m-1})\|_\infty.
\]
With the ordered-product notation 
$\overleftarrow{\prod_{m=a}^b}A_m=A_b\cdots A_a$, we have 
\begin{equation}\label{eq:variable-fast-product}
 \left\|
 \overleftarrow{\prod_{m=a}^{b}}
 \bigl[N_q(x_m)P_1(x_m)\bigr]v
 \right\|_\infty
 \leq e^{\sum_{m=a+1}^{b}L_m}
 \left(1-\frac{17\as}{20}\right)^{b-a+1}
 \|P_1(x_a)v\|_\infty.
\end{equation}
\end{enumerate}

\end{lemma}

\begin{proof}
\smallskip
\noindent\emph{Proof of \textup{(i)}.} Lemma~\ref{lem:modal-polynomial}
shows that $\Gamma_0$ encloses exactly the largest root
$z_{0,q}(x)$ (i.e, the largest eigenvalue of $N_q(x)$).  The disk bounded by $\Gamma_1$ contains the other
two eigenvalues, counted with algebraic multiplicity, and excludes
$z_{0,q}(x)$.  Hence $P_0$ is the (rank 1) oblique projection to the eigenspace corresponding to $z_{0,q}(x)$, and $P_1=I-P_0$. On each contour,
the Neumann resolvent identity and
$\|N_q(x)-N_q(0)\|_\infty\leq2a_qx$ give
\[
 \|(zI-N_q(x))^{-1}\|_\infty\leq\frac4{\as}.
\]
The Riesz formula and the second resolvent identity now give
\eqref{eq:projector-primitives}.  Complementarity gives
\[
 P_1(y)P_0(x)=[P_0(x)-P_0(y)]P_0(x),
 \qquad
 P_0(y)P_1(x)=P_0(y)[P_0(y)-P_0(x)],
\]
so \eqref{eq:projector-interface} follows from
\eqref{eq:projector-primitives}.

\smallskip
\noindent\emph{Proof of \textup{(ii)}.} Write $
 N_q(x)=N_q(0)+xN_q'(0)$ (which is linear). 
Differentiating the Riesz formula at \(x=0\) gives
\begin{equation}
P_0'(0)=\frac{1}{2\pi\mathrm i}
 \int_{\Gamma_0}R_0(z)N_q'(0)R_0(z)\dd z=
 \begin{pmatrix}
  0&-\rf/\af^2&-\chi q\rs/\as^2\\
  a_q/\af&0&0\\
  a_q/\as&0&0
 \end{pmatrix},
\end{equation}
where 
\(R_x(z):=(zI-N_q(x))^{-1}\) and we use 
\(R_0(z)=\diag((z-1)^{-1},(z-\rf)^{-1},(z-\rs)^{-1})\).  Applying the second-order
resolvent identity
\[
 R_x=R_0+xR_0N_q'(0)R_0
 +x^2R_0N_q'(0)R_xN_q'(0)R_0,
\]
we have 
\begin{equation}
\begin{aligned}
  P_0(x)-P_0(0)-xP_0'(0)=&  \frac{1}{2\pi i}\int_{\Gamma_0}\left(R_x(z)-R_0(z)-xR_0(z)N_q'(0)R_0(z)x\right)\,dz
  \\=&\frac{x^2}{2\pi i}\int_{\Gamma_0}\left(R_0(z)N_q'(0)R_x(z)N_q'(0)R_0(z)\right)\,dz
\end{aligned}
\end{equation}

Using the bounds $ \|R_0\|_\infty\leq\frac2{\as},
 \|R_x\|_\infty\leq\frac4{\as},
 \|N_q'(0)\|_\infty\leq2a_q$, we get

\begin{equation}
\begin{aligned}
   \|P_0(x)-P_0(0)-xP_0'(0)\|_\infty\le &\frac{x^2}{2\pi}
|\Gamma_0|
 \|R_0\|_\infty^2
\|N_q'(0)\|_\infty^2
\|R_x\|_\infty
\\\le&\frac{x^2}{2\pi}
(\pi\as) \left(\frac2{\as}\right)^2
(2a_q)^2
\frac4{\as}\leq\frac{32a_q^2x^2}{\as^2}.  
\end{aligned}
\end{equation}

\smallskip
\noindent\emph{Proof of \textup{(iii)}.} 
Direct computation verifies that
\(N_q(x)r_q(x)=z_{0,q}(x)r_q(x)\).  Since
\(0\leq\delta\leq\as/4<\af\), the last two coordinates of \(r_q(x)\)
belong to \([0,1/3]\), proving 
\eqref{eq:normalized-slow-right-eigenvector}.

For the complementary subspace, set
$
 \theta_q(x):=
 x\left(
 \frac{\rf}{\af(\af-\delta)}
 +\frac{\chi q\rs}{\as(\as-\delta)}
 \right)$. Then by computation, 
an eigenvector for \(z_{0,q}(x)\) with respect to $N_q(x)^\top$
is
\[
 \ell_q(x)=\left(
 1,
 -\frac{x\rf}{\af(\af-\delta)(1-\theta_q(x))},
 -\frac{x\chi q\rs}{\as(\as-\delta)(1-\theta_q(x))}\right)^\top.
\]
The bounds \(\delta\leq\as/4\), \(d_q\leq a_q/\as\), and
\(a_qx\leq\as/8\) imply
\[
 0\leq\theta_q(x)\leq\frac43xd_q\leq\frac16.
\]
Because \(\operatorname{Range}P_1(x)=\ker\ell_q(x)^\top\), every
vector in the fast range satisfies
\[
 |v_1|\leq\frac{\theta_q(x)}{1-\theta_q(x)}
 \max\{|v_2|,|v_3|\}
 \leq\frac15\max\{|v_2|,|v_3|\}.
\]
The last two coordinates of \(N_q(x)v\) are therefore bounded by
\[
 \left(1-\as+\frac65a_qx\right)\|v\|_\infty
 \leq\left(1-\frac{17\as}{20}\right)\|v\|_\infty.
\]
 This proves
\eqref{eq:one-step-fast-contraction}.

\smallskip
\noindent\emph{Proof of \textup{(iv)}.} If \(v\in\operatorname{Range}P_1(x_{m-1})\), then
\(\|P_1(x_m)v\|_\infty\leq(1+L_m)\|v\|_\infty\).  Iteration with
\eqref{eq:one-step-fast-contraction} proves
\eqref{eq:variable-fast-product}.
\end{proof}

Before proceeding with the proof, we recall some previous notations and introduce a few new ones.
\begin{equation}\label{eq:adiabatic-kernels}
\begin{gathered}
 a_j:=\begin{cases}\atpg,&j\leq r,\\ \aperp,&j>r,\end{cases}
 \qquad
 A:=\atpg(I_d-Q)+\aperp Q,\\
 S_A(t):=\sum_{j=1}^d\lambda_j|w_j^\star|^2e^{-2a_j\lambda_jt},
 \qquad
 K_A(t):=\sum_{j=1}^da_j^2\lambda_j^2e^{-2a_j\lambda_jt}.
\end{gathered}
\end{equation}

\begin{lemma}[Uniform kernel  bounds]\label{lem:power-law-kernels}
Suppose that Assumption~\ref{ass:power-law} and 
\(d\geq2(r+1)\) hold.  Then:
\begin{enumerate}[(i)]
\item There are positive constants
\(c_{S,-},c_{S,+},c_{K,-},c_{K,+}\), given in the proof and depending only
on the power-law constants, \(s,\nu\), and \(C_{d}\), such that, for
\(0\leq\aperp t\leq C_{d}d^\nu\),
\begin{equation}\label{eq:closed-exact-comparison}
\begin{aligned}
 \min\{1,c_{S,-}\}\SMS(t)
 &\leq S_A(t)\leq\max\{1,c_{S,+}\}\SMS(t),\\
 \min\{1,c_{K,-}\}\KMS(t)
 &\leq K_A(t)\leq\max\{1,c_{K,+}\}\KMS(t).
\end{aligned}
\end{equation}

\item The kernel $K_A$ satisfies
\begin{equation}\label{eq:kernel-mass}
 \int_0^\infty K_A(t)\dd t=\frac12\tr(AH).
\end{equation}

\item If \((r+1)^\nu\leq\aperp t\leq C_{d}d^\nu\), then
\begin{equation}\label{eq:late-kernel-laws}
 \SMS(t)\asymp
 \sum_{j=1}^r\lambda_j|w_j^\star|^2
 e^{-2\atpg\lambda_jt}
 +(\aperp t)^{-s},
 \qquad
 \KMS(t)\asymp\aperp^{1/\nu}t^{-p}.
\end{equation}

\item There is a constant \(C_{h}\)  depending only on the
power-law constants, \(s\), and \(\nu\), such that, if \(r\geq4\) and  $\atpg t\geq C_{h}r^\nu(1+\log\kappa)$, then
\begin{equation}\label{eq:finite-top-catchup}
\sum_{j=1}^r\lambda_j|w_j^\star|^2e^{-2\atpg\lambda_jt}
 \lesssim(\aperp t)^{-s},
\end{equation}
where \(\kappa:=\aperp/\atpg\geq1\).
\end{enumerate}

\end{lemma}

\begin{proof}
Throughout the proof, we abbreviate
$\beta_S:=1+s\nu,\, \beta_K:=2\nu$, $R:=r+1$.  

\smallskip
\noindent\emph{Proof of \textup{(i)}.}
For \(\beta>1\), define
\[
 L_1(\beta):=
 \min\left\{4^{-\beta}e^{-2\ell_+},
 \frac12\,4^{1-\beta}
 e^{-2^{\nu+1}\ell_+C_{d}}\right\},
\]
and set
\begin{align}
 c_{S,-}:={}&
 \vartheta_-\ell_-^sL_1(\beta_S),
 \label{eq:cSminus}\\
 c_{S,+}:={}&
 \vartheta_+\ell_+^s2^s
 \max\left\{1+\frac1{s\nu},
 \frac{\Gamma(s)}{\nu(2\ell_-)^s}
 +2\left(\frac{s+\nu^{-1}}{2e\ell_-}\right)^{s+\nu^{-1}}\right\},
 \label{eq:cSplus}\\
 c_{K,-}:={}&
 \ell_-^2L_1(\beta_K),
 \label{eq:cKminus}\\
 c_{K,+}:={}&
 \ell_+^22^p
 \max\left\{1+\frac1{2\nu-1},
 \frac{\Gamma(p)}{\nu(2\ell_-)^p}
 +\frac{8e^{-2}}{(2\ell_-)^2}\right\},
 \label{eq:cKplus}
\end{align}
where 
\(\Gamma(a):=\int_0^\infty y^{a-1}e^{-y}\,dy\) denotes the Gamma function for $a>0$. The following proof uses that  
\(c_-A\leq D\leq c_+A\) implies 
 $\min\{1,c_-\}(B+A)\leq B+D
 \leq\max\{1,c_+\}(B+A)$ for positive $A,B,D$.

\begin{itemize}
    \item \noindent\emph{Lower bounds.}
Set  \(u=\aperp t\), and for any 
\(\beta>1\),
\[
 F_{\beta,d}(u):=\sum_{j=R}^dj^{-\beta}
 e^{-2\ell_+u j^{-\nu}},
 \qquad \delta_\beta:=\frac{\beta-1}{\nu}.
\]
We prove that the tail estimate
\begin{equation}\label{eq:finite-tail-block-lower}
 F_{\beta,d}(u)\geq
 L_1(\beta)(R^\nu+u)^{-\delta_\beta}
\end{equation}
holds for any $0\leq u\leq C_{d}d^\nu$.
\begin{itemize}
    \item \textit{Case 1: \(u\leq R^\nu\)}. We have 
\[
 F_{\beta,d}(u)\ge\sum_{j=R}^{2R-1}j^{-\beta}
 e^{-2\ell_+u j^{-\nu}}\ge2^{-\beta}e^{-2\ell_+}R^{1-\beta}
 \geq4^{-\beta}e^{-2\ell_+}
 (R^\nu+u)^{-\delta_\beta}.
\]

\item \textit{Case 2: \(u> R^\nu\)}. Define \(J:=u^{1/\nu}>R\). 

 If \(J\leq d/4\),
\[
 F_{\beta,d}(u)\ge\sum_{J\le j \le 4J}j^{-\beta}
 e^{-2\ell_+u j^{-\nu}}\ge
 4^{-\beta}e^{-2\ell_+}J^{1-\beta}.
\]
If \(J>d/4\), \(u\le C_{d}d^\nu\) gives \(u j^{-\nu}
\le u(d/2)^{-\nu}
\le C_{d}d^\nu(d/2)^{-\nu}
=2^\nu C_{d}\).  It follows that 
\[
 F_{\beta,d}(u)\ge\sum_{\frac{d}{2}\le j\le d}j^{-\beta}
 e^{-2\ell_+u j^{-\nu}}\ge \frac d2\,d^{-\beta}
e^{-2^{\nu+1}\ell_+C_{d}}
 \geq\frac12\,4^{1-\beta}
 e^{-2^{\nu+1}\ell_+C_{d}}J^{1-\beta},
\]
where the last inequality uses $d<4J$.

In both cases,
\((R^\nu+u)^{-\delta_\beta}\leq
u^{-\delta_\beta}=J^{1-\beta}\), proving
\eqref{eq:finite-tail-block-lower}.

\end{itemize}

For the signal lower bound, since 
$\lambda_j|w_j^\star|^2e^{-2\aperp\lambda_jt}
\ge
\vartheta_-\ell_-^s
j^{-\beta_S}e^{-2\ell_+u j^{-\nu}}$, 
\[
\begin{aligned}
\sum_{j=R}^d
\lambda_j|w_j^\star|^2e^{-2\aperp\lambda_jt}
\ge
\vartheta_-\ell_-^s F_{\beta_S,d}(u)\ge
\vartheta_-\ell_-^s
L_1(\beta_S)
(R^\nu+u)^{-(\beta_S-1)/\nu}\ge
c_{S,-}(R^\nu+u)^{-s}.
\end{aligned}
\]

For the kernel lower bound, we use 
$
\aperp^2\lambda_j^2e^{-2\aperp\lambda_jt}
\ge
\aperp^2\ell_-^2j^{-2\nu}
e^{-2\ell_+u j^{-\nu}}$  and get
\[
\begin{aligned}
\sum_{j=R}^d
\aperp^2\lambda_j^2e^{-2\aperp\lambda_jt}
\ge
\aperp^2\ell_-^2F_{\beta_K,d}(u)
\ge
\aperp^2\ell_-^2L_1(\beta_K)
(R^\nu+u)^{-(\beta_K-1)/\nu}\ge
c_{K,-}\aperp^2(R^\nu+u)^{-p}.
\end{aligned}
\]

\item 
\noindent\emph{Upper bounds.}
For the upper bounds, it suffices to control the  sum

\[
G_\beta(u):=
\sum_{j=R}^\infty j^{-\beta}
e^{-2\ell_-u j^{-\nu}},\qquad \beta>1.
\]

\begin{itemize}
    \item \textit{Case 1: $u\leq R^\nu$.} We have
\[
G_\beta(u)\leq\sum_{j=R}^\infty j^{-\beta}
\le R^{1-\beta}\left(1+\frac1{\beta-1}\right).
\]
\item \textit{Case 2: $u> R^\nu$.} Define 
\[
f_{\beta,u}(x)
:=x^{-\beta}e^{-2\ell_-u x^{-\nu}}.
\]
 We use the elementary inequality
\[
\sum_{j\geq1}f(j)
\leq\int_0^\infty f(x)\,dx+2\sup_{x>0}f(x).
\]
For the integral term, set $\delta_\beta=\frac{\beta-1}{\nu}$ and make the change of variables $y=2\ell_-u x^{-\nu}$. This yields 
\[
\begin{aligned}
\int_0^\infty f(x)\,dx=\int_0^\infty
x^{-\beta}e^{-2\ell_-u x^{-\nu}}\,dx
=
\frac{(2\ell_-u)^{-\delta_\beta}}{\nu}
\int_0^\infty y^{\delta_\beta-1}e^{-y}\,dy
=
\frac{\Gamma(\delta_\beta)}
{\nu(2\ell_-u)^{\delta_\beta}}.
\end{aligned}
\]
To bound the supremum term, we compute the   derivative $\frac{d}{dx}\log f_{\beta,u}(x)
=-\frac{\beta}{x}
+2\ell_-u\nu x^{-\nu-1}$. Hence the maximum is attained at the point satisfying $2\ell_-u\nu x^{-\nu}=\beta$. Thus
$\sup_{x>0}f_{\beta,u}(x)
=\left(
\frac{\beta/\nu}{2e\ell_-u}
\right)^{\beta/\nu}$. Since \(u>R^\nu\geq1\) and
\(\beta/\nu=\delta_\beta+\nu^{-1}>\delta_\beta\), it follows that
\[
 G_\beta(u)\leq
 \left[
  \frac{\Gamma(\delta_\beta)}
       {\nu(2\ell_-)^{\delta_\beta}}
  +2\left(\frac{\beta/\nu}{2e\ell_-}\right)^{\beta/\nu}
 \right]u^{-\delta_\beta}.
\]

\end{itemize}

In either case,
\(R^\nu+u\leq2\max\{R^\nu,u\}\).  Consequently,
\begin{equation}\label{eq:generic-tail-upper}
\begin{aligned}
G_\beta(u)
\leq2^{\delta_\beta}
\max\left\{
 1+\frac1{\beta-1},
 \frac{\Gamma(\delta_\beta)}
      {\nu(2\ell_-)^{\delta_\beta}}
 +2\left(\frac{\beta/\nu}{2e\ell_-}\right)^{\beta/\nu}
\right\}(R^\nu+u)^{-\delta_\beta}.
\end{aligned}
\end{equation}

Substituting the value of $\beta$ into it   yields the upper bounds in 
\eqref{eq:closed-exact-comparison}.
\end{itemize}

\smallskip
\noindent\emph{Proof of \textup{(ii)}.}
Direct computation gives
\[
 \int_0^\infty K_A(t)\dd t
 =\sum_{j=1}^d a_j^2\lambda_j^2\frac1{2a_j\lambda_j}
 =\frac12\sum_{j=1}^d a_j\lambda_j
 =\frac12\tr(AH).
\]

\smallskip
\noindent\emph{Proof of \textup{(iii)}.}
Put \(u=\aperp t\).  Since \(y^2e^{-2y}\leq e^{-2}\) for \(y\geq0\), and  \(r\leq u^{1/\nu}\), 
we have 
\[
\sum_{j=1}^r\atpg^2\lambda_j^2e^{-2\atpg\lambda_jt}
 =\aperp^2u^{-2}\sum_{j=1}^r
 (\kappa^{-1}\lambda_ju)^2e^{-2b\lambda_ju}
 \leq e^{-2}\aperp^2u^{-2}r\le e^{-2}\aperp^2u^{-p}.
\]
This proves the forgetting kernel part in
\eqref{eq:late-kernel-laws}. The signal rate part is obtained by an analogous argument, where the sharp subspace summation is retained.

\smallskip
\noindent\emph{Proof of \textup{(iv)}.}
Define the constants $ c=2\ell_-, x=\atpg t, y=\frac{x}{r^\nu}$ for notational brevity. We suppose $y\ge1$ and will verify this condition later. For \(\beta>1\), set the constant 
\begin{equation}\label{eq:finite-top-constant}
\begin{aligned}
 C_\beta^{top}
 =2^\beta\left[
 \max\left\{1,\frac1{c\nu}\right\}
 +\frac{2}{e c(2^\nu-1)}\sum_{m\geq1}2^{m(\beta-1)}
 e^{-\frac c2(2^{m\nu}-1)}
 \right].
\end{aligned}
\end{equation}

We first consider the sum on \(r/2\leq j\leq r\).  With
\(n=r-j\), Bernoulli's inequality gives
$\left(\frac rj\right)^\nu
 =\left(1-\frac nr\right)^{-\nu}
 \geq1+\frac{\nu n}{r}$. 
Since \(j^{-\beta}\leq2^\beta r^{-\beta}\), we have
\begin{equation}\label{eq:terminal-top-block}
    \begin{aligned}
        \sum_{r/2\leq j\leq r}j^{-\beta}
 e^{-cy(r/j)^\nu}
\leq2^\beta r^{-\beta}e^{-cy}
 \sum_{n\geq0}e^{-c\nu yn/r}\leq2^\beta\max\left\{1,\frac1{c\nu}\right\}
 r^{1-\beta}(y^{-1}+r^{-1})e^{-cy},
    \end{aligned}
\end{equation}
where the last inequality uses
\((1-e^{-z})^{-1}\leq1+z^{-1}\) for \(z>0\).

For the remaining summation, we define
\[
 B_m:=\left\{j:\frac{r}{2^{m+1}}<j\leq\frac{r}{2^m}\right\},
 \qquad m\geq1.
\]  
On \(B_m\),
\[
 j^{-\beta}\leq\left(\frac{2^{m+1}}r\right)^\beta,
 \qquad
 \left(\frac rj\right)^\nu\geq2^{m\nu}.
\]
Consequently,
\begin{equation}\label{eq:dyadic-top-shells}
 \sum_{j\in B_m}j^{-\beta}e^{-cy(r/j)^\nu}
 \leq2^\beta r^{1-\beta}
 2^{m(\beta-1)}e^{-cy2^{m\nu}}=2^\beta r^{1-\beta}y^{-1}e^{-cy}
2^{m(\beta-1)}
\left[
y e^{-c(2^{m\nu}-1)y}
\right].
\end{equation}
For \(a_m:=c(2^{m\nu}-1)\), one has
\[
 y e^{-a_my}
 =\bigl(ye^{-a_my/2}\bigr)e^{-a_my/2}
 \leq\frac{2}{ea_m}e^{-a_m/2}
 \leq \frac{2}{e c(2^\nu-1)}e^{-a_m/2}.
\]
Thus 
\[
 \sum_{j\le r/2}j^{-\beta}e^{-cy(r/j)^\nu}
 \leq2^\beta \frac{2}{e c(2^\nu-1)}r^{1-\beta}y^{-1}e^{-cy}
 \sum_{m\geq1}2^{m(\beta-1)}
 e^{-\frac c2(2^{m\nu}-1)}.
\]
Combining it with 
\eqref{eq:terminal-top-block} gives 
\begin{equation}\label{eq:generic-finite-top-envelope}
 \sum_{j=1}^rj^{-\beta}e^{-cxj^{-\nu}}
 \leq C_\beta^{top}
 r^{1-\beta}(y^{-1}+r^{-1})e^{-cy}.
\end{equation}
Applying
\eqref{eq:generic-finite-top-envelope} with
\(\beta=\beta_S=1+s\nu\) gives
\begin{align}
\sum_{j=1}^r\lambda_j|w_j^\star|^2e^{-2\atpg\lambda_jt}
 \leq\vartheta_+\ell_+^s
 \sum_{j=1}^rj^{-\beta_S}e^{-cxj^{-\nu}}
 \leq C_S^{top}r^{-s\nu}
 (y^{-1}+r^{-1})e^{-2\ell_-y},
\end{align}
where the constant $C_S^{top}:=\vartheta_+\ell_+^sC_{\beta_S}^{top}$
depends only on the power-law constants, \(s\), and \(\nu\).

Finally, set
\[
 C_{h}:=\max\left\{1,\frac{s}{\ell_-}\right\}.
\]
Then 
\(y\geq C_{h}(1+\log\kappa)\ge1\).  Since
\(\kappa x=\aperp t\), division by
\((\kappa x)^{-s}=\kappa^{-s}r^{-s\nu}y^{-s}\) gives
\begin{align*}
\frac{\sum_{j=1}^r\lambda_j|w_j^\star|^2e^{-2\atpg\lambda_jt}}
 {(\aperp t)^{-s}}
 \leq C_S^{top}
 \bigl(y^{s-1}+r^{-1}y^s\bigr)e^{-\ell_-y}\leq C_S^{top}
 \sup_{z\geq1}(z^{s-1}+z^s)e^{-\ell_-z}<\infty,
\end{align*}
where the right hand side is a finite constant. 
This proves \eqref{eq:finite-top-catchup}.

\end{proof}

We now consider iterating the parameter‑momentum update as specified in Lemma \ref{lem:exact-three-state-recursion}. 
For the $j$-th coordinate and integers \(k>i\), define 
\[
 \Phi_{j;k,i}:=M_{j,k-1}M_{j,k-2}\cdots M_{j,i},
 \qquad \Phi_{j;i,i}:=I_3,
\]
and 
\begin{equation}
 R_{j;k,i}:=e_1^\top\Phi_{j;k,i+1}d_{j,i},
\end{equation}
where $e_1=(1,0,0)^\top$. 
Let
\begin{equation}\label{eq:aug-signal-kernel}
 e_k^{\rm signal}:=\frac12\sum_{j=1}^d\lambda_j
 \left|e_1^\top\Phi_{j;k,0}X_{j,0}\right|^2,
 \qquad
 K_{k,i}^{\rm noise}:=\sum_{j=1}^d\lambda_j^2R_{j;k,i}^2.
\end{equation}

\begin{lemma}[Two-sided  convolution-type  bounds on $\cE$]\label{lem:modal-volterra}
For the iterates $w_k$ generated by \eqref{eq:algorithm}, 
\begin{equation}\label{eq:discrete-volterra-sandwich}
\begin{aligned}
 e_k^{\rm signal}
 +\sum_{i<k}K_{k,i}^{\rm noise}
     \left(\E\cE(w_i)+\frac{\sigma^2}{2}\right)
\leq \E\cE(w_k)
 \leq e_k^{\rm signal}
 +\sum_{i<k}K_{k,i}^{\rm noise}
     \left(2\E\cE(w_i)+\frac{\sigma^2}{2}\right).
\end{aligned}
\end{equation}

\end{lemma}

\begin{proof}
Iterating \eqref{eq:exact-modal-state} yields
\[
 u_{j,k}=e_1^\top\Phi_{j;k,0}X_{j,0}
 +\sum_{i<k}R_{j;k,i}\xi_{j,i}.
\]

Since the noise is independent and has zero mean, it follows that 
\[
 \E\cE(w_k)=\frac12\sum_{j=1}^d\lambda_ju_{j,k}^2=e_k^{\rm signal}
 +\frac12\sum_{i<k}\sum_{j=1}^d
 \lambda_jR_{j;k,i}^2\,
 \E[e_j^\top\Sigma(u_i)e_j].
\]
Applying Lemma \ref{lem:gaussian-noise}  proves
\eqref{eq:discrete-volterra-sandwich}.
\end{proof}

Next we replace $e_k^{\rm signal},K_{k,i}^{\rm noise}$ with quantities related to \(S_A,K_A\), which have closed forms defined in 
\eqref{eq:adiabatic-kernels}.  The following lemma characterizes the approximation error of this replacement.  We first introduce some definitions for notational brevity. Let
\[
 V_k(\eta):=\sum_{i=0}^{k-2}|\eta_{i+1}-\eta_i|,
 \quad
 \delta_{\rm step}(k):=
 \max_{0\leq i\leq k-2}
 \frac{|\eta_{i+1}-\eta_i|}{\as\eta_i}, 
 \quad \delta_{\rm step}(1):=0.
\]
Define 
\begin{align}
 \varepsilon_{proj}(k)
 &:=
 \frac{\aperp\lambda_1}{\as^2}
 \bigl(\eta_{\max,k}+V_k(\eta)\bigr)
 +\frac{\delta_{\rm step}(k)}{\as},\label{eq:Eproj}\\
 \varepsilon_{exp}(k)
 &:=
 \aperp\left(d_1+\frac{\aperp}{2}\right)
 \lambda_1^2\sum_{i<k}\eta_i^2
 +c_{\log,1}\lambda_1^3\sum_{i<k}\eta_i^3,\label{eq:Eexp}\\
 \varepsilon_{ep}(k)
 &:=64\varepsilon_{proj}(k)+8\varepsilon_{exp}(k),
 \label{eq:epsilon-dyn}
\end{align}
and  for \(q\in\{0,1\}\), 
\begin{equation}\label{eq:Gbl}
\begin{aligned}
 G_q:=\left(
 \frac{\rf^2}{\af^2(1-\rf^2)}
 +\frac{2\chi q\rf\rs}{\af\as(1-\rf\rs)}
 +\frac{\chi^2q^2\rs^2}{\as^2(1-\rs^2)}\right)^{\frac{1}{2}}.
\end{aligned}
\end{equation}
For any nonnegative sequence
\(\{s_i\}_{i}\), define 
\begin{equation}\label{eq:terminal-abel-coefficients}
\begin{gathered}
 L_{j,k}:=\eta_{k-1}
 +\sum_{t=0}^{k-2}\prod_{m=t+2}^{k-1}(1-a_{q_j}\lambda_j\eta_m)
   \left|\eta_{t+1}-\eta_t+a_{q_j}\lambda_j\eta_t\eta_{t+1}\right|,\\
 \widehat\delta(k):=
 4\sum_{j=1}^d\lambda_j^2G_{q_j}^2L_{j,k}^2,\qquad
 (\widehat{\mathcal B} s)_k:=
 \widehat\delta(k)\max_{0\leq i<k}s_i.
\end{gathered}
\end{equation}
\begin{lemma}[Approximation errors]\label{lem:adiabatic-reduction}
For any   nonnegative sequence \(\{s_i\}\), define the operator $\mathcal T$ that maps $s$ to a sequence \(\{\mathcal T s\}_i\) as 
\begin{equation}\label{eq:def-cT}
    (\mathcal Ts)_k:=
 \sum_{i<k}\eta_i^2
 K_A(T_k-T_i)s_i,\qquad (\mathcal Ts)_0=0.
\end{equation} Suppose
\begin{equation}\label{eq:adiabatic-hypotheses}
 \aperp\lambda_1\eta_{\max,k}\leq\frac{\as}{8},
 \qquad
 \varepsilon_{ep}(k)\leq\frac1{16}.
\end{equation}
Then
\begin{equation}\label{eq:aggregate-adiabatic-signal}
 \left|e_k^{\rm signal}-\frac12S_A(T_k)\right|
 \leq4\varepsilon_{ep}(k)S_A(T_k),
\end{equation}
 and 
\begin{equation}\label{eq:aggregate-adiabatic-kernel}
\begin{aligned}
 \left|\sum_{i=0}^{k-1}K_{k,i}^{\rm noise}s_i
       -(\mathcal T s)_k\right|
 \leq
 \bigl(2\varepsilon_{ep}(k)
       +\varepsilon_{ep}(k)^2\bigr)(\mathcal T s)_k+2\bigl(1+\varepsilon_{ep}(k)\bigr)
   \sqrt{(\mathcal T s)_k
          (\widehat{\mathcal B} s)_k}
 +(\widehat{\mathcal B} s)_k.
\end{aligned}
\end{equation}
Further using Cauchy--Schwarz Inequality to \eqref{eq:aggregate-adiabatic-kernel} gives 
\begin{equation}\label{eq:terminal-abel-kernel-sandwich}
 \frac{225}{512}(\mathcal T s)_k
   -(\widehat{\mathcal B} s)_k
 \leq\sum_{i=0}^{k-1}K_{k,i}^{\rm noise}s_i
 \leq\frac{289}{128}(\mathcal T s)_k
   +2(\widehat{\mathcal B} s)_k.
\end{equation}

\end{lemma}

\begin{proof}
Fix a coordinate $j$ with eigenvalue \(\lambda:=\lambda_j\) and sharp subspace indicator \(q:=q_j\in\{0,1\}\). Set
\(x_m:=\eta_m\lambda\), and consider the projectors from
Lemma~\ref{lem:modal-projectors}.  The first condition in
\eqref{eq:adiabatic-hypotheses} ensures \(a_qx_m\leq\as/8\) at every
step.  Suppress \(q\) from the projector notation and define
\[
 L_m:=\|P_0(x_m)-P_0(x_{m-1})\|_\infty,\qquad m\geq1.
\]
The   estimate in \eqref{eq:projector-primitives} gives
\begin{equation}\label{eq:cumulative-projector-variation}
 \sum_{m=1}^{k-1}L_m
 \leq\frac{16a_q\lambda}{\as}V_k(\eta)
 \leq16\varepsilon_{proj}(k)\leq\frac1{64}.
\end{equation}

We considner  the matrix \(N_q(x)\) in Lemma \ref{lem:modal-projectors}, where \(q\in\{0,1\}\) is the sharp subspace indicator. To be specific, let  
\[
 T_\lambda:=
 \begin{pmatrix}
  1&0&0\\
  -1&\af/\lambda&0\\
  -1&0&\as/\lambda
 \end{pmatrix}.
\]
Then \(T_{\lambda_j}M_{j,\ell}T_{\lambda_j}^{-1}=N_q(x_{\ell,j})\) and \(T_\lambda e_1 =e_1\), where
\(x_{\ell,j}:=\eta_\ell\lambda_j\). When no ambiguity arises, we omit the subscript $j$ for brevity. We define the ordered 
matrix product $\overleftarrow{\prod_{\ell=1}^{k}}A_i
 :=A_kA_{k-1}...A_1$.


For the deterministic part $X$, since  $X_{j,0}=(-w_j^\star,0,0)^\top$, the $T$ transformed coordinates is $T_\lambda X_{j,0}=-w_j^\star h$ for $h=(1,-1,-1)^\top$. Define 
\[
 D^{s}_{j,k}:=
 -w_j^\star e_1^\top
 \overleftarrow{\prod_{m=0}^{k-1}}
 \bigl[N_q(x_m)P_0(x_m)\bigr]h,
 \qquad
 \overline U_{j,k}:=-w_j^\star
 \prod_{m=0}^{k-1}z_{0,q}(x_m),\qquad U^0_{j,k}:=-w_j^\star e^{-a_j\lambda_jT_k}.
\]

Next we bound each
 of relevant approximation errors to prove the corresponding conclusions.

\smallskip
\noindent\emph{Part 1: approximation errors related to  $D^{s}_{j,k}-U^0_{j,k}$.}

Recall that \(c_q=\bd+1+\chi q\), $d_q
=
\frac{1-\af}{\af^{\,2}}
+\chi q\frac{1-\as}{\as^{\,2}}$. Set
\[
 p_0:=P_0(x_0)h,\qquad
 p_m:=P_0(x_m)p_{m-1},\quad 1\leq m<k.
\]
Note that 
\(e_1^\top P_0'(0)h=d_q\) by \eqref{eq:projector-first-derivative} and $P_0(0)h=e_1$.  Hence
using \eqref{eq:projector-second-order-remainder},
\(d_q\leq a_q/\as\) and \(a_qx_0\leq\as/8\), we obtain
\begin{equation}\label{eq:deterministic-slow-input}
\begin{aligned}
 |e_1^\top p_0-1|=&|x_0e_1^\top P_0'(0)h
 +e_1^\top\bigl[P_0(x_0)-P_0(0)-x_0P_0'(0)\bigr]h|
 \\&\leq d_qx_0+\frac{32a_q^2x_0^2}{\as^2}\leq5\varepsilon_{proj}(k).
\end{aligned}
\end{equation}
Since  \(p_m\in\operatorname{Range}P_0(x_m)\),  
\eqref{eq:normalized-slow-right-eigenvector} gives
\(\|p_m\|_\infty=|e_1^\top p_m|\).  Moreover,
\begin{equation}\label{eq:deterministic-p-iter}
    p_m-p_{m-1}
 =[P_0(x_m)-P_0(x_{m-1})]p_{m-1},\qquad \|p_m\|_\infty \le  (1 + L_m)\|p_{m-1}\|_\infty. 
\end{equation}

Thus  
\begin{equation}\label{eq:max_pm-pk-1-p0-bound}
    \begin{aligned}
 \max_{m<k}\|p_m\|_\infty
 &\overset{\eqref{eq:deterministic-p-iter}}{\leq} \prod_{m=1}^{k-1}
(1 +L_m)\|p_0\|_\infty \overset{\eqref{eq:deterministic-slow-input}}{\leq}  e^{\sum_{m=1}^{k-1}L_m}(1+5\varepsilon_{proj}(k))\overset{\eqref{eq:cumulative-projector-variation}}{\leq}\frac{17}{16},
 \\\|p_{k-1}-p_0\|_\infty
 &\overset{\eqref{eq:deterministic-p-iter}}{\leq}\frac{17}{16}\sum_{m=1}^{k-1}L_m\overset{\eqref{eq:cumulative-projector-variation}}{\leq}17\varepsilon_{proj}(k).
    \end{aligned}
\end{equation}

Since \(N_q(x)P_0(x)=z_{0,q}(x)P_0(x)\), we have 
$
 D^{s}_{j,k}=-w_j^\star\prod_{m=0}^{k-1}z_{0,q}(x_m)e_1^\top p_{k-1}=e_1^\top p_{k-1} \overline U_{j,k}$. Then  \eqref{eq:deterministic-slow-input} and
\eqref{eq:max_pm-pk-1-p0-bound}  imply
\begin{equation}\label{eq:deterministic-all-slow}
\begin{aligned}
 |D^{s}_{j,k}-\overline U_{j,k}|
 =|\overline U_{j,k}|\,|e_1^\top p_{k-1}-1|
 \leq|\overline U_{j,k}|
 \left(\|p_{k-1}-p_0\|_\infty
       +|e_1^\top p_0-1|\right)
 \leq22\varepsilon_{proj}(k)|\overline U_{j,k}|.
\end{aligned}
\end{equation}

By \eqref{delta_x-1st_est-1},  \(z_{0,q}(x)\leq1-a_qx\leq e^{-a_qx}\). Combining it with \eqref{eq:slow-root-log-bound} gives
\[
 \left|\sum_{m=0}^{k-1}\log z_{0,q}(x_m)
 +a_q\lambda T_k\right| \leq 
 a_q\left(d_q+\frac{a_q}{2}\right)
 \lambda^2\sum_{m=0}^{k-1}\eta_m^2
 +c_{\log,q}\lambda^3\sum_{m=0}^{k-1}\eta_m^3
\leq \varepsilon_{exp}(k),
\]
where we also use  \(a_q\leq\aperp\), \(d_q\leq d_1\), and
\(c_{\log,q}\leq c_{\log,1}\).  It follows that 
\begin{equation}
   |\overline U_{j,k}-U^0_{j,k}|
 \leq \varepsilon_{exp}(k)|U^0_{j,k}| 
\end{equation}
Combining this estimate with \eqref{eq:deterministic-all-slow}  and $|\overline U_{j,k}|\leq|U^0_{j,k}|$ yields
\begin{equation}\label{eq:deterministic-slow-error}
\begin{aligned}
 |D^{s}_{j,k}-U^0_{j,k}|\leq
 \bigl(22\varepsilon_{proj}(k)+\varepsilon_{exp}(k)\bigr)|U^0_{j,k}|\leq\varepsilon_{ep}(k)|U^0_{j,k}|.
\end{aligned}
\end{equation}

\smallskip
\noindent\emph{Part 2: approximation errors related to   $e_1^\top\Phi_{j;k,0}X_{j,0}-D^{s}_{j,k}$.}

Set
\[
 y_{-1}:=h,\qquad y_m:=N_q(x_m)y_{m-1},\quad 0\leq m<k,
\]
and for $m\geq1$ consider the decomposition
\[
\begin{aligned}
 s_m &:=N_q(x_m)P_0(x_m)y_{m-1},& \,
 f_m &:=N_q(x_m)P_1(x_m)y_{m-1},\\
 s_m^0 &:=N_q(x_m)P_0(x_m)s_{m-1}^0,& \, a_m&:=s_m-s_m^0
\end{aligned}
\]
Note that \(y_m=s_m+f_m\), and \(\{s_m^0\}_m\) generates
\(D^{s}_{j,k}\).
For \(m\geq1\), we have  
\[
\begin{aligned}
 a_m&=N_q(x_m)P_0(x_m)(a_{m-1}+f_{m-1}),\\
 f_m&=N_q(x_m)P_1(x_m)
       (s_{m-1}^0+a_{m-1}+f_{m-1}).
\end{aligned}
\]

For any \(u\in\operatorname{Range}P_{0}(x_{m-1})\), $v\in \operatorname{Range}P_{1}(x_{m-1})$, we have \(P_{0}(x_{m-1})v=0\) and 
\begin{equation}
\begin{aligned}
    P_{0}(x_m)u
&=u+(P_{0}(x_m)-P_{0}(x_{m-1}))u,&\, P_{1}(x_{m})u
&=(P_{0}(x_{m-1})-P_{0}(x_m))u,
\\ P_{0}(x_{m})v&=(P_{0}(x_{m})-P_{0}(x_{m-1}))v,&\,P_{1}(x_{m})v&=(I-P_{0}(x_{m}))v. 
\end{aligned}  
\end{equation}  
Thus 
\begin{equation}\label{eq:p0mu-p1mu}
\begin{aligned}
      \|P_{0}(x_{m})u\|_\infty
&\leq(1+L_m)\|u\|_\infty,&\, \|P_{1}(x_{m})u\|_\infty
&\leq L_m\|u\|_\infty,\\
\|P_{0}(x_{m})v\|_\infty&\leq L_m\|v\|_\infty, &\, \|P_{1}(x_{m})v\|_\infty
&\leq(1+L_m)\|v\|_\infty.
\end{aligned}
\end{equation}

Set 
\[
 Q_m:=\frac{\|s_m^0\|_\infty}{\exp(-a_q\sum_{\ell=0}^m x_\ell)},\qquad
 A_m:=\frac{\|a_m\|_\infty}{\exp(-a_q\sum_{\ell=0}^m x_\ell)},\qquad
 F_m:=\frac{\|f_m\|_\infty}{\exp(-a_q\sum_{\ell=0}^m x_\ell)}.
\]
The inequality
\(a_qx_m\leq\as/8<17\as/20\) implies
\(1-\frac{17}{20}\as\leq e^{-a_qx_m}\).  Hence combining \eqref{eq:p0mu-p1mu}  and \eqref{eq:one-step-fast-contraction} yields
\begin{equation}\label{eq:deterministic-component-recurrence}
\begin{aligned}
 Q_m&\leq(1+L_m)Q_{m-1},\\
 A_m&\leq(1+L_m)A_{m-1}+L_mF_{m-1},\\
 F_m&\leq L_m(Q_{m-1}+A_{m-1})+(1+L_m)F_{m-1}.
\end{aligned}
\end{equation}

At \(m=0\), the identities
$
 s_0^0=z_{0,q}(x_0)P_0(x_0)h,\,
 f_0=N_q(x_0)P_1(x_0)h
$
together with \(z_{0,q}(x_0)\leq e^{-a_qx_0}\),
\eqref{eq:one-step-fast-contraction}, and \(1-\frac{17}{20}\as\leq e^{-a_qx_0}\) give
\[
 Q_0\leq\|P_0(x_0)h\|_\infty,
 \qquad
 F_0\leq\|P_1(x_0)h\|_\infty.
\]

Using $\|h\|_\infty=1$, \eqref{eq:projector-first-derivative} and \eqref{eq:projector-second-order-remainder}, we get 
\begin{equation}
\begin{aligned}
   \|P_0(x_0)h\|_\infty&=|e_1^\top P_0(x_0)h|
   \\&\le  |x_0e_1^\top P_0(0)h|+|x_0e_1^\top P_0'(0)h|
 +|e_1^\top(P_0(x_0)-P_0(0)-x_0P'_0(0))h|
 \\&\le 1+d_qx_0+\frac{32a_q^2x_0^2}{\as^2}
   \leq1+5\varepsilon_{proj}(k)
 \leq\frac{65}{64},      
\end{aligned}
\end{equation}
where the third inequality also uses \(d_q\leq a_q/\as\) and \(a_qx_0\leq\frac{\as}{8}\). 
On the other hand, \(P_1(0)h=(0,-1,-1)^\top\) implies 
\[
 \|P_1(x_0)h\|_\infty
 \leq1+\|P_0(x_0)-P_0(0)\|_\infty
 \leq1+\frac{16a_qx_0}{\as}
 \leq\frac{65}{64}.
\]
Consequently \(\max\{Q_0,F_0\}\leq65/64\), and hence
\(Q_m\leq17/16\).  Define \(C_m:=A_m+F_m\). Then
\eqref{eq:deterministic-component-recurrence} yields
\[
 C_m\leq(1+2L_m)C_{m-1}+L_mQ_{m-1}.
\]
Using \(S_{0,k}\leq1/64\) and iterating gives
\[
 \max_{m<k}C_m
 \leq e^{2S_{0,k}}
 \left(\frac{65}{64}+\frac{17}{16}S_{0,k}\right)
 \leq\frac76.
\]
For \(k\geq2\), iteration of the second line of
\eqref{eq:deterministic-component-recurrence} gives
\begin{equation}\label{eq:Ak-bound}
   A_{k-1}\leq e^{S_{0,k}}\max_{m<k}C_m\leq e^{S_{0,k}}\frac76S_{0,k}
 \leq19\varepsilon_{proj}(k). 
\end{equation}
 
It also holds for $k=1$ as $A_0=0$. 
Finally, \(e_1^\top P_1(0)=0\), so
\begin{equation}\label{eq:1-norm-ep1k}
     \|e_1^\top P_1(x_{k-1})\|_1
 \leq\|P_0(x_{k-1})-P_0(0)\|_\infty
 \leq\frac{16a_qx_{k-1}}{\as}.
\end{equation}

Note that 
\[
e_1^\top\Phi_{j;k,0}X_{j,0}-D^{s}_{j,k}
=-w_j^\star e_1^\top( y_{k-1}-  s_{k-1}^0)=-w_j^\star e_1^\top(a_{k-1}+f_{k-1}),
\]
and  
\begin{equation}
    \begin{aligned}
      |e_1^\top a_{k-1}|
&\leq\|a_{k-1}\|_\infty,
\\|e_1^\top f_{k-1}|
&=|e_1^\top P_1(x_{k-1})f_{k-1}|
\leq
\|e_1^\top P_1(x_{k-1})\|_1
\|f_{k-1}\|_\infty\overset{\eqref{eq:1-norm-ep1k}}{\leq}
\frac{16a_qx_{k-1}}{\as}
\|f_{k-1}\|_\infty.
    \end{aligned}
\end{equation} 
 
Thus 
\[
\begin{aligned}
 |e_1^\top\Phi_{j;k,0}X_{j,0}-D^{s}_{j,k}|
 &\leq |w_j^\star|e^{\sum_{i=0}^{k-1} x_i}
 \left(A_{k-1}+\frac{16a_qx_{k-1}}{\as}F_{k-1}\right)\\
 &\overset{\eqref{eq:Ak-bound}}{\leq}38\varepsilon_{proj}(k)|U^0_{j,k}|
 \leq\varepsilon_{ep}(k)|U^0_{j,k}|,
\end{aligned}
\]
where the second inequality also uses \(a_qx_{k-1}/\as^2\leq \varepsilon_{proj}(k)\). 
Combining this estimate with \eqref{eq:deterministic-slow-error} proves
\begin{equation}\label{eq:deterministic-total-error}
 |e_1^\top\Phi_{j;k,0}X_{j,0}-U^0_{j,k}|
 \leq2\varepsilon_{ep}(k)|U^0_{j,k}|.
\end{equation}
Thus $
 \left|
|e_1^\top\Phi_{j;k,0}X_{j,0}|^2-|U^0_{j,k}|^2
 \right|
 \leq4\varepsilon_{ep}(k)
 \bigl(1+\varepsilon_{ep}(k)\bigr)|U^0_{j,k}|^2$. 
Since $\sum_{j=1}^d\lambda_j|U^0_{j,k}|^2=S_A(T_k)$, 
  and  
\(\varepsilon_{ep}(k)\leq1/16\), we have  
\[
 \left|e_k^{\rm signal}-\frac12S_A(T_k)\right|
 \leq2\varepsilon_{ep}(k)
 \bigl(1+\varepsilon_{ep}(k)\bigr)S_A(T_k)
 \leq4\varepsilon_{ep}(k)S_A(T_k).
\]
This proves \eqref{eq:aggregate-adiabatic-signal}.

\smallskip
\noindent\emph{Part 3: approximation errors related to the noise decay.}
Here we focus on a fixed coordinate (i.e., fix $\lambda,q$), and omit the subscript for notational brevity.  
For \(0\leq i\leq k-1\), define
\begin{equation}
\begin{aligned}
    &\mathbf 0_k:=(0,\ldots,0)^\top\in \mathbb R^k,\qquad  \mathbf e_i^{(k)}
 :
 =\bigl(
\underbrace{0,\ldots,0}_{i\text{ entries}},
   1,
   \underbrace{0,\ldots,0}_{k-i-1\text{ entries}}
  \bigr)^\top\in \mathbb R^k, 
  \\&\mathbf u_0=\mathbf 0_k,
 \qquad
 \mathbf m_{-1}^{\mathrm f}
 =\mathbf m_{-1}^{\mathrm s}=\mathbf 0_k,
 \qquad
 \mathbf g_i=\lambda\mathbf u_i+\mathbf e_i^{(k)}
  \end{aligned}
\end{equation}  
If \([\mathbf x]_i\) denotes coordinate \(i\) of
\(\mathbf x\in \mathbb R^k\), then the preceding definition implies 
$
 [\mathbf g_t]_i
 =\lambda[\mathbf u_t]_i+\mathbb I_{\{i=t\}}$ for  $0\leq i,t\leq k-1$. We further let 
\begin{equation}\label{eq:scalar-modal-coordinate-recurrences}
\begin{aligned}
 \mathbf m_t^{\mathrm f}
 &=\rf\mathbf m_{t-1}^{\mathrm f}+\mathbf g_t,\\
 \mathbf m_t^{\mathrm s}
 &=\rs\mathbf m_{t-1}^{\mathrm s}+\mathbf g_t,\\
 \mathbf u_{t+1}-\mathbf u_t
 &=-\eta_t\bigl(\bd\mathbf g_t+\mathbf m_t^{\mathrm f}
                  +\chi q\mathbf m_t^{\mathrm s}\bigr),
 \qquad 0\leq t\leq k-1.
\end{aligned}
\end{equation}
Then 
\begin{equation}\label{eq:exact-response-row-vector}
 \mathbf u_t
 =\bigl(
   R_{j;t,0},\ldots,R_{j;t,t-1},
   \underbrace{0,\ldots,0}_{k-t\text{ entries}}
  \bigr)^\top
 =\sum_{i=0}^{t-1}R_{j;t,i}\mathbf e_i^{(k)},
 \qquad 0\leq t\leq k,
\end{equation}
and 
\[
 \mathbf g_t
 =\bigl(
   \lambda R_{j;t,0},\ldots,\lambda R_{j;t,t-1},1,
   \underbrace{0,\ldots,0}_{k-t-1\text{ entries}}
  \bigr)^\top,\qquad 0\leq t\leq k-1.
\]
We also define \(\mathbf g_r=\mathbf0_k\) for \(r<0\).  Since  $
 \mathbf m_t^{\mathrm f}
 =\sum_{\ell=0}^{t}\rf^\ell\mathbf g_{t-\ell},\, 
 \mathbf m_t^{\mathrm s}
  =\sum_{\ell=0}^{t}\rs^\ell\mathbf g_{t-\ell}$, we get that for $0\leq t\leq k-1$,
\begin{equation}\label{eq:scalar-modal-convolution}
\begin{aligned}
 \mathbf u_{t+1}-\mathbf u_t
 =-\eta_t\left(
   \bd\mathbf g_t+\sum_{\ell=0}^{t}\rf^\ell\mathbf g_{t-\ell}
   +\chi q\sum_{\ell=0}^{t}\rs^\ell\mathbf g_{t-\ell}
   \right)
 =-\eta_t\sum_{\ell=0}^{t}C_\ell(q)\mathbf g_{t-\ell},
\end{aligned}
\end{equation}
where $
 C_\ell(q):=\bd\mathbb I_{\{\ell=0\}}
       +\rf^\ell+\chi q\rs^\ell$ for every integer \(\ell\geq0\).

where the series is absolutely convergent because
\(0\leq\rf,\rs<1\).  For a  \(\mathbb R_k\)-valued sequence
\(\mathbf b=(\mathbf b_t)_{t\geq0}\), define the convolution with
\(D_n(q):=\sum_{\ell=n}^{\infty}C_\ell(q)
 =\frac{\rf^n}{\af}+\chi q\frac{\rs^n}{\as}\) ($n\ge1$) by
\begin{equation}\label{eq:Dq-causal-convolution}
 (D_q*\mathbf b)_t:=
 \begin{cases}
  \displaystyle\sum_{n=1}^{t+1}D_n(q)\mathbf b_{t-n+1},&t\geq0,\\[2mm]
  0,&t<0.
 \end{cases}
\end{equation}
It follows that for every \(t\geq0\),
\begin{equation}\label{eq:finite-causal-abel-identity}
\begin{aligned}
 (D_q*\mathbf b)_t-(D_q*\mathbf b)_{t-1}
 =D_1(q)\mathbf b_t
   +\sum_{\ell=1}^{t}
     \bigl[D_{\ell+1}(q)-D_\ell(q)\bigr]\mathbf b_{t-\ell}
 =a_q\mathbf b_t
   -\sum_{\ell=0}^{t}C_\ell(q)\mathbf b_{t-\ell}.
\end{aligned}
\end{equation}
Applying \eqref{eq:finite-causal-abel-identity} with
\(\mathbf b=\mathbf g\) turns
\eqref{eq:scalar-modal-convolution} into
\begin{equation}\label{eq:scalar-abel-recurrence}
 \mathbf u_{t+1}-\mathbf u_t
 =-a_q\eta_t\mathbf g_t
 +\eta_t\bigl[(D_q*\mathbf g)_t-(D_q*\mathbf g)_{t-1}\bigr],
 \qquad 0\leq t\leq k-1.
\end{equation}
Thus the effective gain \(a_q\) is separated exactly from a discrete
boundary difference; no spectral approximation has been made.

Define the sequence $\mathbf u_0^{E}=0$,  
\begin{equation}\label{eq:euler-modal-response}
 \mathbf u_{t+1}^{E}
 =(1-a_q\lambda\eta_t)\mathbf u_t^{E}
 -a_q\eta_t\mathbf e_t^{(k)},
 \qquad 0\leq t\leq k-1.
\end{equation}
It satisfies 
\[
 \mathbf u_t^{E}
 =-a_q\sum_{i=0}^{t-1}\eta_i
   \prod_{\ell=i+1}^{t-1}(1-a_q\lambda\eta_\ell)
   \mathbf e_i^{(k)},
 \qquad 0\leq t\leq k.
\]
Define
\[
\begin{aligned}
 \mathbf g_t^{E}&:=
 \lambda\mathbf u_t^{E}+\mathbf e_t^{(k)},
 &&0\leq t\leq k-1,\\
 \mathbf v_t&:=\mathbf u_t-\mathbf u_t^{E},
 &&0\leq t\leq k,
\end{aligned}
\]
and set \(\mathbf g_k^{E}:=0\).  Let
\(\mathsf X_k:=(\mathbb R_k)^{\{0,\ldots,k\}}\).  Define the map 
\(\mathsf B_q:\mathsf X_k\to\mathsf X_k\) as $(\mathsf B_q\mathbf b)_0:=0$, 
\begin{equation}\label{eq:boundary-solution-operator}
\begin{aligned}
 (\mathsf B_q\mathbf b)_{t+1}
 :=(1-a_q\lambda\eta_t)(\mathsf B_q\mathbf b)_t +\eta_t\bigl[(D_q*\mathbf b)_t
                    -(D_q*\mathbf b)_{t-1}\bigr],
 \qquad 0\leq t\leq k-1.
\end{aligned}
\end{equation}
Since
\(\mathbf g_t-\mathbf g_t^{E}=\lambda\mathbf v_t\) for
\(0\leq t\leq k-1\), subtracting
\eqref{eq:euler-modal-response} from
\eqref{eq:scalar-abel-recurrence} gives that for \(0\leq t\leq k-1\),
\[
 \mathbf v_{t+1}
 =(1-a_q\lambda\eta_t)\mathbf v_t
 +\eta_t\bigl[(D_q*\mathbf g^{E})_t
                    -(D_q*\mathbf g^{E})_{t-1}\bigr]+\lambda\eta_t\bigl[(D_q*\mathbf v)_t
                    -(D_q*\mathbf v)_{t-1}\bigr].
\]
Thus $ \mathbf v\in \mathsf X_k$ satisfies 
\begin{equation}\label{eq:boundary-feedback-equation}
 \mathbf v=\mathsf B_q\mathbf g^{E}
 +\lambda\mathsf B_q\mathbf v.
\end{equation}
Thus we consider the decomposition $\mathbf v_t=\sum_{r=1}^{t}\mathbf v_t^{(r)}
 $, where $\mathbf v^{(r+1)}:=\lambda\mathsf B_q\mathbf v^{(r)}$ ($1\leq r<k$) and $\mathbf v^{(1)}:=\mathsf B_q\mathbf g^{E}$. We set $\mathbf v_0=0$.


We first bound \(\mathsf B_q\) for a general input.  Set
$
 \gamma_{m,t}:=\eta_t
 \prod_{\ell=t+1}^{m-1}(1-a_q\lambda\eta_\ell)$ for $ 0\leq t<m\leq k$.
Then \eqref{eq:boundary-solution-operator} gives
\[
 (\mathsf B_q\mathbf b)_m
 =\sum_{t=0}^{m-1}\gamma_{m,t}
 \bigl[(D_q*\mathbf b)_t-(D_q*\mathbf b)_{t-1}\bigr].
\]
Since \((D_q*\mathbf b)_{-1}=0\), summation by parts yields
\begin{equation}\label{eq:boundary-summation-by-parts}
 (\mathsf B_q\mathbf b)_m
 =\eta_{m-1}(D_q*\mathbf b)_{m-1}
 +\sum_{t=0}^{m-2}
   (\gamma_{m,t}-\gamma_{m,t+1})(D_q*\mathbf b)_t.
\end{equation}
Then by $\sum_{n=1}^{\infty}D_n(q)=d_q$,  we  get
\begin{equation}\label{eq:generic-boundary-operator-bound-1}
\begin{aligned}
 \|(\mathsf B_q\mathbf b)_m\|\le (\eta_{m-1}
 +\sum_{t=0}^{m-2}|\gamma_{m,t}-\gamma_{m,t+1}|)d_q\sup_{0\leq t\leq m-1}
       \|\mathbf b_t\|.
       \end{aligned}
\end{equation}
Substituting
\begin{equation}\label{eq:sum-m-delta-gamma-bound}
    \begin{aligned}
 \sum_{t=0}^{m-2}|\gamma_{m,t+1}-\gamma_{m,t}|
 =&\sum_{t=0}^{m-2}\left|\left(\prod_{\ell=t+2}^{m-1}
       (1-a_q\lambda\eta_\ell)\right)
  \left(\eta_{t+1}-\eta_t\right)+a_q\lambda\eta_t\eta_{t+1}\prod_{\ell=t+2}^{m-1}
       (1-a_q\lambda\eta_\ell)\right|
  \\\le & V_k(\eta)+\eta_{\max,k}
 \sum_{t=0}^{m-2}(\prod_{\ell=t+2}^{m-1}(1-a_q\lambda\eta_\ell)-\prod_{\ell=t+1}^{m-1}(1-a_q\lambda\eta_\ell))
 \\\le & V_k(\eta)+\eta_{\max,k}.
\end{aligned}
\end{equation}
into \eqref{eq:generic-boundary-operator-bound-1} yields
\begin{equation}\label{eq:generic-boundary-operator-bound}
\begin{aligned}
\sup_{0\leq m\leq k}
 \|(\mathsf B_q\mathbf b)_m\|
\leq2\bigl(\eta_{\max,k}+V_k(\eta)\bigr)d_q
       \sup_{0\leq t\leq k-1}
       \|\mathbf b_t\|.
       \end{aligned}
\end{equation}

 Set $y_i:=a_q\lambda\eta_i$ and note that 
\[
G_q^2=\sum_{n=1}^{\infty}D_n(q)^2=
 \frac{\rf^2}{\af^2(1-\rf^2)}
 +\frac{2\chi q\rf\rs}{\af\as(1-\rf\rs)}
 +\frac{\chi^2q^2\rs^2}{\as^2(1-\rs^2)}
 \leq G_1^2.
\]
Direct computation 
gives
\[
 \sum_{n=1}^{\infty}\frac{\rf^n}{\af}
 \leq\sqrt{\frac2{\as}}
 \left(\sum_{n=1}^{\infty}\frac{\rf^{2n}}{\af^2}\right)^{1/2},
 \qquad
 \sum_{n=1}^{\infty}\frac{\chi q\rs^n}{\as}
 \leq\sqrt{\frac2{\as}}
 \left(\sum_{n=1}^{\infty}
       \frac{\chi^2q^2\rs^{2n}}{\as^2}\right)^{1/2},
\]
implying that $d_q=\sum_{n=1}^{\infty}D_n(q)\leq\frac{2G_q}{\sqrt{\as}}$. By \eqref{eq:euler-modal-response}, we have that for
\(1\leq t\leq k\),
\[
 \lambda\mathbf u_t^{E}
 =-\sum_{i=0}^{t-1}y_i
   \prod_{\ell=i+1}^{t-1}(1-y_\ell)\mathbf e_i^{(k)}.
\]
Then   \(0\leq y_i\leq\as/8\) gives that for $1\le t\le k$
\begin{equation}\label{eq:euler-residual-energy}
\begin{aligned}
 \|\lambda\mathbf u_t^{E}\|
 &=\sum_{i=0}^{t-1}y_i^2
   \prod_{\ell=i+1}^{t-1}(1-y_\ell)^2\leq
 \left(\max_{0\leq i\leq t-1}y_i\right)
 \sum_{i=0}^{t-1}\bigl[1-(1-y_i)^2\bigr]
   \prod_{\ell=i+1}^{t-1}(1-y_\ell)^2\\
 &=\left(\max_{0\leq i\leq t-1}y_i\right)
   \left[1-\prod_{\ell=0}^{t-1}(1-y_\ell)^2\right]\leq\frac{\as}{8}.
\end{aligned}
\end{equation}
For \(0\leq t\leq k-1\),
\[
 (D_q*\mathbf g^{E})_t
 =\sum_{n=1}^{t+1}D_n(q)\mathbf e_{t-n+1}^{(k)}
  +\sum_{n=1}^{t+1}D_n(q)
       \lambda\mathbf u_{t-n+1}^{E}.
\]
 Hence for any $0\le t\le k-1$,
\begin{equation}\label{eq:filtered-euler-energy}
\begin{aligned}
 \|(D_q*\mathbf g^{E})_t\| 
 &\leq
 \left\|\sum_{n=1}^{t+1}D_n(q)
              \mathbf e_{t-n+1}^{(k)}\right\| 
 +\sum_{n=1}^{t+1}D_n(q)
   \|\lambda\mathbf u_{t-n+1}^{E}\| \\
 &=\left(\sum_{n=1}^{t+1}D_n(q)^2\right)^{1/2}
 +\sum_{n=1}^{t+1}D_n(q)
   \|\lambda\mathbf u_{t-n+1}^{E}\| \\
 &\leq G_q+d_q
   \sup_{0\leq s\leq t}
   \|\lambda\mathbf u_s^{E}\|  \\ 
 &\overset{\eqref{eq:euler-residual-energy}}{\leq}\left(1+\frac1{\sqrt2}\right)G_q.
\end{aligned}
\end{equation}
Since \((\mathbf v^{(1)})_m=(\mathsf B_q\mathbf g^{E})_m=\eta_{m-1}(D_q*\mathbf g^{E})_{m-1}
+\sum_{t=0}^{m-2}
(\gamma_{m,t}-\gamma_{m,t+1})(D_q*\mathbf g^{E})_t,\)  combining 
\eqref{eq:boundary-summation-by-parts},
\eqref{eq:sum-m-delta-gamma-bound}, and
\eqref{eq:filtered-euler-energy} yields 
\begin{equation}\label{eq:first-boundary-layer}
 \|\mathbf v_m^{(1)}\| \le  
\left(
\eta_{m-1}
+\sum_{t=0}^{m-2}
|\gamma_{m,t}-\gamma_{m,t+1}|
\right)
\sup_{0\leq t<m}\|(D_q*\mathbf g^{E})_t\| 
 \leq(2+\sqrt2)
 \bigl(\eta_{\max,k}+V_k(\eta)\bigr)G_q.
\end{equation}

Set $\theta_k:=2\lambda
\bigl(\eta_{\max,k}+V_k(\eta)\bigr)d_q$. Then $\theta_k\leq2\as \varepsilon_{proj}(k)
 \leq2\varepsilon_{proj}(k)\leq\frac1{512}$ by \(d_q\leq a_q/\as\) and  
\(a_q\lambda\leq\aperp\lambda_1\).  \eqref{eq:generic-boundary-operator-bound} gives that 
for \(1\leq r<k\),
\[
 \sup_{0\leq m\leq k}
 \|\mathbf v_m^{(r+1)}\| 
 \leq\theta_k\sup_{0\leq m\leq k}
 \|\mathbf v_m^{(r)}\|
 .
\]

Combining this with \eqref{eq:first-boundary-layer} yields 
\[
\begin{aligned}
\sup_{0\leq m\leq k} \|\mathbf v_m\| 
 &\leq\sum_{r=1}^{k}\sup_{0\leq m\leq k}\|\mathbf v_m^{(r)}\|  \leq(2+\sqrt2)
 \bigl(\eta_{\max,k}+V_k(\eta)\bigr)G_q
 \sum_{r=1}^{k}\theta_k^{r-1}\\
 &\leq\frac{2+\sqrt2}{1-\theta_k}
 \bigl(\eta_{\max,k}+V_k(\eta)\bigr)G_q.
\end{aligned}
\]
Substituting it  and \eqref{eq:generic-boundary-operator-bound-1} into \eqref{eq:boundary-feedback-equation} yields a finer estimate:
\begin{align*}
 \|\mathbf v_k\|
 &\leq \left((1+\frac{1}{\sqrt{2}})G_{q} 
 +\lambda d_{q} 
   \sup_{0\leq m\leq k}\|\mathbf v_m\|\right)(\eta_{k-1}
 +\sum_{t=0}^{k-2}|\gamma_{k,t}-\gamma_{k,t+1}|)\\
 &\leq 
 2G_{q}(\eta_{k-1}
 +\sum_{t=0}^{k-2}|\gamma_{k,t}-\gamma_{k,t+1}|).
\end{align*}

Define $R^0_{j;k,i}:=-\eta_i a_q
 e^{-a_q\lambda(T_k-T_i)}$ and $
 R^{E}_{j;k,i}
 :=-a_q\eta_i
   \prod_{\ell=i+1}^{k-1}(1-a_q\lambda\eta_\ell)$ 
for $ 0\leq i\leq k-1$. Note that \eqref{eq:exact-response-row-vector}  gives
\[
 \mathbf v_k
 =\sum_{i=0}^{k-1}
   \bigl(R_{j;k,i}-R^{E}_{j;k,i}\bigr)
   \mathbf e_i^{(k)}.
\]
 Then 
\begin{equation}\label{eq:R-RE-bound}
    \sum_{i=0}^{k-1}(R_{j;k,i}-R^{E}_{j;k,i})^2=\|\mathbf v_k\|^2
 \le 4(\eta_{k-1}
 +\sum_{t=0}^{k-2}|\gamma_{k,t}-\gamma_{k,t+1}|)^2G_1^2. 
\end{equation}

Next we bound $|R^{E}_{j;k,i}-R^0_{j;k,i}|$.  Note that 
\[
 y_i\leq a_q\lambda\eta_{\max,k}
 \leq\as^2\varepsilon_{proj}(k)\leq \varepsilon_{proj}(k),
\]
and $\sum_{\ell=i+1}^{k-1}y_\ell^2\leq2\varepsilon_{exp}(k)$ as 
\(a_q^2\lambda^2\leq\aperp^2\lambda_1^2\leq2\aperp(d_1+\aperp/2)\lambda_1^2\). Combining them with $0\leq-\log(1-y)-y 
 \leq\frac47y^2$ for \(0\leq y\leq1/8\), 
we get 
\[
 \log\frac{|R^{E}_{j;k,i}|}{|R^0_{j;k,i}|}
 =y_i+\sum_{\ell=i+1}^{k-1}
       \bigl[\log(1-y_\ell)+y_\ell\bigr]
 \leq \varepsilon_{proj}(k)+\frac87\varepsilon_{exp}(k)
 \leq\frac17\varepsilon_{ep}(k).
\]
Using \(|e^x-1|\leq2|x|\) for \(|x|\leq1/2\), we obtain
\begin{equation}\label{eq:euler-exponential-response}
 |R^{E}_{j;k,i}-R^0_{j;k,i}|
 \leq\varepsilon_{ep}(k)|R^0_{j;k,i}|.
\end{equation}

Taking summation of \eqref{eq:R-RE-bound}, \eqref{eq:euler-exponential-response} over the coordinate $j$  gives 
\begin{equation}\label{eq:response-error-E-norms}
 \sum_{i=0}^{k-1}\sum_{j=1}^d\lambda_j^2
 (R_{j;k,i}-R^{E}_{j;k,i})^2 s_i
\leq\widehat\delta(k)
       \max_{0\leq i<k}s_i
 =(\widehat{\mathcal B} s)_k,   
\end{equation}
and 
\begin{equation}\label{eq:response-error-norms}
\begin{aligned}
 \sum_{i=0}^{k-1}\sum_{j=1}^d\lambda_j^2
 (R^{E}_{j;k,i}-R^0_{j;k,i})^2s_i
 &\leq\varepsilon_{ep}(k)^2(\mathcal Ts)_k,
\end{aligned}
\end{equation}
respectively.  
Note that
\begin{equation}\label{eq:weighted-r0-kaugs}
     \sum_{i=0}^{k-1}\sum_{j=1}^d\lambda_j^2
 (R^0_{j;k,i})^2s_i=(\mathcal Ts)_k,\qquad \sum_{i=0}^{k-1}K_{k,i}^{\rm noise}s_i=\sum_{i=0}^{k-1}\sum_{j=1}^d\lambda_j^2
 (R_{j;k,i})^2s_i.
\end{equation}

Thus
\begin{equation}
    \begin{aligned}
 \left|(\mathcal Ts)_k-       \sum_{i=0}^{k-1}K_{k,i}^{\rm noise}s_i\right|^2\le& \sum_{i=0}^{k-1}\sum_{j=1}^d\lambda_j^2
 (R_{j;k,i}+R^0_{j;k,i})|R_{j;k,i}-R^0_{j;k,i}|s_i
 \\\le&\sqrt{\sum_{i=0}^{k-1}\sum_{j=1}^d\lambda_j^2
(R_{j;k,i}+R^0_{j;k,i})^2s_i}\sqrt{\sum_{i=0}^{k-1}\sum_{j=1}^d\lambda_j^2
(R_{j;k,i}-R^0_{j;k,i})^2s_i}
\\
 \leq&\bigl((2+\varepsilon_{ep}(k))\sqrt {(\mathcal Ts)_k}+\sqrt{(\widehat{\mathcal B} s)_k}\bigr)
          \bigl(\varepsilon_{ep}(k)\sqrt {(\mathcal Ts)_k}+\sqrt{  (\widehat{\mathcal B} s)_k}\bigr)\\
 =&\bigl(2\varepsilon_{ep}(k)+\varepsilon_{ep}(k)^2\bigr) (\mathcal Ts)_k
   +2(1+\varepsilon_{ep}(k))\sqrt{(\mathcal Ts)_k(\widehat{\mathcal B} s)_k}+(\widehat{\mathcal B} s)_k,
    \end{aligned}
\end{equation}
where the second inequality uses Cauchy--Schwarz Inequality, the third inequality uses 
\begin{equation}\label{eq:sum-r-r0-weighted}
\begin{aligned}
 \sqrt{\sum_{i=0}^{k-1}\sum_{j=1}^d\lambda_j^2
(R_{j;k,i}-R^0_{j;k,i})^2s_i}
 \leq&\sqrt{\sum_{i=0}^{k-1}\sum_{j=1}^d\lambda_j^2
(R_{j;k,i}-R^{\mathrm{E}}_{j;k,i})^2s_i}+\sqrt{\sum_{i=0}^{k-1}\sum_{j=1}^d\lambda_j^2
(R^{\mathrm{E}}_{j;k,i}-R^0_{j;k,i})^2s_i}
 \\\overset{\eqref{eq:response-error-E-norms},\eqref{eq:response-error-norms}}{\leq}&\sqrt{(\widehat{\mathcal B} s)_k}+\varepsilon_{ep}(k)\sqrt{(\mathcal Ts)_k},
\end{aligned}
\end{equation}
and
\begin{equation}
\begin{aligned}
 \sqrt{\sum_{i=0}^{k-1}\sum_{j=1}^d\lambda_j^2
(R_{j;k,i}+R^0_{j;k,i})^2s_i}
 \leq&\sqrt{\sum_{i=0}^{k-1}\sum_{j=1}^d\lambda_j^2
(R_{j;k,i}-R^{0}_{j;k,i})^2s_i}+2\sqrt{\sum_{i=0}^{k-1}\sum_{j=1}^d\lambda_j^2
(R^0_{j;k,i})^2s_i}
 \\\overset{\eqref{eq:sum-r-r0-weighted},\eqref{eq:weighted-r0-kaugs}}{\leq}&\sqrt{(\widehat{\mathcal B} s)_k}+(2+\varepsilon_{ep}(k))\sqrt{(\mathcal Ts)_k}.
\end{aligned}
\end{equation}

This proves \eqref{eq:aggregate-adiabatic-kernel} as \(\varepsilon_{ep}(k)\leq1/16\).

\end{proof}

We define  the following notations for proving Theorem~\ref{thm:discrete-fsl}: 
\begin{equation}\label{eq:epsilon-fast}
\begin{aligned}
&\FMS(k):=\SMS(k)+\sigma^2\sum_{i=0}^{k-1}\eta_i^2\KMS(T_k-T_i),
\\&c_{sk,-}:=\min\{1,c_{S,-},c_{K,-}\},
 \qquad
 c_{sk,+}:=\max\{1,c_{S,+},c_{K,+}\},\\
 &\widehat\delta^{\rm max}(k)
 :=\max_{1\leq n<k}\widehat\delta(n),
 \qquad
 K_{\eta}^{\rm max}(k)
 :=\max_{1\leq n<k}\sum_{i=0}^{n-1}
 \eta_i^2K_A(T_n-T_i), 
 \qquad
\widehat\varepsilon(k):=
\frac{\sigma^2\widehat\delta(k)}{\FMS(k)},
\end{aligned}
\end{equation}
where we assume $\E\cE(w_0)\le C_0\sigma^2$ with $C_0$ treated as a constant. At \(k=1\), both maxima over \(1\leq n<k\) in
\eqref{eq:epsilon-fast} are defined to be zero, and \(K_{\eta}^{\rm max}(1):=0\).

\begin{assumption}[Admissible hyper-parameters]
    Suppose that the horizon \(k\) and the learning rate  schedule $\{\eta_i\}$
satisfy
\begin{equation}\label{eq:discrete-regime}
 \varepsilon_{ep}(k)\leq\frac1{16},
 \qquad
  K_{\eta}^{\rm max}(k)\le\frac{1}{8},
 \qquad
 \widehat\delta^{\rm max}(k)\leq\frac{1}{25},
 \qquad
 \widehat\varepsilon(k)
 \leq\frac{c_{sk,-}}{16(6C_0+1)},\qquad \aperp\lambda_1\eta_{\max,k}\leq\frac{\as}{8}.
\end{equation}

\end{assumption}

\begin{remark}[Simplified sufficient conditions for admissible hyperparameters]
  Here we provide a cleaner version of the hyperparameter constraints in \eqref{eq:discrete-regime}. The constraints on $\varepsilon_{ep}(k)$ and $\widehat\delta^{\rm max}(k)$ reduce to upper bounds on $\eta_{\max,k}$, $\sum_{i=1}^{k-2}|\eta_{i+1}-\eta_i|$, and \(\max_{i<k-1}|\eta_{i+1}/\eta_i-1|\). By \eqref{eq:fixed-prefix-row-bound}, a sufficient condition for
  \(K_{\eta}^{\rm max}(k)\le \frac{1}{8}\) is $
    \eta_{\max,k}\le \frac{1}{4}\left(a_0\sum_{j=1}^r j^{-\nu}+a_1\sum_{j=r+1}^d j^{-\nu}\right)^{-1}$, which  also becomes an upper bound constraint to $\eta_{\max,k}$.

  We next turn to the upper bound constraint on $\widehat\varepsilon(k)$. Deriving a simplified sufficient condition for general learning rate schedules is challenging. However, for some common choices such as constant and power-law decay, it can be satisfied. Excluding the very early stage, we consider the regime $T_k\gtrsim \frac{r^\nu}{\atpg}$ so that
  $\KMS(T_k-s)\lesssim\KMS(T_k-T_i)$ for $s\in[T_i,T_{i+1}]$. Hence,
  \begin{align}
    \int_0^{T_k}\KMS(z)\,dz
    =\sum_{i=0}^{k-1}\int_{T_i}^{T_{i+1}}
      \KMS(T_k-s)\,ds 
    \lesssim\sum_{i=0}^{k-1}\eta_i\KMS(T_k-T_i)
    \leq\frac{1}{\eta_{k-1}}
    \sum_{i=0}^{k-1}\eta_i^2\KMS(T_k-T_i).
  \end{align}
  This yields $
    \sum_{i=0}^{k-1}\eta_i^2\KMS(T_k-T_i)
    \gtrsim\eta_{k-1}$. 
  Now consider \(\eta_i=\eta_0(1+i/\tau_0)^{-\gamma}\), where \(\tau_0\in\mathbb{R}_+\cup\{+\infty\}\) and \(1/2<\gamma\le 1\). The inequality
    $\frac{\eta_t-\eta_{t+1}}{\eta_{t+1}}
    =\left(1+\frac{1}{\tau_0+t}\right)^\gamma-1
    \leq\frac{\gamma}{\tau_0+t}$
  gives
  \begin{align}
    \left[
      \sum_{t=0}^{k-2}(\eta_t-\eta_{t+1})
      \sqrt{\KMS(T_k-T_{t+2})}
    \right]^2 \lesssim
    \sum_{t=0}^{k-2}
    \left(\frac{\eta_t-\eta_{t+1}}{\eta_{t+1}}\right)^2
    \sum_{t=0}^{k-2}\eta_{t+1}^2
    \KMS(T_k-T_{t+1}) \leq
    \frac{2\gamma^2}{\tau_0}
    \sum_{i=0}^{k-1}\eta_i^2\KMS(T_k-T_i).
  \end{align}
  Combining this with
  \[
    \frac{\eta_{k-1}^2\KMS(0)}
    {\sum_{i=0}^{k-1}\eta_i^2\KMS(T_k-T_i)}
    \lesssim\eta_{k-1},
  \]
  we obtain
  \[
    \frac{\left[
      \eta_{k-1}\sqrt{\KMS(0)}
      +\sum_{t=0}^{k-2}(\eta_t-\eta_{t+1})
        \sqrt{\KMS(T_k-T_{t+2})}
    \right]^2}
    {\displaystyle\sum_{i=0}^{k-1}
      \eta_i^2\KMS(T_k-T_i)}
    \lesssim \eta_{k-1}+\frac{\gamma^2}{\tau_0}.
  \]
  Note that $
    V_k(\eta)\leq\eta_0,\,
    \delta_{\rm step}(k)\leq\frac{\gamma}{\alpha_{\text{slow}}\tau_0}$. 
  By Lemma~\ref{lem:admissibility-decreasing-schedule}, for sufficiently large \(d,\tau_0\) and sufficiently small \(\eta_0\), the admissibility condition \eqref{eq:discrete-regime} holds.
\end{remark}

\begin{proof}[Proof of Theorem~\ref{thm:discrete-fsl}]

Suppose the conditions in \eqref{eq:discrete-regime} hold. We structure the proof as follows.
 
\smallskip
\noindent\emph{Step 1: the uniform upper bound for $\cE$.} 

For \(0\leq i\leq k-1\), we define the sequence 
$\cE_i:=\E\cE(w_i)$, $
 s_i^-:=\cE_i+\frac{\sigma^2}{2},\,
 s_i^+:=2\cE_i+\frac{\sigma^2}{2}$, and we denote by \(\mathbf 1\)  the sequence consisting entirely of ones. 
Since \(K_A\) is decreasing, 
\begin{equation}\label{eq:fixed-T-infinity-bound}
\begin{aligned}
 (\mathcal T\mathbf1)_n
 &=\sum_{i=0}^{n-1}\eta_i^2K_A(T_n-T_i)\leq\eta_{\max,k}\sum_{i=0}^{n-1}
       \int_{T_i}^{T_{i+1}}K_A(T_n-s)\dd s\\
&=\eta_{\max,k}\int_0^{T_n}K_A(u)\dd u
 \overset{\eqref{eq:kernel-mass}}{\le}\frac{\eta_{\max,k}}2\tr(AH).
\end{aligned}
\end{equation}
Taking the maximum over \(0\leq n<k\) gives
\begin{equation}\label{eq:fixed-prefix-row-bound}
\max_{0\leq n<k} (\mathcal T\mathbf1)_n\le K_{\eta}^{\rm max}(k)
 \leq\frac{\eta_{\max,k}}2\tr(AH).
\end{equation}

The  upper bounds in
\eqref{eq:discrete-volterra-sandwich} and  
\eqref{eq:terminal-abel-kernel-sandwich} give that for $1\leq n\leq k$, 
\begin{equation}\label{eq:fixed-prefix-kernel-chain}
\begin{aligned}
 \cE_n
  \leq& e_n^{\rm signal}
   +\sum_{i=0}^{n-1}K_{n,i}^{\rm noise}s_i^{+} \leq e_n^{\rm signal}
   +\frac{289}{128}(\mathcal T s^{+})_n +2(\widehat{\mathcal B} s^+)_n
   \\\le &\frac32C_0\sigma^2+\frac{289}{64}(\mathcal T\cE)_n
 +\frac{289}{256}\sigma^2(\mathcal T\mathbf 1)_n +4\widehat\delta(n)\max_{0\leq i\leq n-1}\cE_i
 +\widehat\delta(n)\sigma^2.
\end{aligned}
\end{equation}
where  the last inequality uses $ (\mathcal T s^{+})_n 
 =2(\mathcal T\cE)_n
   +\frac{\sigma^2}{2}(\mathcal T\mathbf 1)_n,\, 
 (\widehat{\mathcal B} s^+)_n
 =\widehat\delta(n)
   \left(2\max_{0\leq i\leq n-1}\cE_i
         +\frac{\sigma^2}{2}\right)$ and 
\begin{align}
         e_n^{\rm signal}
  \overset{\eqref{eq:aggregate-adiabatic-signal}}{\leq}\left(\frac12+4\varepsilon_{ep}(n)\right)S_A(T_n)
 \leq\frac34S_A(0)
 \leq\frac32C_0\sigma^2.
\end{align}
 Then using $\max_{0\leq n<k} (\mathcal T\mathcal E )_n\le K_{\eta}^{\rm max}(k)\max_{0\leq i\leq k-1}\cE_i$ and  \eqref{eq:fixed-prefix-row-bound} to \eqref{eq:fixed-prefix-kernel-chain}, we get
 \begin{equation}\label{eq:fixed-prefix-scalar-closure}
 \max_{0\leq i\leq k-1}\cE_i\leq\frac32C_0\sigma^2
 +\left[
   \frac{289}{64}K_{\eta}^{\rm max}(k)+4\widehat\delta^{\rm max}(k)
  \right]
  \left(\max_{0\leq i\leq k-1}\cE_i+\frac{\sigma^2}{4}\right).
\end{equation}
With the condition $
   \frac{289}{64}K_{\eta}^{\rm max}(k)+4\widehat\delta^{\rm max}(k)\le \frac{3}{4}$ in \eqref{eq:discrete-regime}, we get 
\begin{equation}\label{eq:fixed-prefix-bound}
 \max_{0\leq i<k}\cE_i
\leq\left(6C_0+\frac34\right)\sigma^2.
\end{equation}

\smallskip
\noindent\emph{Step 2: the asymptotic bounds for $\cE_k$.}
  
The estimate in  \eqref{eq:fixed-prefix-bound} gives
$
 \max_{0\leq i<k}s_i^-
 \leq\left(6C_0+\frac54\right)\sigma^2,\,
 \max_{0\leq i<k}s_i^+
 \leq2\left(6C_0+1\right)\sigma^2
$.

Thus 
\begin{equation}\label{eq:fixed-final-TBS-bound}
    \begin{aligned}
        (\mathcal T s^-)_k
 &=\sum_{i=0}^{k-1}\eta_i^2K_A(T_k-T_i)
   \left(\cE_i+\frac{\sigma^2}{2}\right)
 \geq\frac{\sigma^2}{2}(\mathcal T \mathbf 1)_k,\\
 (\mathcal T s^+)_k
 &=\sum_{i=0}^{k-1}\eta_i^2K_A(T_k-T_i)
   \left(2\cE_i+\frac{\sigma^2}{2}\right)
 \leq 2\left(6C_0+1\right)\sigma^2(\mathcal T \mathbf 1)_k,
 \\
 (\widehat{\mathcal B} s^-)_k
 &=\widehat\delta(k)\max_{0\leq i<k}s_i^-
 \leq \left(6C_0+\frac54\right)\sigma^2\widehat\delta(k),
 \\
 (\widehat{\mathcal B} s^+)_k
 &=\widehat\delta(k)\max_{0\leq i<k}s_i^+
 \leq 2\left(6C_0+1\right)\sigma^2\widehat\delta(k)
    \end{aligned}
\end{equation}
Besides, \eqref{eq:aggregate-adiabatic-signal} gives
\begin{equation}\label{eq:fixed-final-signal}
 \frac14S_A(T_k)\le\left(\frac12-4\varepsilon_{ep}(k)\right)S_A(T_k)
 \leq e_k^{\rm signal}
 \leq\left(\frac12+4\varepsilon_{ep}(k)\right)S_A(T_k)\le\frac34S_A(T_k).
\end{equation}

 Define $G_k:=S_A(T_k)+\sigma^2(\mathcal T \mathbf 1)_k=S_A(T_k)+\sigma^2\sum_{i=0}^{k-1}\eta_i^2K_A(T_k-T_i)$. 
The estimates in
\eqref{eq:closed-exact-comparison}  give 
\begin{equation}\label{eq:Gk-up-low-bound}
    \begin{aligned}
      G_k&\geq\min\{1,c_{S,-}\}\SMS(T_k)
 +\sigma^2\min\{1,c_{K,-}\}
   \sum_{i=0}^{k-1}\eta_i^2\KMS(T_k-T_i)\geq c_{sk,-}\FMS(k), 
   \\ G_k&\leq\max\{1,c_{S,+}\}\SMS(T_k)
 +\sigma^2\max\{1,c_{K,+}\}
   \sum_{i=0}^{k-1}\eta_i^2\KMS(T_k-T_i)\leq c_{sk,+}\FMS(k).
    \end{aligned}
\end{equation}
we obtain
\begin{equation}\label{eq:fixed-final-lower}
\begin{aligned} 
 \cE_k
 &\overset{\eqref{eq:discrete-volterra-sandwich}}{\geq} e_k^{\rm signal}
   +\sum_{i=0}^{k-1}K_{k,i}^{\rm noise}s_i^- \overset{\eqref{eq:terminal-abel-kernel-sandwich}}{\geq}\frac14S_A(T_k)
   +\frac{225}{512}(\mathcal Ts^-)_k-(\widehat{\mathcal B} s^-)_k\\
&\overset{\eqref{eq:fixed-final-TBS-bound}}{\geq}\frac14S_A(T_k)+\frac{225}{1024}\sigma^2(\mathcal T \mathbf 1)_k
   -\left(6C_0+\frac54\right)\sigma^2\widehat\delta(k)
   \\
&\overset{\eqref{eq:discrete-regime}}{\geq}\left[
       \frac{225}{1024}
       -\frac{24C_0+5}{64(6C_0+1)}\right]
       c_{sk,-}G_k
 =\frac{966C_0+145}{1024(6C_0+1)}
   c_{sk,-}G_k.
\end{aligned}
\end{equation}
and 
\begin{equation}\label{eq:fixed-final-upper}
\begin{aligned}
 \cE_k
 &\overset{\eqref{eq:discrete-volterra-sandwich}}{\leq} e_k^{\rm signal}
   +\sum_{i=0}^{k-1}K_{k,i}^{\rm noise}s_i^+ \overset{\eqref{eq:terminal-abel-kernel-sandwich}}{\leq}\frac34S_A(T_k)+\frac{289}{128}(\mathcal Ts^+)_k+2 (\widehat{\mathcal B} s^+)_k \\
&\overset{\eqref{eq:fixed-final-TBS-bound}}{\leq}\frac34S_A(T_k)+\frac{289}{64}(6C_0+1)\sigma^2(\mathcal T \mathbf 1)_k
   +4(6C_0+1)\sigma^2\widehat\delta(k)
\overset{\eqref{eq:discrete-regime}}{\leq}\left(
   \frac{289}{64}(6C_0+1)c_{sk,+}
   +\frac{c_{sk,-}}4\right)G_k.
\end{aligned}
\end{equation}

Combining \eqref{eq:fixed-final-lower}, \eqref{eq:fixed-final-upper} and \eqref{eq:Gk-up-low-bound}, we get
\[
\cE_k\asymp\FMS(k).
\]
This proves the conclusion.
\end{proof}

\begin{proof}[Proof of Corollary~\ref{cor:constant-schedule}]
Set $t_k:=k\eta,\,R:=r+1$.  

\smallskip
\noindent\emph{Part 1: estimate the noise accumulation term.}
Enlarge  the constant $C_{h}$ supplied by
Lemma~\ref{lem:power-law-kernels} \textup{(iv)} so that
\begin{equation}\label{eq:constant-catchup-normalization}
 2\ell_-C_{h}\geq\log 2,
 \qquad
 C_{h}\left(\frac45\right)^\nu
 \geq\max\left\{1,2^{1/(1-\frac1\nu)}-1\right\}.
\end{equation}

 \(T_i=i\eta\)  gives $\sum_{i=0}^{k-1}\eta^2\KMS(T_k-T_i)=\eta^2\sum_{m=1}^{k}\KMS(m\eta)$. 
Decompose \(\KMS(z)=\sum_{j=1}^rK_j^{top}(z)+K^{tail}(z)\), where for $j\le r$, 
\[
 K_j^{top}(z):=\atpg^2\lambda_j^2
 e^{-2\atpg\lambda_jz},\qquad K^{tail}(z):=\aperp^2(R^\nu+\aperp z)^{-p}.
\]
Let 
\begin{equation}
    J_{\mathsf K}(t):=\int_0^t\KMS(z)\dd z=\frac{\atpg}{2}\sum_{j=1}^r\lambda_j
 \left(1-e^{-2\atpg\lambda_jt}\right)+\frac{\aperp}{1-\nu^{-1}}
 \left[R^{1-\nu}-(R^\nu+\aperp t)^{-(1-\nu^{-1})}\right].
\end{equation} For  \(1\leq m\leq k\),
\begin{equation}\label{eq:constant-top-tail-mesh-ratio}
    \begin{aligned}
       \frac{K_j^{top}(m\eta)}
      {K_j^{top}((m-1)\eta)}
 &=e^{-2\atpg\lambda_j\eta}
 \geq e^{-2\atpg\lambda_1\eta},
 \\
 \frac{K^{tail}(m\eta)}
      {K^{tail}((m-1)\eta)}
 &=\left(1+\frac{\aperp\eta}
 {R^\nu+\aperp(m-1)\eta}\right)^{-p}
 \geq\left(1+\frac{\aperp\eta}{R^\nu}\right)^{-p}. 
    \end{aligned}
\end{equation}
Set the constant $
 c_m:=\min\left\{
 e^{-2\atpg\lambda_1\eta},
 \left(1+\frac{\aperp\eta}{R^\nu}\right)^{-p}
 \right\}$. 
Then \eqref{eq:constant-top-tail-mesh-ratio} implies that for
\(z\in[(m-1)\eta,m\eta]\),
$
 \KMS(m\eta)
 \leq\KMS(z)
 \leq\KMS((m-1)\eta)
 \leq c_m^{-1}\KMS(m\eta)$. 
Thus 
\[
 c_m\int_{(m-1)\eta}^{m\eta}\KMS(z)\dd z
 \leq\eta\KMS(m\eta)
 \leq\int_{(m-1)\eta}^{m\eta}\KMS(z)\dd z.
\]
It follows that 
\begin{equation}\label{eq:constant-grid-integral-comparison}
 c_m\eta J_{\mathsf K}(t_k)
 \leq\eta^2\sum_{m=1}^{k}\KMS(m\eta)
 \leq\eta J_{\mathsf K}(t_k).
\end{equation}

By the hyper-parameter constraints in Theorem~\ref{thm:discrete-fsl}, we have
\begin{align}
 2\atpg\lambda_1\eta
 \leq2\aperp\lambda_1\eta
 \leq\frac{\as}{4}
 \leq\frac14, \qquad 
 \frac{\aperp\eta}{R^\nu}
 =\frac{\aperp\lambda_1\eta}{\lambda_1R^\nu}
 \leq\frac{\as}{8\lambda_1R^\nu}
 \leq\frac1{8\ell_-R^\nu}
 \leq\frac1{8\ell_-}.
\end{align}
Consequently, $c_m\geq c_{m,0}:=
 \min\left\{e^{-1/4},
 \left(1+\frac1{8\ell_-}\right)^{-p}\right\}\gtrsim 1$.

Since  \(d\geq2R\), we have
\begin{equation}\label{eq:constant-tail-spectral-mass}
\begin{aligned}
 \sum_{j=R}^d\lambda_j
 \geq\sum_{j=R}^{2R-1}\ell_-j^{-\nu}
 \geq \ell_-2^{-\nu}R^{1-\nu},\qquad
 \sum_{j=R}^d\lambda_j
 \leq \ell_+\sum_{j=R}^{\infty}j^{-\nu}
 \leq\frac{\ell_+}qR^{1-\nu}.
\end{aligned}
\end{equation}
Then defining $ c_{tr}=\min\left\{\frac12,\frac1{\ell_+}\right\},\,
 C_{tr}=\max\left\{\frac12,
 \frac{2^\nu}{(1-\nu^{-1}) \ell_-}\right\}$ and using $\tr(AH)=\atpg\sum_{j=1}^r\lambda_j
 +\aperp\sum_{j=R}^d\lambda_j$, we get
\begin{equation}\label{eq:constant-proxy-trace-comparison}
 c_{tr}\tr(AH)
 \leq =\underbrace{\frac{\atpg}{2}\sum_{j=1}^r\lambda_j
 +\frac{\aperp}{1-\nu^{-1}}R^{1-\nu}}_{=J_{\mathsf K}(\infty)}
 \leq C_{tr}\tr(AH).
\end{equation}

The constraint 
\eqref{eq:constant-schedule-window}  implies that $
 \atpg\lambda_jt_k
 \geq \ell_-C_{h}
 \left(\frac rj\right)^\nu(1+\log\kappa)
 \geq \ell_-C_{h}$. Combining this with \eqref{eq:constant-catchup-normalization} implies that for every \(j\leq r\),
\begin{equation}\label{eq:1-ealamtk}
  1-e^{-2\atpg\lambda_jt_k}
 \geq1-e^{-2\ell_-C_{h}}
 \geq\frac12.  
\end{equation}

Moreover, \(\kappa\geq1\)  and
\(r/R=r/(r+1)\geq4/5\) give 
\begin{equation}\label{eq:alamtk-lower-ccat}
    \frac{\aperp t_k}{R^\nu}
 \geq\kappa C_{h}
 \left(\frac rR\right)^\nu(1+\log\kappa)
 \geq C_{h}\left(\frac45\right)^\nu.
\end{equation}
 This implies that  
\begin{equation}\label{eq:R-R-t_R}
    \frac{R^{1-\nu}-(R^\nu+\aperp t_k)^{-(1-\nu^{-1})}}
 {R^{1-\nu}}=1-\left(1+\frac{\aperp t_k}{R^\nu}\right)^{-(1-\nu^{-1})}
 \geq 1-\left[1+C_{h}\left(\frac45\right)^\nu\right]^{-(1-\nu^{-1})}
 \geq\frac12.
\end{equation}
Applying \eqref{eq:1-ealamtk} and \eqref{eq:R-R-t_R} to the closed form of $J_{\mathsf K}(t_k)$ yields
\begin{equation}\label{eq:constant-kernel-saturation}
 \frac12J_{\mathsf K}(\infty)
 \leq J_{\mathsf K}(t_k)
 \leq J_{\mathsf K}(\infty).
\end{equation}
Combining \eqref{eq:constant-grid-integral-comparison}, 
\eqref{eq:constant-proxy-trace-comparison}, and
\eqref{eq:constant-kernel-saturation} gives  
\begin{equation}\label{eq:constant-noise-kernel-comparison}
 \frac{c_{m,0}c_{tr}}2\eta\tr(AH)
 \leq \eta^2\sum_{m=1}^{k}\KMS(m\eta)
 \leq C_{tr}\eta\tr(AH).
\end{equation}

\smallskip
\noindent\emph{Part 2: estimate  the signal term.}
By \eqref{eq:finite-top-catchup}, there exists a fixed
\(C_{top}>0\) such that the lower inequality in
\eqref{eq:constant-schedule-window} implies
\begin{equation}\label{eq:constant-top-signal-bound}
 \sum_{j=1}^r\lambda_j|w_j^\star|^2
 e^{-2\atpg\lambda_jt_k}
 \leq C_{top}(\aperp t_k)^{-s}.
\end{equation}
 \eqref{eq:constant-catchup-normalization} and
\eqref{eq:alamtk-lower-ccat} imply
\(\aperp t_k\geq R^\nu\), and therefore
\begin{equation}\label{eq:constant-tail-signal-bound}
 2^{-s}(\aperp t_k)^{-s}
 \leq(R^\nu+\aperp t_k)^{-s}
 \leq(\aperp t_k)^{-s}.
\end{equation}
Using 
\eqref{eq:constant-top-signal-bound} and
\eqref{eq:constant-tail-signal-bound} to   $\SMS(t)
 =\sum_{j=1}^{r}\lambda_j|w_j^\star|^2
       e^{-2\atpg\lambda_jt}
   +\bigl((r+1)^\nu+\aperp t\bigr)^{-s}$, we have  
\begin{equation}\label{eq:constant-signal-comparison}
 2^{-s}(\aperp\eta k)^{-s}
 \leq\SMS(\eta k)
 \leq(C_{top}+1)(\aperp\eta k)^{-s}.
\end{equation}

Finally, combining
 \eqref{eq:constant-noise-kernel-comparison} and
\eqref{eq:constant-signal-comparison}   gives
  \eqref{eq:constant-schedule-final}.
\end{proof}

\begin{lemma}[Admissibility for a decreasing schedule]
\label{lem:admissibility-decreasing-schedule}
Suppose Assumption~\ref{ass:power-law} and \(d\geq2(r+1)\) hold.
Let
\(k\geq2\), and suppose
\[
 \eta_0\geq\eta_1\geq\cdots\geq\eta_{k-1}>0,
 \qquad
 \aperp T_k\leq C_{d}d^\nu,
 \qquad
 \aperp\lambda_1\eta_0\leq\frac{\as}{8}.
\]
Then 
\begin{equation}\label{eq:decreasing-prefix-bounds}
 \widehat\delta^{\rm max}(k)
 \leq16\eta_0^2\sum_{j=1}^d\lambda_j^2G_{q_j}^2,
 \qquad
 K_{\eta}^{\rm max}(k)
 \leq\frac{\eta_0}{2}\tr(AH).
\end{equation}
If we further assume
\begin{equation}\label{eq:closed-terminal-relative-sufficient-condition}
\begin{aligned}
&\max\left\{\frac{G_0^2}{\atpg^2},
             \frac{G_1^2}{\aperp^2}\right\}
 \left[
  \eta_{k-1}\sqrt{\KMS(0)}
  +\sum_{t=0}^{k-2}(\eta_t-\eta_{t+1})
    \sqrt{\KMS(T_k-T_{t+2})}
 \right]^2\\
&\qquad\leq
 \frac{c_{sk,-}}
 {256(6C_0+1)\max\{1,c_{K,+}\}}
 \sum_{i=0}^{k-1}\eta_i^2\KMS(T_k-T_i),
\end{aligned}
\end{equation}
it holds that 
\[
 \widehat\varepsilon(k)
 \leq
 \frac{c_{sk,-}}{16(6C_0+1)}.
\]
\end{lemma}

\begin{proof}
For every \(1\leq j\leq d\),
\(0\leq t\leq k-1\), and \(1\leq\ell\leq k\), we have
\begin{equation}\label{eq:terminal-product-positivity}
 0\leq a_{q_j}\lambda_j\eta_t
 \leq\aperp\lambda_1\eta_0
 \leq\frac{\as}{8}<1,
 \qquad
 0\leq\Pi_{j;k,\ell}\leq1.
\end{equation}
Using 
$\Pi_{j;k,t+2}-\Pi_{j;k,t+1}=a_{q_j}\lambda_j\eta_{t+1}\Pi_{j;k,t+2}$, we get
\begin{equation}\label{eq:terminal-monotone-telescoping}
\begin{aligned}
 \sum_{t=0}^{k-2}a_{q_j}\lambda_j
   \eta_t\eta_{t+1}\Pi_{j;k,t+2}
 =
 \eta_{k-1}
 +\sum_{t=0}^{k-2}\Pi_{j;k,t+2}(\eta_t-\eta_{t+1})
 -\eta_0\Pi_{j;k,1}.
\end{aligned}
\end{equation}
Combining this with
$
 \Pi_{j;k,t+2}
 \leq e^{-a_{q_j}\lambda_j(T_k-T_{t+2})}$  yields
\begin{equation}\label{eq:terminal-L-full-variation}
\begin{aligned}
 L_{j,k}
 &\leq
 2\left[
 \eta_{k-1}
 +\sum_{t=0}^{k-2}\Pi_{j;k,t+2}(\eta_t-\eta_{t+1})
 -\frac{\eta_0}{2}\Pi_{j;k,1}
 \right]\\
 &\leq
 2\left[
 \eta_{k-1}
 +\sum_{t=0}^{k-2}(\eta_t-\eta_{t+1})
 e^{-a_{q_j}\lambda_j(T_k-T_{t+2})}
 \right].
\end{aligned}
\end{equation}
For \(0\leq u\leq T_k\),  
\eqref{eq:closed-exact-comparison} implies 
\[
 \sum_{j=1}^d\lambda_j^2G_{q_j}^2
 e^{-2a_{q_j}\lambda_ju}
 \leq
 \max\left\{\frac{G_0^2}{\atpg^2},
             \frac{G_1^2}{\aperp^2}\right\}
 \max\{1,c_{K,+}\}\KMS(u).
\]
Combing it with \eqref{eq:terminal-L-full-variation}   gives
\begin{equation}\label{eq:terminal-delta-full-variation}
\begin{aligned}
\widehat\delta(k)
\leq 
16\max\left\{\frac{G_0^2}{\atpg^2},
              \frac{G_1^2}{\aperp^2}\right\}
\max\{1,c_{K,+}\} 
\left[
 \eta_{k-1}\sqrt{\KMS(0)}
 +\sum_{t=0}^{k-2}(\eta_t-\eta_{t+1})
 \sqrt{\KMS(T_k-T_{t+2})}
\right]^2 .
\end{aligned}
\end{equation}
Note that  \(L_{j,n}\leq2\eta_0\) by \eqref{eq:terminal-L-full-variation} for $n\le k$.  Therefore
 $\widehat\delta^{\rm max}(k)\leq16\eta_0^2\sum_{j=1}^d\lambda_j^2G_{q_j}^2$,
and then we prove 
\eqref{eq:decreasing-prefix-bounds} by using \eqref{eq:fixed-prefix-row-bound}. Finally, using 
$ \FMS(k)\geq
 \sigma^2\sum_{i=0}^{k-1}\eta_i^2\KMS(T_k-T_i)$, \eqref{eq:terminal-delta-full-variation} and  
\eqref{eq:closed-terminal-relative-sufficient-condition} proves the conclusion.
\end{proof}

\subsection{Useful Lemmas}

\begin{theorem}[Gershgorin disk theorem]
\label{prop:gershgorin}
For \(M=(m_{ij})\in\mathbb C^{n\times n}\), define the row disks
\[
 D_i:=\left\{z\in\mathbb C:
 |z-m_{ii}|\leq\sum_{j\ne i}|m_{ij}|\right\}.
\]
Every eigenvalue of \(M\) belongs to \(\bigcup_{i=1}^nD_i\).  More
precisely, if a connected component of this union consists of exactly
\(k\) disks and is disjoint from all remaining disks, then it contains
exactly \(k\) eigenvalues of \(M\), counted with algebraic multiplicity.
\end{theorem}
\begin{theorem}[Neumann series and matrix resolvent identities]
\label{prop:neumann-resolvent}
Let \(E\in\mathbb C^{n\times n}\), and the norm $\|\cdot\|\in\{\|\cdot\|_2,\|\cdot\|_\infty\}$.  If \(\|E\|<1\), then
\begin{equation}
 (I-E)^{-1}=\sum_{n=0}^\infty E^n,
 \qquad
 \|(I-E)^{-1}\|\leq\frac1{1-\|E\|}.
 \label{eq:neumann-series}
\end{equation}
For \(A\in\mathbb C^{n\times n}\), write
\(R_A(z):=(zI-A)^{-1}\) whenever \(zI-A\) is invertible.  If
\(\|R_A(z)E\|<1\), then
\begin{equation}
 R_{A+E}(z)
 =(I-R_A(z)E)^{-1}R_A(z),
 \qquad
 \|R_{A+E}(z)\|
 \leq\frac{\|R_A(z)\|}{1-\|R_A(z)E\|}.
 \label{eq:neumann-resolvent-bound}
\end{equation}
Whenever \(zI-A\) and \(wI-A\) are invertible,
\begin{equation}
 R_A(z)-R_A(w)=(w-z)R_A(z)R_A(w).
 \label{eq:first-resolvent-identity}
\end{equation}
For \(A,B\in\mathbb C^{n\times n}\), whenever both \(zI-A\) and \(zI-B\)
are invertible,
\begin{equation}
 R_A(z)-R_B(z)=R_A(z)(A-B)R_B(z).
 \label{eq:second-resolvent-identity}
\end{equation}
\end{theorem}

\begin{theorem}[Riesz spectral projection formula]
\label{prop:riesz}
Let \(A\in\mathbb C^{n\times n}\), and \(\Gamma\) be a positively
oriented,  smooth, simple closed contour with no eigenvalue of \(A\)
on the contour.  The matrix
\begin{equation}
 P_\Gamma(A):=\frac1{2\pi\mathrm i}
 \int_\Gamma(zI-A)^{-1}\,\dd z
 \label{eq:riesz-projection}
\end{equation}
is the projection (i.e., $P_\Gamma(A)^2=P_\Gamma(A)$) onto the generalized eigenspaces corresponding to the
eigenvalues inside \(\Gamma\).  It commutes with \(A\), and for every
integer \(n\geq0\),
\begin{equation}
 A^nP_\Gamma(A)=\frac1{2\pi\mathrm i}
 \int_\Gamma z^n(zI-A)^{-1}\,\dd z.
 \label{eq:riesz-power-formula}
\end{equation}

If no eigenvalue of \(B\in\mathbb C^{n\times n}\) lies on \(\Gamma\), then
\begin{align}
 P_\Gamma(A)-P_\Gamma(B)
  =\frac1{2\pi\mathrm i}\int_\Gamma
 R_A(z)(A-B)R_B(z)\dd z.
\end{align}
\end{theorem}
\end{document}